\documentclass{article}

\usepackage{natbib}

\usepackage[utf8]{inputenc}
\usepackage{amsmath, amssymb, amsthm}
\usepackage[dvipsnames]{xcolor}
\usepackage[margin=1in]{geometry}

\usepackage[
    colorlinks=true,
    linkcolor=black,
    citecolor=black,
    urlcolor=black
]{hyperref}

\newtheorem{theorem}{Theorem}[section]
\newtheorem{lemma}[theorem]{Lemma}
\newtheorem{proposition}[theorem]{Proposition}
\newtheorem{corollary}[theorem]{Corollary}

\theoremstyle{definition}
\newtheorem{definition}[theorem]{Definition}
\newtheorem{remark}[theorem]{Remark}
\newtheorem{assumption}[theorem]{Assumption}

\DeclareMathOperator{\clr}{clr}
\DeclareMathOperator{\supp}{supp}
\DeclareMathOperator{\dist}{dist}
\renewcommand{\d}{\mathrm{d}}
\newcommand{\obsdata}{\mathsf{d}}

\title{Measure flow path recovery in Bayes Hilbert spaces}
\author{
S. David Mis\\
Rice University\\
Houston, TX, USA
\and
Maarten V. de Hoop\\
Rice University\\
Houston, TX, USA
}
\date{\today}

\begin{document}

\maketitle

% !TeX root = ../main.tex

\begin{abstract}
    We study the ill-posed problem of recovering a probability measure flow from finitely many moving localized sensors using a Bayes Hilbert framework.
    Relative to a fixed reference probability measure, a probability law is represented by its centered log-ratio coordinates, so that an evolving law becomes a path in a Hilbert space of functions.
    For sufficiently regular Bayes Hilbert paths, we construct a canonical minimum-energy transport realization of the path by solving a weighted Neumann problem at each time, which yields an intrinsic transport form on tangent directions.

    We then formulate an inverse problem directly on Bayes Hilbert path space.
    Linearization of an observation operator yields an observability form, and recoverability is governed by its interaction with the transport geometry through a joint transport--observability form.
    In the ambient infinite-dimensional setting, we develop the corresponding regularized variational theory and identify a basic limitation of localized sensing: mobile sensors can make the joint form injective, but they do not in general yield a coercive stability estimate on the full state space.

    This obstruction leads naturally to finite-dimensional Bayes Hilbert reductions.
    There the transport form becomes a kinetic tensor and the linearized observations become reduced sensing matrices, so recoverability can be expressed through explicit Gramian conditions.
    We show that localized bump sensors detect every fixed reduced direction, that finitely many suitably placed static sensors yield uniform reduced observability, and there exist path-dependent sensor trajectories such that even a single moving sensor can recover the reduced path.
    Finally, we show that these reduced recovery results lift to approximate ambient recovery for paths that are well approximated by the chosen finite-dimensional subspaces, yielding stable reconstruction up to projection error.
\end{abstract}
\tableofcontents

% !TeX root = ../main.tex

\section{Introduction}

We study the recovery of a time-dependent probability law from indirect, time-dependent observations. Fix a bounded domain $\Omega \subset \mathbb{R}^d$ and a reference probability measure \(\nu_0\). In Bayes Hilbert coordinates, a probability measure $\rho$ is represented using the \emph{centered log-ratio transform} \citep{van_den_boogaart_bayes_2014}
\[
    h=\clr(\rho)\in L^2_0(\nu_0), \qquad L^2_0(\nu_0) = \left\{ h \in L^2(\nu_0) : \int_\Omega h \; \d \nu_0 = 0\right\}.
\]
Using this transform, an evolving law $\rho_t$ becomes a path $h(t)$ in a Hilbert space of functions. The inverse problem is to reconstruct the unknown path \(h(\cdot)\), equivalently the law path \(t\mapsto \rho_{h(t)}\), from sensor data. We show that for any measure flow $\rho_{h(t)}$ lying in a finite-dimensional Bayes Hilbert subspace \(V_m\), there exist a configuration of finitely many static localized sensors making the path fully observable at each time. Furthermore, for every finite-dimensional $\rho_t$ there exists a ($\rho_t$-dependent) sensor trajectory in $\Omega$ such that a single moving sensor can fully recover $\rho_t$ even though the full state is not observed at any time. We then show that these finite-dimensional recovery results lift to approximate recovery of $\rho_t$ in the full infinite dimensional setting, with error determined by the intrinsic error of projecting $\rho_t$.

Unlike the finite-dimensional setting, localized sensing does not in general yield a coercive stability estimate on a full infinite-dimensional Bayes Hilbert state space. Mobile sensing can make the relevant bilinear form injective, but it does not by itself restore ambient coercivity except in the trivial fully observed case. This obstruction explains why finite-dimensional Bayes Hilbert reductions are not merely technical approximations, but the natural level at which localized sensing leads to genuine recovery theorems.

In this paper, we first construct a canonical transport field associated to a prescribed Bayes Hilbert path. For sufficiently regular \(h\), a canonical velocity is determined by a weighted Neumann problem at each time. This selects the unique transport realization of minimum kinetic energy compatible with the instantaneous evolution of the law and induces an intrinsic transport form $\mathfrak g_h(\xi,\zeta)$ on tangent directions. Thus a Bayes Hilbert path carries both a law-valued evolution and a distinguished transport geometry.

The inverse problem is formulated directly on path space. If \(\mathcal G_t^y\) denotes the observation operator at time \(t\), with \(y\) the sensor trajectory, then linearization at a state \(h\) yields an observability form
\[
    \mathfrak j_{t,h}^y(\xi,\zeta)
    :=
    \bigl\langle D\mathcal G_t^y(h)\xi,\;D\mathcal G_t^y(h)\zeta\bigr\rangle .
\]
Recoverability is governed by the interaction between \(\mathfrak g_h\) and \(\mathfrak j_{t,h}^y\). The observation term determines which directions are visible to the data, while the transport term penalizes the time variation required for a perturbation to remain hidden as the sensing configuration changes. This leads naturally to a joint transport--observability form on perturbation paths, which is the central object driving our analysis.

This perspective also clarifies the role of sensor motion. In finite dimensions, moving sensors can rotate nullspaces of the observability form over time, so that directions invisible at any single time become recoverable over the observation horizon.
In infinite dimensions, by contrast, localized sensing remains compact at each fixed time, and this compactness prevents full ambient coercivity.
The finite-dimensional theory makes this mechanism explicit: on \(V_m\), the transport form becomes a reduced kinetic tensor and the linearized observations become reduced sensing matrices, so recoverability can be expressed through concrete conditions on a Gramian.

The Bayes Hilbert framework is well suited to this program because it cleanly separates three distinct objects: the measure path itself, its dynamical transport realization, and its recovery from data. The primary unknown is the coordinate path \(h(\cdot)\). The corresponding probability law is obtained by exponentiation and normalization, while the canonical dynamics are recovered from \(h(\cdot)\) through the weighted Neumann problem. This separation allows the forward and inverse theories to be developed using Hilbert space techniques that are often not available when working with probability measures; other geometries, such as Wasserstein and Fisher-Rao, construct infinite-dimensional Riemannian manifolds of probability measures rather than vector spaces.

The remainder of the paper follows a forward-to-inverse progression. Section~\ref{sec:bayes-hilbert-spaces} recalls basic facts of Bayes Hilbert spaces, the $\clr$ map, and its inverse. Section~\ref{sec:intrinsic-forward-geometry} develops the canonical forward realization and the transport form \(\mathfrak g_h\). Section~\ref{sec:ambient-inverse-problem} introduces and analyzes our inverse problem in the ambient infinite-dimensional setting. We introduce the regularized variational problem, prove local stability, introduce the observability and joint transport--observability forms, and identify the limitations of localized sensing in infinite dimensions. Section~\ref{sec:finite-dimensional-reduction} then passes to finite-dimensional Bayes Hilbert subspaces, proves the static- and moving-sensor recovery results, and lifts them to approximate ambient recovery.

\subsection{Related work}

Our Bayes Hilbert framework is related to the theory of metric gradient flows of probability measures \citep{ambrosioGradientFlowsMetric2008, santambrogio_euclidean_2017, chenHuangHuangReichStuart2023}, but the conceptual starting point of the present paper is different. We do not begin with a fixed
energy functional and then derive its steepest-descent evolution under a prescribed metric.
Instead, we begin with an \emph{arbitrary prescribed path} $t \mapsto h(t)$ in Bayes Hilbert
coordinates and then solve a weighted Neumann problem to recover the unique gradient velocity
field of minimum kinetic energy that realizes this path in the continuity equation. In this sense,
our transport form $\mathfrak g_h$ is an execution geometry induced by dynamical realization of a prescribed
log-density path, rather than a Riemannian metric used to define a gradient flow of a fixed
functional. The inverse problem is likewise formulated directly on path space, with observability
encoded by the companion form $\mathfrak j_h$; we are not aware of this path-space forward--inverse pairing having a direct analogue in the gradient-flow literature.

Some gradient flows, particularly Fisher--Rao gradient flows, arise as special cases of our framework. The geometric annealing path
\[
    \rho_t \propto \rho_0^{\,1-t}\rho_1^{\,t}
\]
appearing in Fisher--Rao-based sampling and continuation methods is simply a straight-line in
Bayes Hilbert coordinates: if $h_i=\clr(\rho_i)$, then
\[
    \clr(\rho_t) = (1-t)h_0 + t h_1.
\]
Thus the Fisher--Rao annealing path is contained in our framework as a distinguished special case.
This observation is consistent with recent work showing that geometric annealing has a
Fisher--Rao gradient-flow interpretation and can be dynamically realized by solving an elliptic or
Poisson-type equation for a transport velocity
\citep{maurais2024SamplingInUnitTime,domingoEnrichPooladian2023, taghvaei2023surveyfeedbackparticlefilter, Reich2011ADS}. In Section~3, we extend these techniques to paths that are not straight-lines in Bayes Hilbert space.

The relation to Wasserstein--Fisher--Rao (WFR), also called Hellinger--Kantorovich (HK), is different in a more fundamental way.
The WFR/HK geometry is an
\emph{unbalanced} transport theory on nonnegative measures: it interpolates between quadratic
Wasserstein transport and Fisher--Rao reaction, and its dynamic formulation allows source terms in
the continuity equation
\citep{chizatPeyreSchmitzerVialard2018Interpolating,
    chizatPeyreSchmitzerVialard2018Unbalanced,
    lieroMielkeSavare2018}. By contrast, our present theory is formulated on normalized probability
measures and, once a Bayes Hilbert path has been chosen, produces a \emph{conservative}
continuity equation
\[
    \partial_t \rho_t + \nabla \cdot (\rho_t v_t) = 0.
\]
For this reason, our framework should not be viewed as a special case of WFR/HK gradient-flow
theory. Rather, it provides a complementary log-density coordinate formalism for path design,
minimum-energy dynamical realization, and inverse reconstruction on spaces of probability laws.

Bayes Hilbert spaces are closely related to the field of information geometry \citep{Amari2016,AmariNagaoka2000,AyJostLeSchwachhoefer2017}, but they encode
a different geometric choice. In information geometry, a family of probability distributions is
treated as a statistical manifold equipped with the Fisher metric and a dual pair of affine
connections, with exponential and mixture coordinates playing a central role. In Bayes Hilbert
space, by contrast, one fixes a reference measure and represents a law by its centered
log-density, thereby obtaining a global Hilbert-space model in which addition is Bayes
updating and affine subspaces correspond to exponential families. The common thread is the privileged
role of log-densities and exponential families; the difference is that information geometry is
primarily a local Riemannian differential geometry on statistical manifolds, whereas the
Bayes Hilbert approach is a global linear functional-analytic geometry. For a more thorough comparison between the two geometries, we refer to \citet{pistoneUnifiedApproachAitchisons2024}.

Recent machine-learning-oriented uses of Bayes Hilbert spaces remain comparatively limited, but they point in several distinct directions. Barfoot and D'Eleuterio formulate variational inference in a Bayesian Hilbert space and show that, under suitable conditions, KL-based variational inference can be interpreted as iterative Euclidean projection of the posterior onto a chosen approximation family; they emphasize applications to high-dimensional robotic state estimation and sparsity-aware inference \citep{barfootVariationalInferenceIterative2023}. Wynne studies Bayes Hilbert spaces as a framework for posterior approximation, highlighting connections to Bayesian coreset constructions and kernel-based discrepancies for comparing posterior and pseudo-posterior measures \citep{wynneBayesHilbertSpaces2023}. More recently, Lach, Fottner, and Okhrin introduce pseudo \(f\)-divergences that extend classical \(f\)-divergences to include the metric induced by Bayes Hilbert spaces, and they develop a variational estimation framework with a generative-adversarial interpretation, reporting strong empirical performance relative to standard \(f\)-GAN variants and competitive results against Wasserstein GANs \citep{lachBridgingGapFdivergences2025}.

The literature already contains several nearby formulations of inverse problems for measures, though not, to our knowledge, the particular Bayes Hilbert path-recovery problem studied here. The closest precedent on the inverse-problems side is the work of Bredies and Fanzon on dynamic inverse problems in spaces of measures, where the unknown is a time-dependent curve of Radon measures and reconstruction is regularized by balanced or unbalanced dynamic optimal transport \citep{bredies2020dynamic}. In a different but clearly related direction, Li, Oprea, Wang, and Yang study stochastic inverse problems in which the unknown is itself a probability law, and subsequently formulate inverse problems directly over probability measure space through pushforward constraints; these works are static rather than path-valued, but they place inverse problems for distributions on a rigorous infinite-dimensional footing \citep{li2024stochastic,li2025measure}. There is also a neighboring line of work on recovering dynamics from ensemble snapshot data, including system-identification formulations based on distributional evolution and, more recently, Schr\"odinger-bridge-based reconstruction from snapshot measurements \citep{aalto2019snapshot,morimoto2025linear}.
Relative to these works, our contribution is to formulate indirect recovery of a \emph{measure flow} in Bayes Hilbert coordinates, to couple that recovery with a canonical minimum-energy dynamical realization obtained from weighted Neumann problems, and to identify the joint transport--observability form as the object governing recoverability in both the ambient infinite-dimensional setting and in concrete finite-dimensional reductions.

% !TeX root = ../main.tex

\section{Bayes Hilbert spaces}
\label{sec:bayes-hilbert-spaces}

We provide a brief overview of the most important facts about Bayes Hilbert spaces for this work. Our goal is not to present the theory in full generality, but rather we tailor our presentation to streamline applicability to our later results. We refer to the original papers \citep{vandenboogaartBayesLinearSpaces2010,van_den_boogaart_bayes_2014} for additional details.

Bayes Hilbert spaces provide a linear coordinate model for strictly positive probability
measures.
Once a reference probability measure $\nu_0$ is
fixed, a probability measure can be encoded by its centered log-density using the $\clr$-transform
(Section~\ref{subsec:bayes-hilbert-coordinates}), and recovered by an exponentiation-normalization map
(Section~\ref{subsec:exp-normalization}).
This allows us to work with evolving measures through
ordinary function-valued paths (Section~\ref{subsec:regular-paths}).
The key computation of this section is the differential of the exponential-normalization map (Proposition~\ref{prop:diff_exp_normalization}), which produces the centered forcing term driving the weighted Neumann problems in Section~\ref{sec:intrinsic-forward-geometry}.

For the remainder of the paper, let $\Omega \subset \mathbb{R}^d$ be open and connected, and let \(\nu_0 \ll \d x\) be a fixed probability measure with \(\supp(\nu_0)=\overline{\Omega}\).
We define the \emph{centered $L^2(\nu_0)$ functions} by
\[
    L^2_0(\nu_0)
    :=
    \left\{
    h\in L^2(\nu_0) : \int_\Omega h\,\d\nu_0 = 0
    \right\}.
\]

\subsection{Bayes Hilbert coordinates}
\label{subsec:bayes-hilbert-coordinates}

Let \(\rho\) and \(\eta\) be positive measures on $\Omega$ that are each equivalent to \(\nu_0\) (that is, $\rho \ll \nu_0$ and $\nu_0 \ll \rho$, and similarly for $\eta$).
We say that \(\rho\) and \(\eta\) are \emph{Bayes-equivalent}, and write \(\rho\sim_B \eta\), if there exists \(c>0\) such that \(\rho=c\,\eta\). In particular, Bayes-equivalent measures are proportional, and the total mass of a measure is invisible to Bayes Hilbert methods.

We denote the Radon-Nikodym derivative of $\rho$ with respect to $\nu_0$ by $\frac{\d \rho}{\d \nu_0}$.

\begin{definition}[Bayes Hilbert space]
    The Bayes Hilbert space relative to \(\nu_0\) is
    \[
        B^2(\nu_0)
        :=
        \left\{
        \rho : \rho \ll \nu_0,\ \rho>0\ \nu_0\text{-a.e.},\
        \log \frac{\d\rho}{\d\nu_0}\in L^2(\nu_0)
        \right\}\Big/\sim_B.
    \]
\end{definition}

\begin{definition}[Centered log-ratio transform]
    For \(\rho\in B^2(\nu_0)\), define
    \[
        \clr(\rho)
        :=
        \log \frac{\d\rho}{\d\nu_0}
        -
        \int_\Omega \log \frac{\d\rho}{\d\nu_0}\,\d\nu_0.
    \]
\end{definition}

The \(\clr\) transform is well defined on Bayes-equivalence classes and takes values in
\(L^2_0(\nu_0)\). It is the basic coordinate map on \(B^2(\nu_0)\). With this in mind, we can use $\clr$ to pull back the inner-product from $L^2_0(\nu_0)$ to $B^2(\nu_0)$, making the following proposition true.

\begin{proposition}[Hilbert structure \citep{van_den_boogaart_bayes_2014}]
    The map
    \[
        \clr : B^2(\nu_0)\to L^2_0(\nu_0)
    \]
    is an isometric isomorphism between Hilbert spaces.
\end{proposition}

\begin{remark}
    The vector space operations on $B^2(\nu_0)$ are not the usual pointwise addition and scalar multiplication, but rather \emph{perturbation} and \emph{powering}:
    \begin{align*}
        & \rho \oplus \mu \text{ is represented by the density } \frac{\d \rho}{\d \nu_0} \frac{\d \mu}{\d \nu_0}, \qquad  \rho, \mu \in B^2(\nu_0) \\
        \qquad
        & \alpha \odot \rho \text{ is represented by the density } \left( \frac{\d \rho}{\d \nu_0} \right)^\alpha, \quad \alpha \in \mathbb{R}.
    \end{align*}
    The perturbation $\rho \oplus \mu$ corresponds to Bayesian updating---multiplying the density of $\rho$ by that of $\mu$---while the powering $\alpha \odot \rho$ raises the density to the power $\alpha$.
    For this manuscript, it is our preference to work in the $\clr$-transformed space $L^2_0(\nu_0)$ where the familiar addition and scalar multiplication apply, rather than directly in $B^2(\nu_0)$ with $\oplus$ and $\odot$.
\end{remark}

\begin{remark}
    In the linear structure of $B^2(\nu_0)$, the reference $\nu_0$ becomes the \emph{neutral} element (i.e. the zero vector). This fact is reflected in our choice of notation for $\nu_0$.
\end{remark}

Although \(B^2(\nu_0)\) is formally a space of Bayes-equivalence classes, and \(B^2(\nu_0)\) naturally includes both finite and infinite measures, in this paper we will work only with finite measures, and in particular with the unique probability representative of each equivalence class of finite measures.

\subsection{The exponential-normalization map}
\label{subsec:exp-normalization}

To pass from centered log-density coordinates back to probability measures, we introduce the
normalized exponential map. 
At the level of Bayes classes, the inverse of the $\clr$ transform is the exponential map \(h \mapsto e^h \nu_0\); however, $e^h \nu_0$ is not necessarily a finite measure. 
For simplicity, we therefore work in the bounded coordinate class
\[
    \mathcal{H} := L^2_0(\nu_0) \cap L^\infty(\nu_0)
\]
equipped with the $L^\infty(\nu_0)$-norm (since $\nu_0$ is a probability measure, this is stronger than the $L^2(\nu_0)$ norm). For every \(h \in \mathcal{H}\), the measure \(e^h \nu_0\) is finite and can be normalized to a probability
measure. This class is precisely the $\clr$-image of the bounded-density class
\citep{van_den_boogaart_bayes_2014}
\[
    B^2_b(\nu_0)
    :=
    \left\{
    \rho \in B^2(\nu_0) :
    \exists\, b>0 \text{ such that }
    \frac{1}{b}
    <
    \frac{\d\rho}{\d\nu_0}
    <
    b
    \quad \nu_0\text{-a.e.}
    \right\}.
\]
Although \(B^2_b(\nu_0)\) does not contain all finite elements of \(B^2(\nu_0)\), it avoids technicalities when considering tangent directions for $h$ in the Propositions below, and it contains all measures satisfying the additional regularity assumptions required in Section~3 onwards, so it is not unduly restrictive.

\begin{definition}[Exponential-normalization map] For $h \in \mathcal{H}$, define
    \begin{equation}
        \label{eq:rho_h_def}
        \rho_h
        :=
        \frac{e^h}{\int_\Omega e^h\,\d\nu_0}\,\nu_0.
    \end{equation}
\end{definition}

The Radon-Nikodym derivative of $\rho_h$ is
\[
    \frac{\d\rho_h}{\d\nu_0}
    =
    \frac{e^h}{\int_\Omega e^h\,\d\nu_0}.
\]
By construction, \(\rho_h\) is a probability measure, \(\rho_h\ll \nu_0\), and $\clr(\rho_h)=h$. The differential of the exponential-normalization map is one of the basic objects used later.

\begin{proposition}[Fr\'echet differentiability of the exponential-normalization map]
    \label{prop:diff_exp_normalization}
    Define $\mathcal E:\mathcal H \to L^1(\nu_0)$, $\mathcal E(h):=\frac{\d\rho_h}{\d\nu_0}$.
    Then \(\mathcal E\) is Fr\'echet differentiable with derivative
    \begin{equation}
        \label{eq:diff_exp_normalization}
        D\mathcal E(h)[\xi]
        =
        \frac{\d\rho_h}{\d\nu_0}
        \bigl(\xi-\mathbb E_{\rho_h}[\xi]\bigr).
    \end{equation}
\end{proposition}

\begin{proof}
    Write
    \[
        \mathcal N(h):=e^h \in L^1(\nu_0),
        \qquad
        Z(h):=\int_\Omega e^h\,\d\nu_0 \in (0,\infty).
    \]
    Then
    \[
        \mathcal E(h)=\frac{\mathcal N(h)}{Z(h)}.
    \]

    We first show that \(\mathcal N:\mathcal H\to L^1(\nu_0)\) is Fr\'echet differentiable with
    \[
        D\mathcal N(h)[\xi]=e^h\xi.
    \]
    Indeed,
    \[
        \mathcal N(h+\xi)-\mathcal N(h)-e^h\xi
        =
        e^h\bigl(e^\xi-1-\xi\bigr).
    \]
    Using the elementary bound
    \[
        |e^r-1-r|\le \tfrac12 e^{|r|}r^2,
        \qquad r\in\mathbb R,
    \]
    we obtain
    \[
        \|\mathcal N(h+\xi)-\mathcal N(h)-e^h\xi\|_{L^1(\nu_0)}
        \le
        \frac12 e^{\|h\|_{L^\infty}+\|\xi\|_{L^\infty}}
        \|\xi\|_{L^\infty}^2
        =
        o(\|\xi\|_{L^\infty}).
    \]
    Here we used that \(\nu_0\) is a probability measure.

    Similarly,
    \[
        Z(h+\xi)-Z(h)-\int_\Omega e^h\xi\,\d\nu_0
        =
        \int_\Omega e^h\bigl(e^\xi-1-\xi\bigr)\,\d\nu_0,
    \]
    so \(Z:\mathcal H\to \mathbb R\) is Fr\'echet differentiable with
    \[
        DZ(h)[\xi]=\int_\Omega e^h\xi\,\d\nu_0.
    \]

    Since \(Z(h)>0\) and inversion on \((0,\infty)\) is smooth, the quotient rule yields that
    \(\mathcal E=\mathcal N/Z\) is Fr\'echet differentiable, with
    \[
        D\mathcal E(h)[\xi]
        =
        \frac{e^h\xi}{Z(h)}
        -
        \frac{e^h}{Z(h)^2}\int_\Omega e^h\xi\,\d\nu_0
        =
        \frac{e^h}{Z(h)}
        \left(
            \xi
            -
            \frac{1}{Z(h)}\int_\Omega e^h\xi\,\d\nu_0
        \right)
        =
        \frac{\d\rho_h}{\d\nu_0}
        \bigl(\xi-\mathbb E_{\rho_h}[\xi]\bigr).
    \]
\end{proof}

\begin{remark}
    \label{rem:ambient_tangent_interpretation}
    The derivative formula \eqref{eq:diff_exp_normalization} shows that a coordinate perturbation
    \(\xi\) induces the density variation
    \[
        \delta \rho
        =
        \rho_h\bigl(\xi-\mathbb E_{\rho_h}[\xi]\bigr).
    \]
    Thus Bayes Hilbert tangent directions are automatically centered with respect to the current law.
    This centered forcing term will be the source term in the weighted Neumann problems considered
    later.
\end{remark}

% \begin{remark}
%     The exponentiation-normalization map \eqref{eq:rho_h_def} is the continuous distribution analog of the ``softmax'' function commonly used for discrete distributions in many neural network architectures.
% \end{remark}

\subsection{Regular paths in Bayes Hilbert space}
\label{subsec:regular-paths}

We now pass from single states to continuous paths $h(t)$ along a time interval $t \in [0, T]$, and we derive the basic time-differentiation formula that drives the transport theory in Section~3.

\begin{definition}[Regular coordinate path]
    A \emph{regular coordinate path} is a map \(h\in C^1([0,T]; \mathcal H)\).
    For such a path we define $\rho_t := \rho_{h(t)}.$
\end{definition}

\begin{proposition}[Log-density evolution along coordinate paths]
    \label{prop:ambient_log_density_evolution}
    Let \(h\in C^1([0,T];\mathcal H)\) be a regular coordinate path.
    Then
    \begin{equation}
        \label{eq:ambient_density_derivative}
        \partial_t \frac{\d\rho_t}{\d\nu_0}
        =
        \frac{\d\rho_t}{\d\nu_0}
        \Bigl(\dot h(t)-\mathbb E_{\rho_t}[\dot h(t)]\Bigr),
    \end{equation}
    and hence
    \begin{equation}
        \label{eq:ambient_log_density_derivative}
        \partial_t \log \frac{\d\rho_t}{\d\nu_0}
        =
        \dot h(t)-\mathbb E_{\rho_t}[\dot h(t)].
    \end{equation}
\end{proposition}

\begin{proof}
    Equation \eqref{eq:ambient_density_derivative} follows from Proposition~\ref{prop:diff_exp_normalization} and the Banach space chain rule.
    Equation \eqref{eq:ambient_log_density_derivative} is obtained by dividing \eqref{eq:ambient_density_derivative} by \(\d\rho_t/\d\nu_0\), which is positive \(\nu_0\)-a.e.
\end{proof}

% \begin{remark}[State variable and probability measure]
%     \label{rem:state_vs_measure}
%     In what follows, \(h\) denotes the Bayes Hilbert coordinate of the state, while \(\rho_h\) denotes the associated probability measure. 
%     Thus the primary unknown is the function-valued path \(h(\cdot)\), and the corresponding path of measures is obtained by the exponential-normalization map.
% \end{remark}

% !TeX root = ../main.tex

\section{Transport realization of regular Bayes Hilbert paths}
\label{sec:intrinsic-forward-geometry}

In this section we develop the forward geometry of regular Bayes Hilbert paths in a form
adapted to the bounded-domain setting used later in the paper. The guiding principle is classical
in optimal transport: an infinitesimal density variation should be realized by the velocity field of
least kinetic energy among all solutions of the continuity equation. The minimum-energy realization is in the spirit of the Benamou-Brenier/Otto dynamic viewpoint on transport \citep{villaniOptimalTransport2009}, the difference here being the entire law path is prescribed, rather than just the endpoints.
Our contribution here is to show that Bayes Hilbert tangent directions induce such minimum-energy realizations through a state-dependent weighted elliptic problem.  Closely related results in the context of Reproducing Kernel Hilbert Spaces can also be found in the recent work \citet{Nakano2026ContinuumMarginal}.

\subsection{Abstract weighted setting}

For any finite measure \(\mu\ll \d x\), define
\[
    H^1(\Omega;\mu)
    :=
    \left\{
    u\in H^1_{\mathrm{loc}}(\Omega):
    u\in L^2(\mu),\ \nabla u\in L^2(\mu;\mathbb R^d)
    \right\},
\]
with norm
\[
    \|u\|_{H^1(\Omega;\mu)}^2
    :=
    \|u\|_{L^2(\mu)}^2+\|\nabla u\|_{L^2(\mu)}^2.
\]
We also define the weighted mean-zero space
\[
    H^1_{\mu,\diamond}(\Omega)
    :=
    \left\{
    u\in H^1(\Omega;\mu): \int_\Omega u\,\d\mu =0
    \right\}.
\]

\begin{assumption}[Reference spectral gap and strong state space]
    \label{ass:abstract-forward-reference}
    Assume:
    \begin{enumerate}
        \item there exists \(C_{\nu_0}>0\) such that
              \begin{equation}
                  \label{eq:reference-poincare}
                  \int_\Omega
                  \left|u-\mathbb E_{\nu_0}[u]\right|^2\,\d\nu_0
                  \le
                  C_{\nu_0}
                  \int_\Omega |\nabla u|^2\,\d\nu_0
                  \qquad
                  \forall u\in H^1(\Omega;\nu_0);
              \end{equation}
        \item \(\mathcal Z\) is a Banach space continuously embedded in
              \(L^\infty(\nu_0)\cap L^2(\nu_0)\).
        \item \(C_c^\infty(\Omega)\) is dense in \(L^2(\nu_0)\).
    \end{enumerate}
\end{assumption}

\begin{assumption}[Admissible state class]
    \label{ass:ambient-admissible-class}
    Let \(\mathcal X_{\mathrm{ad}}\subset \mathcal Z\cap L^2_0(\nu_0)\). Assume that there exist
    constants \(0<c<C<\infty\) such that, for every \(h\in\mathcal X_{\mathrm{ad}}\), the density
    \[
        w_h:=\frac{\d\rho_h}{\d\nu_0}
        =
        \frac{e^h}{\int_\Omega e^h\,\d\nu_0}
    \]
    satisfies
    \begin{equation}
        \label{eq:ambient-uniform-density-bounds}
        c\le w_h(x)\le C
        \qquad\text{for }\nu_0\text{-a.e. }x\in\Omega.
    \end{equation}
\end{assumption}

\begin{remark}
    \label{rem:ambient-admissible-implies-linfty}
    The admissibility condition \eqref{eq:ambient-uniform-density-bounds} implies a uniform
    \(L^\infty(\nu_0)\)-bound on \(h\). Indeed,
    \[
        \log w_h
        =
        h-\log\!\left(\int_\Omega e^h\,\d\nu_0\right),
    \]
    and since \(h\in L^2_0(\nu_0)\),
    \[
        h=\log w_h-\mathbb E_{\nu_0}[\log w_h].
    \]
    Hence
    \[
        \|h\|_{L^\infty(\nu_0)}
        \le
        2\max\{|\log c|,\;|\log C|\}.
    \]
    Moreover, on the uniformly bounded \(L^\infty(\nu_0)\)-set \(\mathcal X_{\mathrm{ad}}\), the map
    \[
        h\longmapsto w_h=\frac{e^h}{\int_\Omega e^h\,\d\nu_0}
    \]
    is Lipschitz from \(L^\infty(\nu_0)\) to \(L^\infty(\nu_0)\). Since
    \(\mathcal Z\hookrightarrow L^\infty(\nu_0)\) continuously, it follows that
    \begin{equation}
        \label{eq:ambient-weight-lipschitz}
        \|w_{h_1}-w_{h_2}\|_{L^\infty(\nu_0)}
        \le
        M_{\mathcal X_{\mathrm{ad}}}\|h_1-h_2\|_{\mathcal Z}
        \qquad
        \forall h_1,h_2\in\mathcal X_{\mathrm{ad}}
    \end{equation}
    for some constant \(M_{\mathcal X_{\mathrm{ad}}}>0\).
\end{remark}

\begin{lemma}[Equivalence of weighted Sobolev norms]
    \label{lem:ambient-weighted-sobolev-equivalence}
    Under Assumption~\ref{ass:ambient-admissible-class}, for every
    \(h\in\mathcal X_{\mathrm{ad}}\),
    \[
        H^1(\Omega;\rho_h)=H^1(\Omega;\nu_0)
    \]
    as sets, and the norms are uniformly equivalent:
    \begin{equation}
        \label{eq:ambient-weighted-sobolev-equivalence}
        c\|u\|_{H^1(\Omega;\nu_0)}^2
        \le
        \|u\|_{H^1(\Omega;\rho_h)}^2
        \le
        C\|u\|_{H^1(\Omega;\nu_0)}^2
        \qquad
        \forall u\in H^1(\Omega;\nu_0).
    \end{equation}
    In particular,
    \[
        H^1_{\rho_h,\diamond}(\Omega)
        =
        \left\{
        u\in H^1(\Omega;\nu_0):\int_\Omega u\,\d\rho_h=0
        \right\}.
    \]
\end{lemma}

\begin{proof}
    Since \(\d\rho_h=w_h\,\d\nu_0\) and \(c\le w_h\le C\) \(\nu_0\)-a.e., we have
    \[
        c\int_\Omega |u|^2\,\d\nu_0
        \le
        \int_\Omega |u|^2\,\d\rho_h
        \le
        C\int_\Omega |u|^2\,\d\nu_0,
    \]
    and similarly
    \[
        c\int_\Omega |\nabla u|^2\,\d\nu_0
        \le
        \int_\Omega |\nabla u|^2\,\d\rho_h
        \le
        C\int_\Omega |\nabla u|^2\,\d\nu_0.
    \]
    Summing the two inequalities yields \eqref{eq:ambient-weighted-sobolev-equivalence}.
\end{proof}

\begin{proposition}[Weighted spectral gap inherited from the reference measure]
    \label{prop:ambient-weighted-poincare}
    Under Assumptions~\ref{ass:abstract-forward-reference} and
    \ref{ass:ambient-admissible-class}, for every \(h\in\mathcal X_{\mathrm{ad}}\) and every
    \(u\in H^1(\Omega;\rho_h)\),
    \begin{equation}
        \label{eq:ambient-weighted-poincare}
        \int_\Omega
        \left|u-\mathbb E_{\rho_h}[u]\right|^2\,\d\rho_h
        \le
        \frac{C}{c}\,C_{\nu_0}
        \int_\Omega |\nabla u|^2\,\d\rho_h.
    \end{equation}
    Consequently, the weighted Dirichlet form
    \[
        u\longmapsto \int_\Omega |\nabla u|^2\,\d\rho_h
    \]
    is coercive on \(H^1_{\rho_h,\diamond}(\Omega)\), with constants uniform in
    \(h\in\mathcal X_{\mathrm{ad}}\).
\end{proposition}

\begin{proof}
    Fix \(h\in\mathcal X_{\mathrm{ad}}\) and \(u\in H^1(\Omega;\rho_h)\). Since
    \(u\mapsto \int |u-a|^2\,\d\rho_h\) is minimized at \(a=\mathbb E_{\rho_h}[u]\), for every
    constant \(a\in\mathbb R\),
    \[
        \int_\Omega
        \left|u-\mathbb E_{\rho_h}[u]\right|^2\,\d\rho_h
        \le
        \int_\Omega |u-a|^2\,\d\rho_h.
    \]
    Choosing \(a=\mathbb E_{\nu_0}[u]\) and using \(\d\rho_h=w_h\,\d\nu_0\), together with
    \eqref{eq:ambient-uniform-density-bounds}, we obtain
    \begin{align*}
        \int_\Omega
        \left|u-\mathbb E_{\rho_h}[u]\right|^2\,\d\rho_h
         & \le
        \int_\Omega \left|u-\mathbb E_{\nu_0}[u]\right|^2\,\d\rho_h \\
         & \le
        C
        \int_\Omega \left|u-\mathbb E_{\nu_0}[u]\right|^2\,\d\nu_0  \\
         & \le
        C C_{\nu_0}\int_\Omega |\nabla u|^2\,\d\nu_0                \\
         & \le
        \frac{C}{c}\,C_{\nu_0}\int_\Omega |\nabla u|^2\,\d\rho_h,
    \end{align*}
    which proves \eqref{eq:ambient-weighted-poincare}.
\end{proof}

\subsection{Canonical weighted Neumann solve and minimum-energy realization}

The Bayes Hilbert tangent direction \(\xi\in L^2_0(\nu_0)\) induces the centered density
variation
\[
    \rho_h\bigl(\xi-\mathbb E_{\rho_h}[\xi]\bigr),
\]
which is the forcing term in the weighted elliptic problem below.

\begin{theorem}[Canonical Neumann potential]
    \label{thm:ambient-neumann}
    Assume Assumptions~\ref{ass:abstract-forward-reference} and
    \ref{ass:ambient-admissible-class}. Fix \(h\in \mathcal X_{\mathrm{ad}}\) and
    \(\xi\in L^2_0(\nu_0)\). Then there exists a unique
    \[
        \psi_{h,\xi}\in H^1_{\rho_h,\diamond}(\Omega)
    \]
    such that
    \begin{equation}
        \label{eq:ambient-neumann-weak}
        \int_\Omega \nabla\psi_{h,\xi}\cdot\nabla\eta\,\d\rho_h
        =
        \int_\Omega
        \bigl(\xi-\mathbb E_{\rho_h}[\xi]\bigr)\eta\,\d\rho_h
        \qquad
        \forall \eta\in H^1_{\rho_h,\diamond}(\Omega).
    \end{equation}
    If, in addition, the density of \(\rho_h\) with respect to Lebesgue measure is regular enough,
    then \eqref{eq:ambient-neumann-weak} may be interpreted as the weak form of
    \begin{equation}
        \label{eq:ambient-neumann-strong}
        -\nabla\cdot(\rho_h\nabla\psi_{h,\xi})
        =
        \rho_h\bigl(\xi-\mathbb E_{\rho_h}[\xi]\bigr)
        \qquad\text{in }\Omega,
    \end{equation}
    together with the natural zero-flux boundary condition
    \begin{equation}
        \label{eq:ambient-neumann-bc}
        \rho_h\,\partial_n\psi_{h,\xi}=0
        \qquad\text{on }\partial\Omega
    \end{equation}
    when \(\partial\Omega\neq\varnothing\).
\end{theorem}

\begin{proof}
    Define, on \(H^1_{\rho_h,\diamond}(\Omega)\),
    \[
        B_h(u,\eta):=\int_\Omega \nabla u\cdot \nabla\eta\,\d\rho_h,
        \qquad
        F_{h,\xi}(\eta)
        :=
        \int_\Omega
        \bigl(\xi-\mathbb E_{\rho_h}[\xi]\bigr)\eta\,\d\rho_h.
    \]
    The bilinear form \(B_h\) is continuous. By
    Proposition~\ref{prop:ambient-weighted-poincare},
    \[
        \|u\|_{L^2(\rho_h)}^2
        \le
        \frac{C}{c}\,C_{\nu_0}
        \int_\Omega |\nabla u|^2\,\d\rho_h
        \qquad
        \forall u\in H^1_{\rho_h,\diamond}(\Omega),
    \]
    so \(B_h\) is coercive on \(H^1_{\rho_h,\diamond}(\Omega)\).

    Next,
    \[
        \|\xi-\mathbb E_{\rho_h}[\xi]\|_{L^2(\rho_h)}
        \le
        2\|\xi\|_{L^2(\rho_h)}
        \le
        2\sqrt{C}\,\|\xi\|_{L^2(\nu_0)},
    \]
    and therefore, by Cauchy--Schwarz and Proposition~\ref{prop:ambient-weighted-poincare},
    \[
        |F_{h,\xi}(\eta)|
        \le
        \|\xi-\mathbb E_{\rho_h}[\xi]\|_{L^2(\rho_h)}
        \|\eta\|_{L^2(\rho_h)}
        \le
        C'\|\xi\|_{L^2(\nu_0)}\|\eta\|_{H^1(\Omega;\rho_h)}.
    \]
    Thus \(F_{h,\xi}\) is continuous on \(H^1_{\rho_h,\diamond}(\Omega)\), and the conclusion
    follows from the Lax--Milgram theorem. The strong form is the usual divergence-form
    interpretation of \eqref{eq:ambient-neumann-weak}.
\end{proof}

\begin{corollary}[Energy estimate for the weighted Neumann solve]
    \label{cor:ambient-neumann-energy}
    Under the assumptions of Theorem~\ref{thm:ambient-neumann}, there exists \(M>0\),
    depending only on \(c\), \(C\), and \(C_{\nu_0}\), such that
    \begin{equation}
        \label{eq:ambient-neumann-energy}
        \|\nabla\psi_{h,\xi}\|_{L^2(\rho_h)}
        \le
        M\|\xi\|_{L^2(\nu_0)}
        \qquad
        \forall h\in\mathcal X_{\mathrm{ad}},\ \forall \xi\in L^2_0(\nu_0).
    \end{equation}
\end{corollary}

\begin{proof}
    Take \(\eta=\psi_{h,\xi}\) in \eqref{eq:ambient-neumann-weak}. Then
    \[
        \int_\Omega |\nabla\psi_{h,\xi}|^2\,\d\rho_h
        =
        \int_\Omega
        \bigl(\xi-\mathbb E_{\rho_h}[\xi]\bigr)\psi_{h,\xi}\,\d\rho_h.
    \]
    By Cauchy--Schwarz and Proposition~\ref{prop:ambient-weighted-poincare},
    \[
        \int_\Omega |\nabla\psi_{h,\xi}|^2\,\d\rho_h
        \le
        \|\xi-\mathbb E_{\rho_h}[\xi]\|_{L^2(\rho_h)}
        \|\psi_{h,\xi}\|_{L^2(\rho_h)}
        \le
        M\|\xi\|_{L^2(\nu_0)}
        \|\nabla\psi_{h,\xi}\|_{L^2(\rho_h)}.
    \]
    Canceling the nonzero factor yields \eqref{eq:ambient-neumann-energy}.
\end{proof}

\begin{remark}[Extension of the weak formulation to arbitrary test functions]
    \label{rem:ambient-neumann-all-tests}
    Since both sides of \eqref{eq:ambient-neumann-weak} are unchanged when \(\eta\) is replaced
    by \(\eta-a\) for a constant \(a\in\mathbb R\), the weak formulation extends to all
    \(\eta\in H^1(\Omega;\rho_h)\):
    \begin{equation}
        \label{eq:ambient-neumann-all-tests}
        \int_\Omega \nabla\psi_{h,\xi}\cdot\nabla\eta\,\d\rho_h
        =
        \int_\Omega
        \bigl(\xi-\mathbb E_{\rho_h}[\xi]\bigr)\eta\,\d\rho_h
        \qquad
        \forall \eta\in H^1(\Omega;\rho_h).
    \end{equation}
\end{remark}

\begin{proposition}[Minimum-energy characterization]
    \label{prop:ambient-min-energy}
    Fix \(h\in\mathcal X_{\mathrm{ad}}\) and \(\xi\in L^2_0(\nu_0)\). Then \(\nabla\psi_{h,\xi}\) is the
    unique minimizer of
    \begin{equation}
        \label{eq:ambient-min-energy}
        \inf_{v\in\mathcal A_{h,\xi}}
        \int_\Omega |v|^2\,\d\rho_h,
    \end{equation}
    where
    \begin{equation}
        \label{eq:ambient-admissible-velocities}
        \mathcal A_{h,\xi}
        :=
        \left\{
        v\in L^2(\rho_h;\mathbb R^d):
        \int_\Omega v\cdot\nabla\eta\,\d\rho_h
        =
        \int_\Omega
        \bigl(\xi-\mathbb E_{\rho_h}[\xi]\bigr)\eta\,\d\rho_h
        \ \forall \eta\in H^1(\Omega;\rho_h)
        \right\}.
    \end{equation}
\end{proposition}

\begin{proof}
    By Remark~\ref{rem:ambient-neumann-all-tests},
    \[
        \int_\Omega \nabla\psi_{h,\xi}\cdot\nabla\eta\,\d\rho_h
        =
        \int_\Omega
        \bigl(\xi-\mathbb E_{\rho_h}[\xi]\bigr)\eta\,\d\rho_h
        \qquad
        \forall \eta\in H^1(\Omega;\rho_h),
    \]
    so \(\nabla\psi_{h,\xi}\in\mathcal A_{h,\xi}\).

    Now let \(v\in\mathcal A_{h,\xi}\). Taking \(\eta=\psi_{h,\xi}\) gives
    \[
        \int_\Omega \bigl(v-\nabla\psi_{h,\xi}\bigr)\cdot\nabla\psi_{h,\xi}\,\d\rho_h=0.
    \]
    Hence
    \begin{align*}
        \int_\Omega |v|^2\,\d\rho_h
         & =
        \int_\Omega |\nabla\psi_{h,\xi}|^2\,\d\rho_h
        +
        \int_\Omega |v-\nabla\psi_{h,\xi}|^2\,\d\rho_h \\
         & \ge
        \int_\Omega |\nabla\psi_{h,\xi}|^2\,\d\rho_h,
    \end{align*}
    with equality if and only if \(v=\nabla\psi_{h,\xi}\) in \(L^2(\rho_h;\mathbb R^d)\).
\end{proof}

\subsection{Canonical transport map and the intrinsic transport form}

The weighted Neumann solve defines a canonical map from Bayes Hilbert tangent directions to
minimum-energy velocity fields.

\begin{definition}[Canonical transport map]
    \label{def:ambient-transport-map}
    For \(h\in\mathcal X_{\mathrm{ad}}\), define
    \[
        \mathcal T_h : L^2_0(\nu_0)\to L^2(\rho_h;\mathbb R^d),
        \qquad
        \mathcal T_h\xi := \nabla\psi_{h,\xi}.
    \]
\end{definition}

\begin{definition}[Intrinsic transport form]
    \label{def:intrinsic-transport-form}
    For \(h\in\mathcal X_{\mathrm{ad}}\) and \(\xi,\zeta\in L^2_0(\nu_0)\), define
    \begin{equation}
        \label{eq:intrinsic-transport-form}
        \mathfrak g_h(\xi,\zeta)
        :=
        \int_\Omega \mathcal T_h\xi\cdot \mathcal T_h\zeta\,\d\rho_h
        =
        \int_\Omega \nabla\psi_{h,\xi}\cdot\nabla\psi_{h,\zeta}\,\d\rho_h.
    \end{equation}
\end{definition}

\begin{proposition}[Basic properties of \texorpdfstring{$\mathfrak g_h$}{g\_h}]
    \label{prop:ambient-transport-form-properties}
    For each \(h\in\mathcal X_{\mathrm{ad}}\), the form \(\mathfrak g_h\) is a symmetric,
    nonnegative bilinear form on \(L^2_0(\nu_0)\). Moreover,
    \[
        \mathfrak g_h(\xi,\xi)=0
        \quad\Longleftrightarrow\quad
        \xi=\mathbb E_{\rho_h}[\xi]
        \quad \rho_h\text{-a.e.}
    \]
    and hence, for \(\xi\in L^2_0(\nu_0)\),
    \[
        \mathfrak g_h(\xi,\xi)=0
        \quad\Longleftrightarrow\quad
        \xi=0
        \qquad \nu_0\text{-a.e.}
    \]
\end{proposition}

\begin{proof}
    Bilinearity and symmetry follow from the linearity of \(\xi\mapsto \psi_{h,\xi}\) and the
    symmetry of the \(L^2(\rho_h)\) inner product. Nonnegativity is immediate from
    \[
        \mathfrak g_h(\xi,\xi)=\int_\Omega |\nabla\psi_{h,\xi}|^2\,\d\rho_h.
    \]

    Assume \(\mathfrak g_h(\xi,\xi)=0\). Then \(\nabla\psi_{h,\xi}=0\) \(\rho_h\)-a.e., hence
    \(\psi_{h,\xi}=0\) in \(H^1_{\rho_h,\diamond}(\Omega)\). By
    \eqref{eq:ambient-neumann-all-tests},
    \[
        \int_\Omega
        \bigl(\xi-\mathbb E_{\rho_h}[\xi]\bigr)\eta\,\d\rho_h
        =0
        \qquad
        \forall \eta\in H^1(\Omega;\rho_h).
    \]
    In particular, the same holds for every
    \[
        \eta=\varphi-\mathbb E_{\rho_h}[\varphi],
        \qquad
        \varphi\in C_c^\infty(\Omega).
    \]
    Since \(C_c^\infty(\Omega)\) is dense in \(L^2(\nu_0)\) by
    Assumption~\ref{ass:abstract-forward-reference}, and since
    \(\d\rho_h=w_h\,\d\nu_0\) with \(w_h\) bounded above and below by
    Assumption~\ref{ass:ambient-admissible-class}, it is also dense in
    \(L^2(\rho_h)\). Hence the set
    \[
        \left\{
        \varphi-\mathbb E_{\rho_h}[\varphi]:\varphi\in C_c^\infty(\Omega)
        \right\}
    \]
    is dense in \(L^2_0(\rho_h)\). Therefore
    \[
        \xi-\mathbb E_{\rho_h}[\xi]=0
        \qquad
        \rho_h\text{-a.e.}
    \]
    The converse implication is immediate from \eqref{eq:intrinsic-transport-form}.

    Finally, since \(\rho_h\) and \(\nu_0\) are equivalent and \(\xi\in L^2_0(\nu_0)\), the identity
    \(\xi=\mathbb E_{\rho_h}[\xi]\) \(\rho_h\)-a.e.\ implies that \(\xi\) is almost everywhere constant,
    hence zero by the \(\nu_0\)-mean-zero constraint.
\end{proof}

\begin{proposition}[Stability of the weighted Neumann solve]
    \label{prop:ambient-neumann-stability}
    Assume Assumptions~\ref{ass:abstract-forward-reference} and
    \ref{ass:ambient-admissible-class}. Then there exists \(M>0\), depending only on
    \(\nu_0\), \(c\), \(C\), and the embedding \(\mathcal Z\hookrightarrow L^\infty(\nu_0)\), such
    that for all \(h_1,h_2\in\mathcal X_{\mathrm{ad}}\) and all \(\xi_1,\xi_2\in L^2_0(\nu_0)\),
    \begin{equation}
        \label{eq:ambient-neumann-stability-estimate}
        \inf_{a\in\mathbb R}
        \|\psi_{h_1,\xi_1}-\psi_{h_2,\xi_2}-a\|_{H^1(\Omega;\nu_0)}
        \le
        M\Big(
        \|\xi_1-\xi_2\|_{L^2(\nu_0)}
        +
        \|h_1-h_2\|_{\mathcal Z}
        \bigl(\|\xi_1\|_{L^2(\nu_0)}+\|\xi_2\|_{L^2(\nu_0)}\bigr)
        \Big).
    \end{equation}
    In particular,
    \begin{equation}
        \label{eq:ambient-neumann-gradient-stability}
        \|\nabla\psi_{h_1,\xi_1}-\nabla\psi_{h_2,\xi_2}\|_{L^2(\nu_0)}
        \le
        M\Big(
        \|\xi_1-\xi_2\|_{L^2(\nu_0)}
        +
        \|h_1-h_2\|_{\mathcal Z}
        \bigl(\|\xi_1\|_{L^2(\nu_0)}+\|\xi_2\|_{L^2(\nu_0)}\bigr)
        \Big).
    \end{equation}
\end{proposition}

\begin{proof}
    Fix \(h_1,h_2\in\mathcal X_{\mathrm{ad}}\) and \(\xi_1,\xi_2\in L^2_0(\nu_0)\). Write
    \[
        w_i:=\frac{\d\rho_{h_i}}{\d\nu_0},
        \qquad
        q_i:=\xi_i-\mathbb E_{\rho_{h_i}}[\xi_i],
        \qquad
        \psi_i:=\psi_{h_i,\xi_i},
        \qquad i=1,2.
    \]
    By \eqref{eq:ambient-weight-lipschitz},
    \[
        \|w_1-w_2\|_{L^\infty(\nu_0)}
        \le
        M_{\mathcal X_{\mathrm{ad}}}\|h_1-h_2\|_{\mathcal Z}.
    \]
    Also,
    \[
        \|q_1-q_2\|_{L^2(\nu_0)}
        \le
        \|\xi_1-\xi_2\|_{L^2(\nu_0)}
        +
        \left|
        \mathbb E_{\rho_{h_1}}[\xi_1]-\mathbb E_{\rho_{h_2}}[\xi_2]
        \right|.
    \]
    The second term is bounded by
    \begin{align*}
        \left|
        \int_\Omega \xi_1(w_1-w_2)\,\d\nu_0
        \right|
        +
        \left|
        \int_\Omega (\xi_1-\xi_2)w_2\,\d\nu_0
        \right|
         & \le
        \|\xi_1\|_{L^2(\nu_0)}\|w_1-w_2\|_{L^2(\nu_0)}
        +
        \sqrt{C}\,\|\xi_1-\xi_2\|_{L^2(\nu_0)} \\
         & \le
        M\Big(
        \|\xi_1-\xi_2\|_{L^2(\nu_0)}
        +
        \|h_1-h_2\|_{\mathcal Z}\|\xi_1\|_{L^2(\nu_0)}
        \Big).
    \end{align*}
    Hence
    \begin{equation}
        \label{eq:ambient-centered-forcing-stability}
        \|q_1-q_2\|_{L^2(\nu_0)}
        \le
        M\Big(
        \|\xi_1-\xi_2\|_{L^2(\nu_0)}
        +
        \|h_1-h_2\|_{\mathcal Z}
        \bigl(\|\xi_1\|_{L^2(\nu_0)}+\|\xi_2\|_{L^2(\nu_0)}\bigr)
        \Big).
    \end{equation}

    Let \(\delta\psi:=\psi_1-\psi_2\). Subtract the weak formulations:
    \[
        \int_\Omega w_1\nabla\delta\psi\cdot\nabla\eta\,\d\nu_0
        =
        \int_\Omega (w_1q_1-w_2q_2)\eta\,\d\nu_0
        -
        \int_\Omega (w_1-w_2)\nabla\psi_2\cdot\nabla\eta\,\d\nu_0
    \]
    for all \(\eta\in H^1(\Omega;\nu_0)\). Taking
    \[
        \eta:=\delta\psi-\mathbb E_{\nu_0}[\delta\psi]
    \]
    and using \eqref{eq:reference-poincare}, \eqref{eq:ambient-neumann-energy},
    \eqref{eq:ambient-centered-forcing-stability}, and the uniform bounds on the weights yields
    \[
        \|\nabla\delta\psi\|_{L^2(\nu_0)}
        \le
        M\Big(
        \|\xi_1-\xi_2\|_{L^2(\nu_0)}
        +
        \|h_1-h_2\|_{\mathcal Z}
        \bigl(\|\xi_1\|_{L^2(\nu_0)}+\|\xi_2\|_{L^2(\nu_0)}\bigr)
        \Big),
    \]
    which proves \eqref{eq:ambient-neumann-gradient-stability}. Subtracting the
    \(\nu_0\)-mean and using \eqref{eq:reference-poincare} gives
    \eqref{eq:ambient-neumann-stability-estimate}.
\end{proof}

\begin{corollary}[Continuity of the transport map in the state variable]
    \label{cor:ambient-transport-map-hsprime}
    Assume Assumptions~\ref{ass:abstract-forward-reference} and
    \ref{ass:ambient-admissible-class}. Then there exists \(M>0\) such that for all
    \(h_1,h_2\in\mathcal X_{\mathrm{ad}}\),
    \[
        \|\mathcal T_{h_1}-\mathcal T_{h_2}\|_{\mathcal L(L^2_0(\nu_0),L^2(\nu_0;\mathbb R^d))}
        \le
        M\|h_1-h_2\|_{\mathcal Z}.
    \]
    In particular, if \(h_n\to h\) in \(\mathcal Z\), then
    \[
        \mathcal T_{h_n}\to \mathcal T_h
        \quad\text{in }\mathcal L(L^2_0(\nu_0),L^2(\nu_0;\mathbb R^d)).
    \]
\end{corollary}

\begin{proof}
    Apply Proposition~\ref{prop:ambient-neumann-stability} with \(\xi_1=\xi_2=\xi\), then take
    the supremum over \(\|\xi\|_{L^2(\nu_0)}\le 1\).
\end{proof}

\subsection{Linearization in the state variable}

To differentiate the weighted Neumann solve with respect to the Bayes Hilbert state, we
introduce one additional layer of regularity for admissible perturbation directions.

\begin{assumption}[Differentiable admissible directions]
    \label{ass:abstract-forward-differentiable}
    There exists a Banach space \(\mathcal X\) continuously embedded in \(\mathcal Z\), with
    \[
        \mathcal X \hookrightarrow L^\infty(\nu_0)\cap L^2_0(\nu_0)
    \]
    continuously, such that for every \(h,\eta\in\mathcal X\) with
    \(h\in\mathcal X_{\mathrm{ad}}\), there exists \(\varepsilon_0>0\) for which
    \[
        h+\varepsilon\eta\in\mathcal X_{\mathrm{ad}}
        \qquad
        \text{for all }|\varepsilon|<\varepsilon_0.
    \]
\end{assumption}

\begin{definition}[Weighted covariance]
    \label{def:weighted-covariance}
    For \(h\in \mathcal X_{\mathrm{ad}}\) and \(f,g\in L^2(\rho_h)\), define
    \[
        \operatorname{Cov}_{\rho_h}(f,g)
        :=
        \int_\Omega
        \bigl(f-\mathbb E_{\rho_h}[f]\bigr)
        \bigl(g-\mathbb E_{\rho_h}[g]\bigr)\,\d\rho_h.
    \]
\end{definition}

\begin{lemma}[Directional derivative of the normalized weight]
    \label{lem:ambient-weight-derivative}
    Assume Assumption~\ref{ass:abstract-forward-differentiable}. Let \(h,\eta\in\mathcal X\),
    with \(h\in\mathcal X_{\mathrm{ad}}\), and assume that
    \(h+\varepsilon\eta\in\mathcal X_{\mathrm{ad}}\) for all \(|\varepsilon|<\varepsilon_0\). Define
    \[
        a_h(\eta):=\eta-\mathbb E_{\rho_h}[\eta].
    \]
    Then
    \begin{equation}
        \label{eq:ambient-weight-derivative}
        \frac{w_{h+\varepsilon\eta}-w_h}{\varepsilon}
        \to
        w_h\,a_h(\eta)
        \qquad\text{in }L^\infty(\nu_0)
    \end{equation}
    as \(\varepsilon\to 0\). Consequently, for every \(\xi\in L^2(\nu_0)\),
    \begin{equation}
        \label{eq:ambient-expectation-derivative}
        \frac{\mathbb E_{\rho_{h+\varepsilon\eta}}[\xi]-\mathbb E_{\rho_h}[\xi]}{\varepsilon}
        \to
        \operatorname{Cov}_{\rho_h}(\eta,\xi).
    \end{equation}
\end{lemma}

\begin{proof}
    Write
    \[
        Z(h):=\int_\Omega e^h\,\d\nu_0,
        \qquad
        w_h=\frac{e^h}{Z(h)}.
    \]
    Since \(h,\eta\in L^\infty(\nu_0)\) and \(h+\varepsilon\eta\) remains in the uniformly bounded
    admissible set for \(|\varepsilon|<\varepsilon_0\), the map
    \[
        \varepsilon\longmapsto \frac{e^{h+\varepsilon\eta}}{Z(h+\varepsilon\eta)}
    \]
    is differentiable in \(L^\infty(\nu_0)\), with derivative
    \[
        w_h\bigl(\eta-\mathbb E_{\rho_h}[\eta]\bigr)
        =
        w_h\,a_h(\eta).
    \]
    This proves \eqref{eq:ambient-weight-derivative}. Then, for \(\xi\in L^2(\nu_0)\),
    \[
        \mathbb E_{\rho_{h+\varepsilon\eta}}[\xi]
        =
        \int_\Omega \xi\,w_{h+\varepsilon\eta}\,\d\nu_0,
    \]
    so \eqref{eq:ambient-weight-derivative} implies
    \begin{align*}
        \frac{\mathbb E_{\rho_{h+\varepsilon\eta}}[\xi]-\mathbb E_{\rho_h}[\xi]}{\varepsilon}
         & \to
        \int_\Omega \xi\,w_h a_h(\eta)\,\d\nu_0 \\
         & =
        \int_\Omega \xi\,a_h(\eta)\,\d\rho_h
        =
        \operatorname{Cov}_{\rho_h}(\eta,\xi),
    \end{align*}
    which is \eqref{eq:ambient-expectation-derivative}.
\end{proof}

\begin{proposition}[Linearization of the weighted Neumann solve]
    \label{prop:linearized-neumann}
    Assume Assumptions~\ref{ass:abstract-forward-reference},
    \ref{ass:ambient-admissible-class}, and
    \ref{ass:abstract-forward-differentiable}. Let \(h,\eta\in \mathcal X\) and
    \(\xi\in L^2_0(\nu_0)\), and assume that there exists \(\varepsilon_0>0\) such that
    \[
        h+\varepsilon \eta \in \mathcal X_{\mathrm{ad}}
        \qquad\text{for all } |\varepsilon|<\varepsilon_0.
    \]
    Then there exists a unique
    \[
        \chi_{h;\eta,\xi}\in H^1_{\rho_h,\diamond}(\Omega)
    \]
    such that
    \begin{align}
        \label{eq:linearized-neumann}
        \int_\Omega \nabla \chi_{h;\eta,\xi}\cdot \nabla \varphi\, \d\rho_h
         & =
        \int_\Omega
        \Big(
        a_h(\eta)\,q_h(\xi)
        -
        \operatorname{Cov}_{\rho_h}(\eta,\xi)
        \Big)\varphi\, \d\rho_h \notag \\
         & \quad
        -
        \int_\Omega
        a_h(\eta)
        \nabla \psi_{h,\xi}\cdot \nabla \varphi\, \d\rho_h
        \qquad
        \forall \varphi\in H^1(\Omega;\rho_h),
    \end{align}
    where
    \[
        a_h(\eta):=\eta-\mathbb E_{\rho_h}[\eta],
        \qquad
        q_h(\xi):=\xi-\mathbb E_{\rho_h}[\xi].
    \]
    Moreover, the gradient map is differentiable at \(\varepsilon=0\):
    \[
        \frac{\nabla\psi_{h+\varepsilon\eta,\xi}-\nabla\psi_{h,\xi}}{\varepsilon}
        \longrightarrow
        \nabla\chi_{h;\eta,\xi}
        \qquad\text{in } L^2(\nu_0;\mathbb R^d).
    \]
    Equivalently, \(\varepsilon\mapsto\psi_{h+\varepsilon\eta,\xi}\) is differentiable at
    \(\varepsilon=0\) in \(H^1(\Omega;\nu_0)/\mathbb R\).
\end{proposition}

\begin{proof}
    The right-hand side of \eqref{eq:linearized-neumann} is continuous on
    \(H^1(\Omega;\rho_h)\), so existence and uniqueness of \(\chi_{h;\eta,\xi}\) follow from
    Lax--Milgram exactly as in Theorem~\ref{thm:ambient-neumann}.

    Let
    \[
        w_\varepsilon:=w_{h+\varepsilon\eta},
        \qquad
        q_\varepsilon:=\xi-\mathbb E_{\rho_{h+\varepsilon\eta}}[\xi],
        \qquad
        \psi_\varepsilon:=\psi_{h+\varepsilon\eta,\xi}.
    \]
    Define the difference quotient
    \[
        \delta_\varepsilon
        :=
        \frac{\psi_\varepsilon-\psi_{h,\xi}}{\varepsilon}.
    \]
    Subtract the weak formulations for \(\psi_\varepsilon\) and \(\psi_{h,\xi}\), divide by
    \(\varepsilon\), and compare with the weak equation for \(\chi_{h;\eta,\xi}\). Using
    Lemma~\ref{lem:ambient-weight-derivative}, \eqref{eq:ambient-expectation-derivative},
    the boundedness of the weights, and the energy bound
    \eqref{eq:ambient-neumann-energy}, one obtains an error equation for
    \[
        r_\varepsilon:=\delta_\varepsilon-\chi_{h;\eta,\xi}
    \]
    of the form
    \[
        \int_\Omega \nabla r_\varepsilon\cdot\nabla\varphi\,\d\rho_{h+\varepsilon\eta}
        =
        R_\varepsilon(\varphi),
        \qquad
        \forall \varphi\in H^1(\Omega;\rho_{h+\varepsilon\eta}),
    \]
    where \(R_\varepsilon(\varphi)\to 0\) uniformly for \(\|\varphi\|_{H^1(\Omega;\rho_{h+\varepsilon\eta})}\le 1\).
    Taking \(\varphi=r_\varepsilon-\mathbb E_{\rho_{h+\varepsilon\eta}}[r_\varepsilon]\), applying
    Proposition~\ref{prop:ambient-weighted-poincare}, and using the equivalence of the
    \(\rho_{h+\varepsilon\eta}\)- and \(\nu_0\)-norms, we conclude that
    \[
        \|\nabla r_\varepsilon\|_{L^2(\nu_0)}\to 0.
    \]
    Subtracting an irrelevant constant and using \eqref{eq:reference-poincare} then yields
    \[
        \inf_{a\in\mathbb R}
        \|r_\varepsilon-a\|_{H^1(\Omega;\nu_0)}
        \to 0.
    \]
    In particular, \(\nabla r_\varepsilon\to 0\) in \(L^2(\nu_0;\mathbb R^d)\), which proves the
    gradient differentiability claim and the equivalent differentiability statement in
    \(H^1(\Omega;\nu_0)/\mathbb R\).
\end{proof}

\begin{corollary}[Directional differentiability of the transport form]
    \label{cor:directional-derivative-gh}
    Under the assumptions of Proposition~\ref{prop:linearized-neumann}, the map
    \[
        \varepsilon\longmapsto \mathfrak g_{h+\varepsilon\eta}(\xi,\zeta)
    \]
    is differentiable at \(\varepsilon=0\) for every \(\xi,\zeta\in L^2_0(\nu_0)\), with derivative
    \begin{align}
        \label{eq:directional-derivative-gh}
        D_h \mathfrak g_h[\eta](\xi,\zeta)
         & =
        \int_\Omega
        a_h(\eta)\,
        \nabla\psi_{h,\xi}\cdot \nabla\psi_{h,\zeta}\, \d\rho_h \notag \\
         & \quad
        +
        \int_\Omega \nabla\chi_{h;\eta,\xi}\cdot \nabla\psi_{h,\zeta}\, \d\rho_h
        +
        \int_\Omega \nabla\psi_{h,\xi}\cdot \nabla\chi_{h;\eta,\zeta}\, \d\rho_h.
    \end{align}
\end{corollary}

\begin{proof}
    By definition,
    \[
        \mathfrak g_h(\xi,\zeta)
        =
        \int_\Omega
        \nabla\psi_{h,\xi}\cdot \nabla\psi_{h,\zeta}\, \d\rho_h.
    \]
    Differentiate this identity with respect to \(h\) in the direction \(\eta\). The derivative of
    \(\d\rho_h=w_h\,\d\nu_0\) is given by Lemma~\ref{lem:ambient-weight-derivative}, while
    the derivatives of \(\psi_{h,\xi}\) and \(\psi_{h,\zeta}\) are given by
    Proposition~\ref{prop:linearized-neumann}. Passing to the limit in the resulting difference
    quotient yields \eqref{eq:directional-derivative-gh}.
\end{proof}

\begin{remark}[Pullback interpretation]
    \label{rem:ambient-pullback-interpretation}
    The bilinear form \(\mathfrak g_h\) may be viewed as a pullback of continuity-equation transport
    geometry to Bayes Hilbert coordinates. A tangent direction \(\xi\) first produces the signed
    density variation
    \[
        \rho_h\bigl(\xi-\mathbb E_{\rho_h}[\xi]\bigr),
    \]
    and the weighted Neumann problem then selects the unique minimum-energy velocity field
    realizing that variation. The form \(\mathfrak g_h\) measures the kinetic energy of this
    realization.
\end{remark}

\subsection{Canonical dynamical realization of regular coordinate paths}

We now pass from single tangent directions to time-dependent Bayes Hilbert paths.

\begin{definition}[Regular admissible path]
    \label{def:ambient-regular-admissible-path}
    A path
    \[
        h:[0,T]\to \mathcal X_{\mathrm{ad}}
    \]
    is called \emph{regular admissible} if
    \[
        h\in C^1([0,T];\mathcal Z).
    \]
    For such a path we define
    \[
        \rho_t := \rho_{h(t)},
        \qquad
        v_t := \mathcal T_{h(t)}\dot h(t).
    \]
\end{definition}

\begin{theorem}[Canonical dynamical realization]
    \label{thm:ambient-continuity-equation}
    Let \(h:[0,T]\to \mathcal X_{\mathrm{ad}}\) be a regular admissible path, and define
    \[
        \rho_t := \rho_{h(t)},
        \qquad
        v_t := \mathcal T_{h(t)}\dot h(t).
    \]
    Then \((\rho_t,v_t)\) satisfies the continuity equation
    \begin{equation}
        \label{eq:ambient-continuity-equation}
        \partial_t \rho_t + \nabla\cdot(\rho_t v_t)=0
    \end{equation}
    in the weak sense on \((0,T)\times\Omega\). Equivalently, for every
    \(\eta\in H^1(\Omega;\rho_t)\),
    \begin{equation}
        \label{eq:ambient-continuity-weak}
        \frac{\d}{\d t}\int_\Omega \eta\, \d\rho_t
        =
        \int_\Omega \nabla \eta\cdot v_t\, \d\rho_t.
    \end{equation}
    If \(\Omega\) has boundary, the weak formulation encodes the natural zero-flux condition.
\end{theorem}

\begin{proof}
    Since \(h\in C^1([0,T];\mathcal Z)\) and \(\mathcal Z\hookrightarrow L^\infty(\nu_0)\cap
    L^2_0(\nu_0)\) continuously, \(h\) is a regular coordinate path in the sense of Section~2.
    By Proposition~\ref{prop:ambient_log_density_evolution} from Section~2,
    \[
        \partial_t \frac{\d\rho_t}{\d\nu_0}
        =
        \frac{\d\rho_t}{\d\nu_0}
        \Bigl(\dot h(t)-\mathbb E_{\rho_t}[\dot h(t)]\Bigr).
    \]
    On the other hand, by definition of \(v_t\) and
    Remark~\ref{rem:ambient-neumann-all-tests},
    \[
        \int_\Omega \nabla\eta\cdot v_t\,\d\rho_t
        =
        \int_\Omega
        \Bigl(\dot h(t)-\mathbb E_{\rho_t}[\dot h(t)]\Bigr)\eta\,\d\rho_t
    \]
    for every \(\eta\in H^1(\Omega;\rho_t)\). Therefore
    \begin{align*}
        \frac{\d}{\d t}\int_\Omega \eta\,\d\rho_t
         & =
        \int_\Omega \eta\,
        \partial_t\!\left(\frac{\d\rho_t}{\d\nu_0}\right)\,\d\nu_0      \\
         & =
        \int_\Omega
        \eta
        \Bigl(\dot h(t)-\mathbb E_{\rho_t}[\dot h(t)]\Bigr)\,\d\rho_t \\
         & =
        \int_\Omega \nabla\eta\cdot v_t\,\d\rho_t,
    \end{align*}
    which is \eqref{eq:ambient-continuity-weak}.
\end{proof}

\subsection{Concrete models}
\label{sec:infinite-dimensional-models}

We conclude by recording the bounded-domain specialization used later in the paper.

\subsubsection{Bounded \texorpdfstring{$\Omega$}{Omega}, uniform reference}

\begin{corollary}[Bounded Lipschitz domains]
    \label{cor:abstract-forward-bounded}
    Let \(\Omega\subset\mathbb R^d\) be bounded, connected, and Lipschitz, and let
    \[
        \nu_0=|\Omega|^{-1}\,\d x.
    \]
    Fix exponents
    \[
        s>d,
        \qquad
        \frac d2<s'<\frac s2,
    \]
    and define
    \[
        \mathcal X:=H^s(\Omega)\cap L^2_0(\nu_0),
        \qquad
        \mathcal Z:=H^{s'}(\Omega)\cap L^2_0(\nu_0).
    \]
    Then \(\mathcal Z\hookrightarrow L^\infty(\Omega)\cap L^2(\nu_0)\) continuously, and
    \(\nu_0\) satisfies the Poincar\'e inequality \eqref{eq:reference-poincare}. Hence any
    admissible class \(\mathcal X_{\mathrm{ad}}\subset\mathcal Z\) satisfying
    \eqref{eq:ambient-uniform-density-bounds} falls under the abstract theory above. In this
    case, Theorem~\ref{thm:ambient-neumann} is the weak weighted Neumann problem with
    natural zero-flux boundary condition.

    If, in addition, \(\mathcal X_{\mathrm{ad}}\) is locally stable under \(\mathcal X\)-perturbations in the
    sense that for every \(h,\eta\in\mathcal X\) with \(h\in\mathcal X_{\mathrm{ad}}\) there exists
    \(\varepsilon_0>0\) such that
    \[
        h+\varepsilon\eta\in\mathcal X_{\mathrm{ad}}
        \qquad
        \text{for all }|\varepsilon|<\varepsilon_0,
    \]
    then Assumption~\ref{ass:abstract-forward-differentiable} is satisfied with the above choice of
    \(\mathcal X\), and therefore the linearization results
    (Proposition~\ref{prop:linearized-neumann} and
    Corollary~\ref{cor:directional-derivative-gh}) also apply in this bounded-domain setting.
    Moreover, Corollary~\ref{cor:ambient-transport-map-hsprime} becomes an \(H^{s'}\)-continuity
    statement for the transport map.
\end{corollary}

\begin{proof}
    The Sobolev embedding \(H^{s'}(\Omega)\hookrightarrow L^\infty(\Omega)\) holds because
    \(s'>d/2\), and the Poincar\'e inequality on bounded connected Lipschitz domains is
    classical. Since \(s>s'\), we also have the continuous embedding \(\mathcal X\hookrightarrow \mathcal Z\).
    Thus Assumption~\ref{ass:abstract-forward-reference} holds with the chosen \(\mathcal Z\).
    Together with Assumption~\ref{ass:ambient-admissible-class}, this yields the abstract forward
    theory in the bounded-domain setting. The final statement is immediate from the local
    stability hypothesis.
\end{proof}

% !TeX root = ../main.tex

\section{An inverse problem on Bayes Hilbert path space}
\label{sec:ambient-inverse-problem}

In Section~\ref{sec:intrinsic-forward-geometry}, we associated to each regular Bayes Hilbert path
\[
    h:[0,T]\to \mathcal X_{\mathrm{ad}}
\]
a canonical velocity field obtained from the weighted Neumann problem, together with the induced
transport form
\[
    \mathfrak g_h(\xi,\zeta).
\]
We now turn to the inverse problem. Rather than assuming that the path \(h(\cdot)\) is known, we
ask how to reconstruct it from indirect time-dependent observations.

In many settings, instantaneous observations do not determine the state: for a prescribed mobile
sensor path \(y\), the instantaneous linearized observation operator
\[
    J_{t,h(t)}^{y}
\]
may have a nontrivial kernel at each time \(t\). Reconstruction then requires combining information
from the entire observation record with dynamical structure linking different times. The central
analytic object capturing this interplay is the \emph{joint transport--observability form}, which
couples the instantaneous observability form induced by the mobile sensors with the transport form
induced by the forward geometry. Coercivity of this joint form is the natural replacement for
pointwise observability, and it makes the pair
\[
    (\mathfrak g_h,\mathfrak j_{t,h}^{y})
\]
into the engine of identifiability rather than merely an interpretive device.

\subsection{Mobile sensor observations and admissible paths}
\label{subsec:mobile-sensor-observations}

We retain the setting of Section~\ref{sec:intrinsic-forward-geometry}. Thus
\(\Omega\subset\mathbb R^d\) is bounded, connected, and Lipschitz,
\(\nu_0=|\Omega|^{-1}\d x\), the Sobolev exponents
\[
    s>d,
    \qquad
    \frac d2<s'<\frac s2
\]
are fixed, and
\[
    \mathcal X = H^s(\Omega)\cap L^2_0(\nu_0),
    \qquad
    \mathcal Z = H^{s'}(\Omega)\cap L^2_0(\nu_0),
    \qquad
    \mathcal X_{\mathrm{ad}}\subset \mathcal X
\]
is the admissible state class from Assumption~\ref{ass:ambient-admissible-class}.

For the inverse problem we impose one additional structural assumption on
\(\mathcal X_{\mathrm{ad}}\).

\begin{assumption}[Sobolev-regular admissible state class]
    \label{ass:sobolev-admissible-state-class}
    Assume that \(\mathcal X_{\mathrm{ad}}\) is closed in \(\mathcal Z\).
\end{assumption}

We work throughout this section with a prescribed mobile-sensor observation model.
Let
\[
    \kappa\in L^\infty(\mathbb R^d)\cap C_c(\mathbb R^d)
\]
be a compactly supported sensor kernel, and let
\[
    y=(y_1,\dots,y_r),
    \qquad
    y_j:[0,T]\to \Omega,
\]
be measurable sensor trajectories. For each \(t\in[0,T]\), define the instantaneous observation map
\begin{equation}
    \label{eq:mobile-sensor-observation}
    \mathcal G_t^{y}(h)
    :=
    \left(
    \int_\Omega \kappa(x-y_1(t))\,\d\rho_h(x),\dots,
    \int_\Omega \kappa(x-y_r(t))\,\d\rho_h(x)
    \right)\in \mathbb R^r.
\end{equation}
Thus the observation operator itself depends on time through the sensor locations.

If \(h^\dagger(\cdot)\) denotes the unknown true path, then the ideal data are
\[
    \obsdata^\dagger(t)=\mathcal G_t^y(h^\dagger(t)),
\]
and the measured data are modeled as
\[
    \obsdata(t)=\obsdata^\dagger(t)+\eta(t),
\]
where \(\eta\) is an observation error term.

\begin{definition}[Observability differential]
    \label{def:ambient-observability-differential}
    For \(t\in[0,T]\) and \(h\in\mathcal X_{\mathrm{ad}}\), assume \(\mathcal G_t^y\) is
    Fr\'echet differentiable at \(h\). The corresponding \emph{observability differential} is
    \[
        J_{t,h}^{y}:=D\mathcal G_t^y(h):H^{s'}(\Omega)\to\mathbb R^r.
    \]
\end{definition}

\begin{definition}[Instantaneous observability form]
    \label{def:ambient-observability-form}
    For \(t\in[0,T]\) and \(h\in\mathcal X_{\mathrm{ad}}\), the associated
    \emph{instantaneous observability form} is
    \[
        \mathfrak j_{t,h}^{y}(\xi,\zeta)
        :=
        \langle J_{t,h}^{y}\xi,\;J_{t,h}^{y}\zeta\rangle_{\mathbb R^r},
        \qquad
        \xi,\zeta\in H^{s'}(\Omega)\cap L^2_0(\nu_0).
    \]
\end{definition}

\begin{proposition}[Observability differential for mobile sensor averages]
    \label{prop:localized-sensor-differential}
    For each \(t\in[0,T]\), the map
    \[
        \mathcal G_t^{y}:\mathcal X_{\mathrm{ad}}\to\mathbb R^r
    \]
    defined by \eqref{eq:mobile-sensor-observation} is Fr\'echet differentiable with respect to the
    \(H^{s'}(\Omega)\)-topology, and for every \(h\in\mathcal X_{\mathrm{ad}}\) and
    \(\xi\in H^{s'}(\Omega)\cap L^2_0(\nu_0)\),
    \begin{equation}
        \label{eq:mobile-sensor-differential-formula}
        J_{t,h}^{y}\xi
        =
        \Bigl(
        \operatorname{Cov}_{\rho_h}(\kappa(\cdot-y_1(t)),\xi),\dots,
        \operatorname{Cov}_{\rho_h}(\kappa(\cdot-y_r(t)),\xi)
        \Bigr).
    \end{equation}
    Consequently,
    \begin{equation}
        \label{eq:mobile-sensor-observability-form}
        \mathfrak j_{t,h}^{y}(\xi,\zeta)
        =
        \sum_{j=1}^r
        \operatorname{Cov}_{\rho_h}(\kappa(\cdot-y_j(t)),\xi)\,
        \operatorname{Cov}_{\rho_h}(\kappa(\cdot-y_j(t)),\zeta),
    \end{equation}
    and in particular
    \begin{equation}
        \label{eq:mobile-sensor-observability-energy}
        \mathfrak j_{t,h}^{y}(\xi,\xi)
        =
        \sum_{j=1}^r
        \operatorname{Cov}_{\rho_h}(\kappa(\cdot-y_j(t)),\xi)^2.
    \end{equation}
\end{proposition}

\begin{proof}
    For each \(j=1,\dots,r\), define
    \[
        \mathcal G_{t,j}^{y}(h):=\int_\Omega \kappa(x-y_j(t))\,\d\rho_h(x).
    \]
    Using the differential of the exponential-normalization map from Proposition~\ref{prop:diff_exp_normalization}, we have
    \begin{align*}
        D\mathcal G_{t,j}^{y}(h)[\xi]
         & =
        \int_\Omega \kappa(x-y_j(t))
        \bigl(\xi(x)-\mathbb E_{\rho_h}[\xi]\bigr)\,\d\rho_h(x) \\
         & =
        \operatorname{Cov}_{\rho_h}(\kappa(\cdot-y_j(t)),\xi).
    \end{align*}
    This proves \eqref{eq:mobile-sensor-differential-formula}. The formulas
    \eqref{eq:mobile-sensor-observability-form} and
    \eqref{eq:mobile-sensor-observability-energy} follow immediately from
    Definition~\ref{def:ambient-observability-form}.
\end{proof}

\begin{remark}[Interpretation]
    \label{rem:localized-sensor-interpretation}
    For the mobile-sensor model \eqref{eq:mobile-sensor-observation}, a tangent direction \(\xi\) is
    seen through the data only via its covariance with the translated sensor kernels
    \(\kappa(\cdot-y_j(t))\) under the current law \(\rho_h\). Thus the visible directions vary in time
    both because the state \(h\) evolves and because the sensors move.
\end{remark}

\begin{proposition}[Instantaneous ill-posedness of mobile sensor observations]
    \label{prop:localized-sensor-illposedness}
    Fix \(t\in[0,T]\). For each \(h\in\mathcal X_{\mathrm{ad}}\), the differential
    \(J_{t,h}^{y}\) extends to a bounded linear operator
    \[
        J_{t,h}^{y}\in\mathcal L(L^2_0(\nu_0),\mathbb R^r),
    \]
    and
    \[
        \operatorname{rank}J_{t,h}^{y}\le r.
    \]
    In particular, \(J_{t,h}^{y}\) is finite-rank and therefore compact.

    Consequently, the instantaneous invisible subspace
    \[
        \ker J_{t,h}^{y}\subset L^2_0(\nu_0)
    \]
    is closed and has codimension at most \(r\). If \(L^2_0(\nu_0)\) is infinite dimensional, then
    \(\ker J_{t,h}^{y}\) is infinite dimensional.
\end{proposition}

\begin{proof}
    This is immediate from \eqref{eq:mobile-sensor-differential-formula}.
\end{proof}

The form \(\mathfrak j_{t,h}^{y}\) is symmetric and nonnegative by construction. In the partially
observable setting, \(\mathfrak j_{t,h}^{y}(\xi,\xi)\) may vanish for nonzero \(\xi\); equivalently,
the instantaneous observation differential \(J_{t,h}^{y}\) may have a nontrivial kernel even when the
sensor path is prescribed.

We reconstruct paths from the admissible class
\[
    \mathcal A_{\mathrm{ad}}
    :=
    \left\{
    h\in L^2(0,T;\mathcal X)\cap H^1(0,T;L^2_0(\nu_0))
    :
    h(t)\in \mathcal X_{\mathrm{ad}}
    \text{ for a.e.\ } t\in[0,T]
    \right\}.
\]

\begin{remark}[Observation-side ill-posedness]
    \label{rem:localized-sensor-illposedness}
    Proposition~\ref{prop:localized-sensor-illposedness} shows that the instantaneous linearized inverse
    problem for the mobile-sensor model is severely underdetermined in the ambient
    infinite-dimensional setting: at each fixed time \(t\), the data constrain at most \(r\) directions,
    while an infinite-dimensional family of tangent perturbations remains invisible. Thus
    identifiability cannot be expected from a single-time observation alone. The ambient inverse theory
    below addresses this by coupling instantaneous observability with the transport geometry through the
    joint form.
\end{remark}

\subsection{Pathwise observability and the joint transport--observability form}

We now formulate the central observability condition for the partially observable inverse problem.
The key idea is that directions invisible to the observation operator at a single time may become
visible when integrated over the full time horizon, provided the transport geometry couples
different times through the time derivative.

\begin{definition}[Pathwise observability Gramian]
    \label{def:pathwise-observability-gramian}
    For \(h\in \mathcal A_{\mathrm{ad}}\), the \emph{pathwise observability Gramian} is the bilinear
    form
    \[
        \mathfrak J_y[h](\xi,\zeta)
        :=
        \int_0^T \mathfrak j_{t,h(t)}^{y}\bigl(\xi(t),\zeta(t)\bigr)\,\d t
        =
        \int_0^T \langle J_{t,h(t)}^{y}\xi(t),\;J_{t,h(t)}^{y}\zeta(t)\rangle_{\mathbb R^r}\,\d t
    \]
    defined on path perturbations
    \[
        \xi,\zeta\in L^2(0,T;H^{s'}(\Omega)\cap L^2_0(\nu_0)).
    \]
\end{definition}

In the fully observable case---when the instantaneous form \(\mathfrak j_{t,h}^{y}\) is coercive at
each time---the pathwise Gramian \(\mathfrak J_y[h]\) is already coercive on its own. In the
partially observable setting, however, \(\mathfrak J_y[h]\) may degenerate: there may exist nonzero
path perturbations \(\xi(\cdot)\) with
\[
    \xi(t)\in\ker J_{t,h(t)}^{y}
    \qquad\text{for a.e.\ }t,
\]
so that \(\mathfrak J_y[h](\xi,\xi)=0\). Reconstruction then requires additional structure linking
different times.

This additional structure is provided by the transport form. The following definition combines the
two ambient geometric objects into a single bilinear form on path perturbations.

\begin{definition}[Joint transport--observability form]
    \label{def:joint-transport-observability-form}
    For \(h\in \mathcal A_{\mathrm{ad}}\), the \emph{joint transport--observability form} is the
    bilinear form
    \begin{equation}
        \label{eq:joint-form}
        Q_y[h](\xi,\zeta)
        :=
        \int_0^T \mathfrak j_{t,h(t)}^{y}\bigl(\xi(t),\zeta(t)\bigr)\,\d t
        +
        \int_0^T \mathfrak g_{h(t)}\bigl(\dot\xi(t),\dot\zeta(t)\bigr)\,\d t
    \end{equation}
    defined on path perturbations
    \[
        \xi,\zeta \in L^2(0,T;H^{s'}(\Omega)\cap L^2_0(\nu_0))
        \cap H^1(0,T;L^2_0(\nu_0)).
    \]
    Equivalently,
    \[
        Q_y[h](\xi,\zeta)
        =
        \mathfrak J_y[h](\xi,\zeta)
        +
        \int_0^T \mathfrak g_{h(t)}\bigl(\dot\xi(t),\dot\zeta(t)\bigr)\,\d t.
    \]
\end{definition}

The form \(Q_y[h]\) encodes the interplay between the two ambient geometric objects
\[
    (\mathfrak g_h,\mathfrak j_{t,h}^{y})
\]
at the level of path perturbations. The first term measures how strongly tangent directions are seen
through the data; the second measures the dynamical cost of their time variation. Together, they
quantify how much information about the perturbation \(\xi(\cdot)\) is available from the
combination of observations and dynamical structure.

\begin{remark}[Mechanism of partial observability]
    \label{rem:partial-observability-mechanism}
    Suppose at each time \(t\), the kernel of \(J_{t,h(t)}^{y}\) is nontrivial but \emph{rotates} as the
    state evolves and the sensors move. A nonzero perturbation \(\xi(\cdot)\) that tracks the kernel---staying in
    \(\ker J_{t,h(t)}^{y}\) for all \(t\)---must have a nontrivial time derivative \(\dot\xi\). The
    transport term
    \[
        \int_0^T\mathfrak g_{h(t)}(\dot\xi,\dot\xi)\,\d t
    \]
    then detects this variation. Conversely, a perturbation with \(\dot\xi=0\) is constant in time, so
    it can only hide from the data if it lies in \(\ker J_{t,h(t)}^{y}\) for \emph{all} \(t\)
    simultaneously. Joint coercivity of \(Q_y[h]\) rules out both scenarios: perturbations must be
    visible either through the observations or through their dynamical cost.
\end{remark}

\begin{assumption}[Joint transport--observability coercivity]
    \label{ass:joint-transport-observability}
    There exists \(\kappa>0\) such that for all \(h\in \mathcal A_{\mathrm{ad}}\) and all
    \[
        \xi\in L^2(0,T;H^{s'}(\Omega)\cap L^2_0(\nu_0))
        \cap H^1(0,T;L^2_0(\nu_0)),
    \]
    we have
    \begin{equation}
        \label{eq:joint-coercivity}
        Q_y[h](\xi,\xi)
        \ge
        \kappa \|\xi\|_{L^2(0,T;L^2(\nu_0))}^2.
    \end{equation}
\end{assumption}

\begin{remark}[Comparison with pointwise observability]
    \label{rem:pointwise-vs-pathwise}
    If the instantaneous observability form alone is coercive---that is, if there exists
    \(\kappa_0>0\) such that
    \[
        \mathfrak j_{t,h}^{y}(\xi,\xi)\ge \kappa_0 \|\xi\|_{L^2(\nu_0)}^2
        \qquad
        \text{for all } t\in[0,T],\ h\in\mathcal X_{\mathrm{ad}},\
        \xi\in H^{s'}(\Omega)\cap L^2_0(\nu_0),
    \]
    then Assumption~\ref{ass:joint-transport-observability} is automatically satisfied with
    \(\kappa=\kappa_0\), since the transport term is nonnegative. In this case, the data at each time
    already determine the state, and no dynamical coupling is needed. Conversely, when
    \(\ker J_{t,h(t)}^{y}\neq\{0\}\), Assumption~\ref{ass:joint-transport-observability} is strictly
    weaker than pointwise coercivity, and the transport term is essential: it compensates for
    directions that are invisible instantaneously by detecting their temporal variation. The
    \(\mu\)-regularization introduced in the variational formulation below provides analytic control over
    this transport term.
\end{remark}

\begin{proposition}[Upper bound on the transport term]
    \label{prop:transport-term-upper-bound}
    There exists a constant \(C_{\mathrm{tr}}>0\), depending only on the admissible class
    \(\mathcal X_{\mathrm{ad}}\) and the domain \(\Omega\), such that for all
    \(h\in\mathcal A_{\mathrm{ad}}\) and all
    \[
        \xi\in H^1(0,T;L^2_0(\nu_0)),
    \]
    we have
    \[
        \int_0^T \mathfrak g_{h(t)}(\dot\xi(t),\dot\xi(t))\,\d t
        \le
        C_{\mathrm{tr}}\|\dot\xi\|_{L^2(0,T;L^2(\nu_0))}^2.
    \]
\end{proposition}

\begin{proof}
    By definition of the transport form and the weighted Neumann solve,
    \[
        \mathfrak g_h(\zeta,\zeta)
        =
        \int_\Omega |\nabla\psi_{h,\zeta}|^2\,\d\rho_h.
    \]
    By the energy estimate from Theorem~\ref{thm:ambient-neumann}, together with the uniform density
    bounds on \(\mathcal X_{\mathrm{ad}}\), there exists \(C_{\mathrm{tr}}>0\) such that
    \[
        \mathfrak g_h(\zeta,\zeta)
        \le
        C_{\mathrm{tr}}\|\zeta\|_{L^2(\nu_0)}^2
        \qquad
        \text{for all } h\in\mathcal X_{\mathrm{ad}},\ \zeta\in L^2_0(\nu_0).
    \]
    Applying this pointwise in time with \(\zeta=\dot\xi(t)\) and integrating over \([0,T]\) gives the
    claim.
\end{proof}

\begin{remark}[Interpretation of the bounds]
    \label{rem:joint-form-bounds-interpretation}
    Assumption~\ref{ass:joint-transport-observability} gives a lower bound showing that the joint
    form \(Q_y[h]\) controls the \(L^2(0,T;L^2(\nu_0))\)-size of a path perturbation. On the other
    hand, Proposition~\ref{prop:transport-term-upper-bound} shows that the transport contribution to
    \(Q_y[h]\) is controlled by the natural \(H^1\)-in-time regularity of the perturbation.

    For the prescribed mobile-sensor model, the observation differential is uniformly bounded on
    \(\mathcal X_{\mathrm{ad}}\): there exists \(C_{\mathrm{obs}}>0\), depending only on
    \(\kappa\), \(r\), \(\Omega\), and \(\mathcal X_{\mathrm{ad}}\), such that
    \[
        \|J_{t,h}^{y}\xi\|_{\mathbb R^r}
        \le
        C_{\mathrm{obs}}\|\xi\|_{L^2(\nu_0)}
        \qquad
        \text{for all } t\in[0,T],\ h\in\mathcal X_{\mathrm{ad}},\ \xi\in L^2_0(\nu_0).
    \]
    Hence
    \[
        Q_y[h](\xi,\xi)
        \lesssim
        \|\xi\|_{L^2(0,T;L^2(\nu_0))}^2
        +
        \|\dot\xi\|_{L^2(0,T;L^2(\nu_0))}^2.
    \]
    Thus \(Q_y[h]\) should be viewed as a pathwise energy that is coercive in the \(L^2\)-state
    variable and compatible with the natural regularity of the admissible path space.
\end{remark}

\begin{proposition}[A sufficient criterion for joint transport--observability coercivity]
    \label{prop:kernel-decomposition-criterion}
    Assume that for each \(h\in \mathcal A_{\mathrm{ad}}\) and a.e.\ \(t\in[0,T]\), the observability
    differential \(J_{t,h(t)}^{y}\) extends to a bounded operator
    \[
        J_{t,h(t)}^{y}\in \mathcal L\bigl(L^2_0(\nu_0),\mathbb R^r\bigr).
    \]
    Define the instantaneous invisible subspace
    \[
        K_{y,h}(t):=\ker J_{t,h(t)}^{y}\subset L^2_0(\nu_0),
    \]
    and let
    \[
        \Pi_{y,h}(t):L^2_0(\nu_0)\to K_{y,h}(t)
    \]
    denote the \(L^2(\nu_0)\)-orthogonal projection onto \(K_{y,h}(t)\). Assume moreover that for
    each \(h\in \mathcal A_{\mathrm{ad}}\), the map
    \[
        t\longmapsto \Pi_{y,h}(t)
    \]
    is strongly measurable as an \(\mathcal L(L^2_0(\nu_0),L^2_0(\nu_0))\)-valued map.

    Suppose there exist constants \(c_{\mathrm{obs}},c_{\mathrm{dyn}}>0\) such that for every
    \(h\in \mathcal A_{\mathrm{ad}}\) the following hold:

    \begin{enumerate}
        \item[(i)] (\emph{uniform transverse observability}) for a.e.\ \(t\in[0,T]\) and every
              \(\zeta\in L^2_0(\nu_0)\),
              \begin{equation}
                  \label{eq:transverse-observability}
                  \mathfrak j_{t,h(t)}^{y}(\zeta,\zeta)
                  =
                  \|J_{t,h(t)}^{y}\zeta\|_{\mathbb R^r}^2
                  \ge
                  c_{\mathrm{obs}}
                  \bigl\|(I-\Pi_{y,h}(t))\zeta\bigr\|_{L^2(\nu_0)}^2;
              \end{equation}

        \item[(ii)] (\emph{uniform dynamical detectability of invisible directions}) for every
              \[
                  \xi\in L^2(0,T;H^{s'}(\Omega)\cap L^2_0(\nu_0))
                  \cap H^1(0,T;L^2_0(\nu_0)),
              \]
              one has
              \begin{equation}
                  \label{eq:dynamical-detectability}
                  \int_0^T \mathfrak g_{h(t)}(\dot\xi(t),\dot\xi(t))\,\d t
                  \ge
                  c_{\mathrm{dyn}}
                  \int_0^T
                  \bigl\|\Pi_{y,h}(t)\xi(t)\bigr\|_{L^2(\nu_0)}^2\,\d t.
              \end{equation}
    \end{enumerate}

    Then Assumption~\ref{ass:joint-transport-observability} holds with
    \[
        \kappa=\min\{c_{\mathrm{obs}},c_{\mathrm{dyn}}\}.
    \]
\end{proposition}

\begin{proof}
    Fix \(h\in \mathcal A_{\mathrm{ad}}\), and let
    \[
        \xi\in L^2(0,T;H^{s'}(\Omega)\cap L^2_0(\nu_0))
        \cap H^1(0,T;L^2_0(\nu_0)).
    \]
    Decompose \(\xi\) pointwise in time into its visible and invisible parts:
    \[
        \xi(t)=\xi_\perp(t)+\xi_0(t),
        \qquad
        \xi_\perp(t):=(I-\Pi_{y,h}(t))\xi(t),
        \qquad
        \xi_0(t):=\Pi_{y,h}(t)\xi(t).
    \]
    Since \(\Pi_{y,h}(t)\) is the \(L^2(\nu_0)\)-orthogonal projection onto \(K_{y,h}(t)\), we have
    \[
        \xi_\perp(t)\perp \xi_0(t)
        \qquad\text{in }L^2(\nu_0)
    \]
    for a.e.\ \(t\in[0,T]\), and therefore
    \begin{equation}
        \label{eq:orthogonal-splitting}
        \|\xi\|_{L^2(0,T;L^2(\nu_0))}^2
        =
        \int_0^T \|\xi_\perp(t)\|_{L^2(\nu_0)}^2\,\d t
        +
        \int_0^T \|\xi_0(t)\|_{L^2(\nu_0)}^2\,\d t.
    \end{equation}

    Now apply the two hypotheses. By \eqref{eq:transverse-observability},
    \[
        \int_0^T \mathfrak j_{t,h(t)}^{y}(\xi(t),\xi(t))\,\d t
        \ge
        c_{\mathrm{obs}}
        \int_0^T \|\xi_\perp(t)\|_{L^2(\nu_0)}^2\,\d t.
    \]
    By \eqref{eq:dynamical-detectability},
    \[
        \int_0^T \mathfrak g_{h(t)}(\dot\xi(t),\dot\xi(t))\,\d t
        \ge
        c_{\mathrm{dyn}}
        \int_0^T \|\xi_0(t)\|_{L^2(\nu_0)}^2\,\d t.
    \]
    Adding the two inequalities yields
    \begin{align*}
        Q_y[h](\xi,\xi)
         & =
        \int_0^T \mathfrak j_{t,h(t)}^{y}(\xi(t),\xi(t))\,\d t
        +
        \int_0^T \mathfrak g_{h(t)}(\dot\xi(t),\dot\xi(t))\,\d t \\
         & \ge
        c_{\mathrm{obs}}
        \int_0^T \|\xi_\perp(t)\|_{L^2(\nu_0)}^2\,\d t
        +
        c_{\mathrm{dyn}}
        \int_0^T \|\xi_0(t)\|_{L^2(\nu_0)}^2\,\d t               \\
         & \ge
        \min\{c_{\mathrm{obs}},c_{\mathrm{dyn}}\}
        \left(
        \int_0^T \|\xi_\perp(t)\|_{L^2(\nu_0)}^2\,\d t
        +
        \int_0^T \|\xi_0(t)\|_{L^2(\nu_0)}^2\,\d t
        \right).
    \end{align*}
    Using \eqref{eq:orthogonal-splitting}, we conclude that
    \[
        Q_y[h](\xi,\xi)
        \ge
        \min\{c_{\mathrm{obs}},c_{\mathrm{dyn}}\}
        \|\xi\|_{L^2(0,T;L^2(\nu_0))}^2,
    \]
    which is exactly Assumption~\ref{ass:joint-transport-observability}.
\end{proof}

\begin{remark}[Interpretation of the criterion]
    \label{rem:kernel-decomposition-criterion}
    Proposition~\ref{prop:kernel-decomposition-criterion} formalizes the heuristic in
    Remark~\ref{rem:partial-observability-mechanism}. The component
    \((I-\Pi_{y,h}(t))\xi(t)\) is instantaneously visible and is controlled directly by the data through
    \(\mathfrak j_{t,h(t)}^{y}\). The component \(\Pi_{y,h}(t)\xi(t)\) lies in the instantaneous kernel
    of the observation operator and is therefore invisible at any fixed time; condition
    \eqref{eq:dynamical-detectability} requires that such directions be detected through the dynamical
    cost of their time variation measured by \(\mathfrak g_h\).
\end{remark}

\subsection{A regularized variational inverse problem}
\label{subsec:regularized-variational-inverse}

The natural data-misfit term for the prescribed mobile-sensor system is
\[
    \frac12\int_0^T \|\mathcal G_t^y(h(t))-\obsdata(t)\|_{\mathbb R^r}^2\,\d t,
\]
and the natural dynamical penalty coming from the forward theory is the transport action
\[
    \frac12\int_0^T \mathfrak g_{h(t)}(\dot h(t),\dot h(t))\,\d t.
\]
In the ambient infinite-dimensional setting, however, the transport action alone does not provide
sufficient compactness for the direct-method existence proof. For that reason, we add both a
Bayes Hilbert \(H^1\)-in-time regularization term and a spatial Sobolev regularization term.

\begin{definition}[Regularized inverse functional]
    \label{def:ambient-inverse-functional}
    Let \(\lambda,\mu,\gamma>0\). For \(h\in\mathcal A_{\mathrm{ad}}\), define
    \begin{align}
        \label{eq:ambient-inverse-functional}
        \mathcal I_{\lambda,\mu,\gamma}^{\,y}[h]
         & :=
        \frac12\int_0^T \|\mathcal G_t^y(h(t))-\obsdata(t)\|_{\mathbb R^r}^2\,\d t
        +
        \frac{\lambda}{2}\int_0^T \mathfrak g_{h(t)}(\dot h(t),\dot h(t))\,\d t \notag \\
         & \quad
        +
        \frac{\mu}{2}\int_0^T
        \Big(
        \|h(t)\|_{L^2(\nu_0)}^2+\|\dot h(t)\|_{L^2(\nu_0)}^2
        \Big)\,\d t
        +
        \frac{\gamma}{2}\int_0^T \|h(t)\|_{H^s(\Omega)}^2\,\d t.
    \end{align}
\end{definition}

\begin{remark}[Dual role of the regularization terms]
    \label{rem:dual-role-regularization}
    Under full pointwise observability, the \(\mu\)- and \(\gamma\)-terms serve a purely analytic
    purpose: the former provides \(H^1\)-in-time compactness in the Bayes Hilbert norm, while the
    latter provides the spatial compactness needed for strong convergence in
    \(C([0,T];H^{s'}(\Omega))\).

    Under partial observability, the \(\mu\)-term acquires a second, structural role. The joint
    coercivity of \(Q_y[h]\) (Assumption~\ref{ass:joint-transport-observability}) shows that
    unobserved directions are detected through their dynamical cost in the transport form
    \(\mathfrak g_h\). But the transport term in \(Q_y[h]\) involves \(\dot\xi\), and it is the
    \(\mu\)-weighted \(\|\dot h\|_{L^2(\nu_0)}^2\) penalty that provides a priori control over this
    quantity for minimizers. In this way, \(\mu\) acts not only as a compactness device but as the
    mechanism by which the dynamical structure propagates information across time.
\end{remark}

The existence proof uses compactness of bounded-energy sequences and lower semicontinuity of each
term in the functional.

\begin{proposition}[Lower semicontinuity of the transport action under strong state convergence]
    \label{prop:transport-action-lsc-hsprime}
    Let \(h_n,h\in \mathcal A_{\mathrm{ad}}\) satisfy
    \[
        h_n\to h
        \quad\text{in } C([0,T];H^{s'}(\Omega)),
    \]
    and
    \[
        \dot h_n \rightharpoonup \dot h
        \quad\text{weakly in } L^2(0,T;L^2_0(\nu_0)).
    \]
    Then
    \[
        \int_0^T \mathfrak g_{h(t)}(\dot h(t),\dot h(t))\,\d t
        \le
        \liminf_{n\to\infty}
        \int_0^T \mathfrak g_{h_n(t)}(\dot h_n(t),\dot h_n(t))\,\d t.
    \]
\end{proposition}

\begin{proof}
    For each \(h\in\mathcal X_{\mathrm{ad}}\), define
    \[
        \mathcal S_h : L^2_0(\nu_0)\to L^2(\nu_0;\mathbb R^d),
        \qquad
        \mathcal S_h \xi := w_h^{1/2}\,\mathcal T_h\xi,
    \]
    where \(w_h:=\d\rho_h/\d\nu_0\). Then
    \[
        \mathfrak g_h(\xi,\xi)
        =
        \|\mathcal S_h \xi\|_{L^2(\nu_0;\mathbb R^d)}^2.
    \]

    We claim that
    \[
        \sup_{t\in[0,T]}
        \|\mathcal S_{h_n(t)}-\mathcal S_{h(t)}\|_{\mathcal L(L^2_0,L^2)}
        \to 0.
    \]
    Indeed, writing
    \[
        \mathcal S_{h_n}-\mathcal S_h
        =
        w_{h_n}^{1/2}(\mathcal T_{h_n}-\mathcal T_h)
        +
        (w_{h_n}^{1/2}-w_h^{1/2})\mathcal T_h,
    \]
    the first term vanishes by Corollary~\ref{cor:ambient-transport-map-hsprime} and the second
    by the Sobolev embedding \(H^{s'}(\Omega)\hookrightarrow L^\infty(\Omega)\) and the Lipschitz
    dependence of the square-root weight on the state.

    Now define
    \[
        u_n(t):=\mathcal S_{h_n(t)}\dot h_n(t),
        \qquad
        u(t):=\mathcal S_{h(t)}\dot h(t).
    \]
    The uniform operator convergence of \(\mathcal S_{h_n}\) to \(\mathcal S_h\), together with the
    weak convergence \(\dot h_n\rightharpoonup \dot h\) in \(L^2(0,T;L^2_0(\nu_0))\), implies
    \[
        u_n\rightharpoonup u
        \quad\text{weakly in } L^2(0,T;L^2(\nu_0;\mathbb R^d)).
    \]
    By weak lower semicontinuity of the norm,
    \[
        \|u\|_{L^2_tL^2_x}^2
        \le
        \liminf_{n\to\infty}
        \|u_n\|_{L^2_tL^2_x}^2,
    \]
    which is the desired inequality.
\end{proof}

\begin{proposition}[Compactness of bounded-energy sequences]
    \label{prop:ambient-compactness}
    Let \((h_n)\subset \mathcal A_{\mathrm{ad}}\) be a sequence satisfying
    \[
        \sup_n
        \left(
        \|h_n\|_{L^2(0,T;H^s(\Omega))}
        +
        \|h_n\|_{H^1(0,T;L^2_0(\nu_0))}
        \right)
        <\infty.
    \]
    Then there exist a subsequence, again denoted \((h_n)\), and a limit
    \[
        h\in L^2(0,T;H^s(\Omega))\cap H^1(0,T;L^2_0(\nu_0))
    \]
    such that
    \[
        h_n \rightharpoonup h
        \quad\text{weakly in } L^2(0,T;H^s(\Omega)),
    \]
    \[
        h_n \rightharpoonup h
        \quad\text{weakly in } H^1(0,T;L^2_0(\nu_0)),
    \]
    and
    \[
        h_n\to h
        \quad\text{in } C([0,T];H^{s'}(\Omega)).
    \]
    If, in addition, \(h_n(t)\in \mathcal X_{\mathrm{ad}}\) for a.e.\ \(t\), then
    \(h(t)\in \mathcal X_{\mathrm{ad}}\) for all \(t\in[0,T]\).
\end{proposition}

\begin{proof}
    The assumed bound gives boundedness in
    \[
        L^2(0,T;H^s(\Omega))\cap H^1(0,T;L^2_0(\nu_0)).
    \]
    By the Lions--Magenes interpolation theorem, applied to the Hilbert triple
    \(H^s(\Omega)\hookrightarrow H^{s/2}(\Omega)\hookrightarrow L^2(\Omega)\), this space embeds
    continuously into
    \[
        C([0,T];H^{s/2}(\Omega)).
    \]
    Thus \((h_n)\) is bounded in \(C([0,T];H^{s/2}(\Omega))\). Since \(s'<s/2\), the embedding
    \(H^{s/2}(\Omega)\hookrightarrow H^{s'}(\Omega)\) is compact.

    The \(H^1(0,T;L^2_0(\nu_0))\) bound gives equicontinuity in \(L^2(\nu_0)\): for all
    \(t_0,t_1\in[0,T]\),
    \[
        \|h_n(t_1)-h_n(t_0)\|_{L^2(\nu_0)}
        \le
        |t_1-t_0|^{1/2}
        \|\dot h_n\|_{L^2(0,T;L^2(\nu_0))}.
    \]
    Let \(\theta=2s'/s\in(0,1)\). Since \(\nu_0=|\Omega|^{-1}\d x\), the \(L^2(\nu_0)\) and
    Lebesgue \(L^2\) norms are equivalent, and interpolation gives
    \[
        \|f\|_{H^{s'}(\Omega)}
        \le
        C\|f\|_{L^2(\nu_0)}^{1-\theta}
        \|f\|_{H^{s/2}(\Omega)}^\theta.
    \]
    Applying this to \(f=h_n(t_1)-h_n(t_0)\), and using the uniform
    \(C([0,T];H^{s/2}(\Omega))\) bound, gives equicontinuity in \(H^{s'}(\Omega)\).

    The Arzela--Ascoli theorem now gives relative compactness in
    \[
        C([0,T];H^{s'}(\Omega)),
    \]
    because pointwise relative compactness follows from
    \(H^{s/2}(\Omega)\hookrightarrow H^{s'}(\Omega)\) compactly, and equicontinuity was established
    above. The weak convergences follow from Banach--Alaoglu.

    Finally, take the continuous representatives. Since \(h_n(t)\in\mathcal X_{\mathrm{ad}}\) for
    a.e.\ \(t\), continuity and the \(\mathcal Z\)-closedness of \(\mathcal X_{\mathrm{ad}}\) imply
    \(h_n(t)\in\mathcal X_{\mathrm{ad}}\) for all \(t\in[0,T]\). Passing to the uniform
    \(H^{s'}\)-limit and using closedness again gives \(h(t)\in\mathcal X_{\mathrm{ad}}\) for all
    \(t\in[0,T]\).
\end{proof}

\begin{theorem}[Existence of ambient reconstructions]
    \label{thm:ambient-existence}
    Let \(d\in L^2(0,T;\mathbb R^r)\). Then for every \(\lambda,\mu,\gamma>0\), the functional
    \(\mathcal I_{\lambda,\mu,\gamma}^{\,y}\) admits a minimizer over \(\mathcal A_{\mathrm{ad}}\).
\end{theorem}

\begin{proof}
    Let \((h_n)\subset \mathcal A_{\mathrm{ad}}\) be a minimizing sequence. Since the first two terms
    in \eqref{eq:ambient-inverse-functional} are nonnegative,
    \[
        \mathcal I_{\lambda,\mu,\gamma}^{\,y}[h_n]
        \ge
        \frac{\mu}{2}\int_0^T
        \Big(
        \|h_n(t)\|_{L^2(\nu_0)}^2+\|\dot h_n(t)\|_{L^2(\nu_0)}^2
        \Big)\,\d t
        +
        \frac{\gamma}{2}\int_0^T \|h_n(t)\|_{H^s(\Omega)}^2\,\d t.
    \]
    Thus \((h_n)\) is bounded in
    \[
        L^2(0,T;H^s(\Omega))\cap H^1(0,T;L^2_0(\nu_0)).
    \]
    By Proposition~\ref{prop:ambient-compactness}, after passing to a subsequence we obtain
    \[
        h_n\to h
        \quad\text{in } C([0,T];H^{s'}(\Omega)),
    \]
    \[
        h_n\rightharpoonup h
        \quad\text{weakly in } L^2(0,T;H^s(\Omega)),
    \]
    and
    \[
        \dot h_n\rightharpoonup \dot h
        \quad\text{weakly in } L^2(0,T;L^2_0(\nu_0))
    \]
    for some \(h\in \mathcal A_{\mathrm{ad}}\).

    For the data term, fix \(j\in\{1,\dots,r\}\). Since \(\kappa\in L^\infty\),
    \[
        \left|
        \int_\Omega \kappa(x-y_j(t))\,\d\rho_{h_n(t)}(x)
        \right|
        \le
        \|\kappa\|_{L^\infty(\mathbb R^d)}
    \]
    for all \(n\) and a.e.\ \(t\). Moreover, \(h_n(t)\to h(t)\) in \(H^{s'}(\Omega)\) for each \(t\),
    hence
    \[
        \mathcal G_t^y(h_n(t))\to \mathcal G_t^y(h(t))
        \qquad\text{for a.e.\ }t\in[0,T].
    \]
    By dominated convergence,
    \[
        \mathcal G_t^y(h_n(t))\to \mathcal G_t^y(h(t))
        \quad\text{in } L^2(0,T;\mathbb R^r),
    \]
    so the data term is continuous. Proposition~\ref{prop:transport-action-lsc-hsprime} gives lower
    semicontinuity of the transport action. The final two regularization terms are weakly lower
    semicontinuous by convexity. Therefore
    \[
        \mathcal I_{\lambda,\mu,\gamma}^{\,y}[h]
        \le
        \liminf_{n\to\infty} \mathcal I_{\lambda,\mu,\gamma}^{\,y}[h_n].
    \]
    Thus \(h\) is a minimizer.
\end{proof}

\begin{remark}[On the strong state topology]
    \label{rem:strong-compactness}
    The compactness result in Proposition~\ref{prop:ambient-compactness} produces strong convergence in
    \(C([0,T];H^{s'}(\Omega))\), which is the topology naturally matched to the continuity theory of
    the weighted Neumann solve from Section~\ref{sec:intrinsic-forward-geometry}. Since
    \(H^{s'}(\Omega)\hookrightarrow L^\infty(\Omega)\) continuously, this is in particular strong
    enough for all continuity statements involving the weights \(w_h\) and the transport map
    \(\mathcal T_h\).
\end{remark}

\subsection{Pathwise stability under partial observability}
\label{subsec:pathwise-stability}

We now prove the main stability result: under joint transport--observability coercivity, minimizers
of \(\mathcal I_{\lambda,\mu,\gamma}^{\,y}\) are stable with respect to data perturbations, even when
instantaneous observations do not determine the state. The proof combines three ingredients: a
priori bounds from minimality, the joint coercivity of \(Q_y[h]\), and the compactness theory of
Proposition~\ref{prop:ambient-compactness}.

For the local stability theory, it is convenient to isolate the subclass
\[
    \mathcal A_{\mathrm{ad}}^{s'}
    :=
    \mathcal A_{\mathrm{ad}} \cap C([0,T];H^{s'}(\Omega)).
\]
All local proximity conditions in this subsection are understood on
\(\mathcal A_{\mathrm{ad}}^{s'}\), so that expressions of the form
\[
    \sup_{t\in[0,T]}\|h(t)-k(t)\|_{H^{s'}(\Omega)}
\]
are well defined.

Throughout this subsection, we make the following regularity assumption on the mobile observation
family.

\begin{assumption}[Regularity of the mobile observation family]
    \label{ass:observation-regularity}
    Assume that there exists an open neighborhood \(U\subset \mathcal Z\) of
    \(\mathcal X_{\mathrm{ad}}\), together with constants \(r_0>0\) and \(L>0\), such that for each
    \(t\in[0,T]\),
    \[
        \mathcal G_t^y:U\to\mathbb R^r
    \]
    is \(C^1\) with respect to the \(H^{s'}(\Omega)\)-topology, and its differential
    \[
        J_{t,h}^{y}=D\mathcal G_t^y(h)
    \]
    is locally Lipschitz in \(h\), uniformly in \(t\), as a map into
    \(\mathcal L(L^2(\nu_0),\mathbb R^r)\): for all \(h,k\in U\) with
    \(\|h-k\|_{H^{s'}(\Omega)}\le r_0\),
    \begin{equation}
        \label{eq:J-lipschitz}
        \|J_{t,h}^{y}-J_{t,k}^{y}\|_{\mathcal L(L^2(\nu_0),\mathbb R^r)}
        \le
        L\|h-k\|_{H^{s'}(\Omega)}
        \qquad\text{for all } t\in[0,T].
    \end{equation}
\end{assumption}

\begin{lemma}[A priori bounds from minimality]
    \label{lem:a-priori-bounds}
    Let \(h^\dagger\in\mathcal A_{\mathrm{ad}}\), and let
    \[
        \obsdata(t)=\mathcal G_t^y(h^\dagger(t))+\eta(t)
    \]
    with
    \[
        \|\eta\|_{L^2(0,T;\mathbb R^r)}\le \delta.
    \]
    Define the regularization energy of the true path by
    \begin{equation}
        \label{eq:regularization-energy-true-path}
        R[h^\dagger]
        :=
        \frac{\lambda}{2}\int_0^T \mathfrak g_{h^\dagger(t)}(\dot h^\dagger(t),\dot h^\dagger(t))\,\d t
        +
        \frac{\mu}{2}\int_0^T
        \Big(
        \|h^\dagger(t)\|_{L^2(\nu_0)}^2+\|\dot h^\dagger(t)\|_{L^2(\nu_0)}^2
        \Big)\,\d t
        +
        \frac{\gamma}{2}\int_0^T \|h^\dagger(t)\|_{H^s(\Omega)}^2\,\d t.
    \end{equation}
    If \(h\in\mathcal A_{\mathrm{ad}}\) is a minimizer of \(\mathcal I_{\lambda,\mu,\gamma}^{\,y}\) for
    data \(\obsdata\), then
    \begin{equation}
        \label{eq:a-priori-functional-bound}
        \mathcal I_{\lambda,\mu,\gamma}^{\,y}[h]
        \le
        \tfrac12\delta^2+R[h^\dagger].
    \end{equation}
    In particular,
    \begin{align}
        \label{eq:a-priori-data-misfit}
        \tfrac12\|\mathcal G_\cdot^y(h(\cdot))-d\|_{L^2(0,T;\mathbb R^r)}^2
         & \le
        \tfrac12\delta^2+R[h^\dagger], \\
        \label{eq:a-priori-temporal}
        \tfrac{\mu}{2}\|\dot h\|_{L^2(0,T;L^2(\nu_0))}^2
         & \le
        \tfrac12\delta^2+R[h^\dagger].
    \end{align}
\end{lemma}

\begin{proof}
    Since \(h\) minimizes \(\mathcal I_{\lambda,\mu,\gamma}^{\,y}\) and
    \(h^\dagger\in\mathcal A_{\mathrm{ad}}\),
    \[
        \mathcal I_{\lambda,\mu,\gamma}^{\,y}[h]
        \le
        \mathcal I_{\lambda,\mu,\gamma}^{\,y}[h^\dagger]
        =
        \tfrac12\|\eta\|_{L^2(0,T;\mathbb R^r)}^2+R[h^\dagger]
        \le
        \tfrac12\delta^2+R[h^\dagger].
    \]
    The bounds \eqref{eq:a-priori-data-misfit} and \eqref{eq:a-priori-temporal} follow because each
    term in \(\mathcal I_{\lambda,\mu,\gamma}^{\,y}[h]\) is nonnegative.
\end{proof}

\begin{theorem}[Local pathwise stability under partial observability]
    \label{thm:pathwise-stability}
    Assume Assumptions~\ref{ass:joint-transport-observability}
    and~\ref{ass:observation-regularity}. Let
    \(h^\dagger\in \mathcal A_{\mathrm{ad}}^{s'}\), and let
    \[
        \obsdata(t)=\mathcal G_t^y(h^\dagger(t))+\eta(t)
    \]
    with
    \[
        \|\eta\|_{L^2(0,T;\mathbb R^r)}\le\delta.
    \]
    Let \(h\in\mathcal A_{\mathrm{ad}}^{s'}\) be a minimizer
    of \(\mathcal I_{\lambda,\mu,\gamma}^{\,y}\) for data \(\obsdata\).

    Then there exist constants \(r>0\) and \(C>0\), depending only on
    \(\mathcal X_{\mathrm{ad}}\), \(\Omega\), \(y\), the coercivity constant \(\kappa\), and
    the regularity constants in Assumption~\ref{ass:observation-regularity}, such that if
    \begin{equation}
        \label{eq:proximity-condition}
        \sup_{t\in[0,T]}\|h(t)-h^\dagger(t)\|_{H^{s'}(\Omega)}\le r,
    \end{equation}
    then
    \begin{equation}
        \label{eq:pathwise-stability-estimate}
        \|h-h^\dagger\|_{L^2(0,T;L^2(\nu_0))}^2
        \le
        \frac{C}{\kappa}(1+\mu^{-1})
        \Big(
        \delta^2+R[h^\dagger]
        \Big).
    \end{equation}
\end{theorem}

\begin{proof}
    Set \(\xi:=h-h^\dagger\). We estimate \(\|\xi\|_{L^2(0,T;L^2(\nu_0))}^2\) by applying the joint
    coercivity of \(Q_y[h^\dagger]\) and bounding the resulting terms using minimality and the
    proximity condition.

    \medskip
    \noindent
    \emph{Step 1: joint coercivity.}
    By Assumption~\ref{ass:joint-transport-observability} applied at the path \(h^\dagger\),
    \begin{equation}
        \label{eq:stab-joint-coercivity}
        \kappa\|\xi\|_{L^2_tL^2_x}^2
        \le
        \underbrace{
        \int_0^T \|J_{t,h^\dagger(t)}^{y}\xi(t)\|_{\mathbb R^r}^2\,\d t
        }_{\mathrm{Term\;I}}
        +
        \underbrace{
            \int_0^T \mathfrak g_{h^\dagger(t)}(\dot\xi(t),\dot\xi(t))\,\d t
        }_{\mathrm{Term\;II}}.
    \end{equation}

    \medskip
    \noindent
    \emph{Step 2: control of Term I.}
    Since \(h^\dagger\in \mathcal A_{\mathrm{ad}}^{s'}\), the image
    \(h^\dagger([0,T])\subset \mathcal Z\) is compact. Because \(U\subset\mathcal Z\) is an open
    neighborhood of \(\mathcal X_{\mathrm{ad}}\) and \(h^\dagger(t)\in\mathcal X_{\mathrm{ad}}\subset U\)
    for all \(t\), there exists \(r_U>0\) such that
    \[
        \bigcup_{t\in[0,T]} B_{\mathcal Z}(h^\dagger(t),r_U)\subset U.
    \]
    We choose \(r\) so that
    \[
        r\le \min\left\{r_0,\;r_U,\;\sqrt{\frac{\kappa}{L^2}}\right\}.
    \]
    Then for every \(t\in[0,T]\) and every \(\theta\in[0,1]\),
    \[
        h^\dagger(t)+\theta\xi(t)\in U,
    \]
    so the fundamental-theorem-of-calculus expansion and the Lipschitz estimate
    \eqref{eq:J-lipschitz} may be applied along the segment between
    \(h^\dagger(t)\) and \(h(t)\).

    By the fundamental theorem of calculus in Banach spaces,
    \[
        \mathcal G_t^y(h(t))-\mathcal G_t^y(h^\dagger(t))
        =
        J_{t,h^\dagger(t)}^{y}\xi(t)
        +
        \int_0^1
        \bigl(J_{t,h^\dagger(t)+\theta\xi(t)}^{y}-J_{t,h^\dagger(t)}^{y}\bigr)\xi(t)\,\d\theta.
    \]
    Denote the remainder by
    \[
        r(t)
        :=
        \int_0^1
        \bigl(J_{t,h^\dagger(t)+\theta\xi(t)}^{y}-J_{t,h^\dagger(t)}^{y}\bigr)\xi(t)\,\d\theta.
    \]
    By the Lipschitz condition \eqref{eq:J-lipschitz} and the proximity condition
    \eqref{eq:proximity-condition},
    \[
        \|r(t)\|_{\mathbb R^r}
        \le
        \int_0^1
        L\theta\|\xi(t)\|_{H^{s'}(\Omega)}\|\xi(t)\|_{L^2(\nu_0)}\,\d\theta
        \le
        \frac{Lr}{2}\|\xi(t)\|_{L^2(\nu_0)}.
    \]
    Since
    \[
        J_{t,h^\dagger(t)}^{y}\xi(t)
        =
        \mathcal G_t^y(h(t))-\mathcal G_t^y(h^\dagger(t))-r(t),
    \]
    we obtain
    \[
        \|J_{t,h^\dagger(t)}^{y}\xi(t)\|_{\mathbb R^r}^2
        \le
        2\|\mathcal G_t^y(h(t))-\mathcal G_t^y(h^\dagger(t))\|_{\mathbb R^r}^2
        +
        2\|r(t)\|_{\mathbb R^r}^2.
    \]
    For the first piece, write
    \[
        \mathcal G_t^y(h(t))-\mathcal G_t^y(h^\dagger(t))
        =
        \bigl(\mathcal G_t^y(h(t))-\obsdata(t)\bigr)+\eta(t),
    \]
    so that
    \[
        \|\mathcal G_t^y(h(t))-\mathcal G_t^y(h^\dagger(t))\|_{\mathbb R^r}^2
        \le
        2\|\mathcal G_t^y(h(t))-\obsdata(t)\|_{\mathbb R^r}^2+2\|\eta(t)\|_{\mathbb R^r}^2.
    \]
    Integrating in time and applying Lemma~\ref{lem:a-priori-bounds},
    \[
        \int_0^T \|\mathcal G_t^y(h(t))-\mathcal G_t^y(h^\dagger(t))\|_{\mathbb R^r}^2\,\d t
        \le
        4\bigl(\tfrac12\delta^2+R[h^\dagger]\bigr)+2\delta^2
        \le
        C_1\bigl(\delta^2+R[h^\dagger]\bigr).
    \]
    For the remainder,
    \[
        \int_0^T \|r(t)\|_{\mathbb R^r}^2\,\d t
        \le
        \frac{L^2 r^2}{4}\|\xi\|_{L^2_tL^2_x}^2.
    \]
    Therefore,
    \begin{equation}
        \label{eq:stab-term-I-bound}
        \mathrm{Term\;I}
        \le
        C_1\bigl(\delta^2+R[h^\dagger]\bigr)
        +
        \frac{L^2 r^2}{2}\|\xi\|_{L^2_tL^2_x}^2.
    \end{equation}

    \medskip
    \noindent
    \emph{Step 3: control of Term II.}
    By Proposition~\ref{prop:transport-term-upper-bound},
    \[
        \mathrm{Term\;II}
        =
        \int_0^T \mathfrak g_{h^\dagger(t)}(\dot\xi(t),\dot\xi(t))\,\d t
        \le
        C_{\mathrm{tr}}\|\dot\xi\|_{L^2_tL^2_x}^2.
    \]
    Since \(\xi=h-h^\dagger\), we have
    \[
        \dot\xi=\dot h-\dot h^\dagger,
    \]
    and therefore
    \[
        \|\dot\xi\|_{L^2_tL^2_x}^2
        \le
        2\|\dot h\|_{L^2_tL^2_x}^2
        +
        2\|\dot h^\dagger\|_{L^2_tL^2_x}^2.
    \]
    By the a priori bound \eqref{eq:a-priori-temporal},
    \[
        \frac{\mu}{2}\|\dot h\|_{L^2_tL^2_x}^2
        \le
        \frac12\delta^2+R[h^\dagger],
    \]
    so
    \[
        \|\dot h\|_{L^2_tL^2_x}^2
        \le
        \frac{1}{\mu}\bigl(\delta^2+2R[h^\dagger]\bigr).
    \]
    Also, by the definition of \(R[h^\dagger]\),
    \[
        R[h^\dagger]
        \ge
        \frac{\mu}{2}\|\dot h^\dagger\|_{L^2_tL^2_x}^2,
    \]
    hence
    \[
        \|\dot h^\dagger\|_{L^2_tL^2_x}^2
        \le
        \frac{2}{\mu}R[h^\dagger].
    \]
    Combining these estimates gives
    \begin{equation}
        \label{eq:stab-term-II-bound}
        \mathrm{Term\;II}
        \le
        \frac{C_3}{\mu}\bigl(\delta^2+R[h^\dagger]\bigr)
    \end{equation}
    for a constant \(C_3>0\) depending only on \(C_{\mathrm{tr}}\).

    \medskip
    \noindent
    \emph{Step 4: combine and absorb.}
    Substituting \eqref{eq:stab-term-I-bound} and \eqref{eq:stab-term-II-bound} into
    \eqref{eq:stab-joint-coercivity} gives
    \[
        \kappa\|\xi\|_{L^2_tL^2_x}^2
        \le
        \frac{L^2 r^2}{2}\|\xi\|_{L^2_tL^2_x}^2
        +
        C_1\bigl(\delta^2+R[h^\dagger]\bigr)
        +
        \frac{C_3}{\mu}\bigl(\delta^2+R[h^\dagger]\bigr).
    \]
    Choose
    \[
        r
        \le
        \min\left\{r_0,\;r_U,\;\sqrt{\frac{\kappa}{L^2}}\right\},
    \]
    so that \(\frac{L^2 r^2}{2}\le \frac\kappa2\). Then
    \[
        \frac\kappa2\|\xi\|_{L^2_tL^2_x}^2
        \le
        C(1+\mu^{-1})\bigl(\delta^2+R[h^\dagger]\bigr),
    \]
    and dividing by \(\kappa/2\) gives \eqref{eq:pathwise-stability-estimate}.
\end{proof}

\begin{remark}[The price of partial observability]
    \label{rem:price-of-partial-observability}
    Under full pointwise observability, Term~II in the proof is unnecessary, and the stability
    estimate takes the form
    \[
        \|h-h^\dagger\|_{L^2_tL^2_x}^2
        \le
        \frac{C}{\kappa}\bigl(\delta^2+R[h^\dagger]\bigr)
    \]
    with no factor of \(1/\mu\). Under partial observability, the transport term in \(Q_y[h]\) is
    needed to close the observability gap, and the \(\mu\)-regularization is what provides a priori
    control of this term. The resulting factor of \(1+\mu^{-1}\) is the quantitative cost of
    relying on dynamical coupling rather than instantaneous data; in the common regime
    \(0<\mu\le1\), this cost is equivalent to \(1/\mu\).
\end{remark}

\begin{remark}[On the proximity condition]
    \label{rem:proximity-condition}
    The proximity condition \eqref{eq:proximity-condition} is a local hypothesis. Under fixed
    \(\lambda,\mu,\gamma>0\), vanishing-noise minimizers are precompact in
    \(C([0,T];H^{s'}(\Omega))\), and subsequential limits solve the \emph{noiseless regularized}
    inverse problem. Thus the proximity condition is guaranteed in a small-noise regime provided the
    corresponding noiseless regularized minimizer lies within the ball
    \eqref{eq:proximity-condition} around \(h^\dagger\). In this sense,
    Theorem~\ref{thm:pathwise-stability} is a local stability statement near the reference path.
\end{remark}

\begin{corollary}[Vanishing-noise convergence to a noiseless regularized minimizer]
    \label{cor:vanishing-noise-noiseless-regularized}
    Assume the hypotheses of Theorem~\ref{thm:ambient-existence}. Let
    \[
        \mathcal R_{\lambda,\mu,\gamma}[h]
        :=
        \frac{\lambda}{2}\int_0^T \mathfrak g_{h(t)}(\dot h(t),\dot h(t))\,\d t
        +
        \frac{\mu}{2}\int_0^T
        \Big(
        \|h(t)\|_{L^2(\nu_0)}^2+\|\dot h(t)\|_{L^2(\nu_0)}^2
        \Big)\,\d t
        +
        \frac{\gamma}{2}\int_0^T \|h(t)\|_{H^s(\Omega)}^2\,\d t.
    \]
    For a sequence \(\delta_n\to 0\), let
    \[
        \obsdata_n(t)=\mathcal G_t^y(h^\dagger(t))+\eta_n(t),
        \qquad
        \|\eta_n\|_{L^2(0,T;\mathbb R^r)}\le \delta_n,
    \]
    and let \(h_n\in\mathcal A_{\mathrm{ad}}\) be a minimizer of
    \[
        \mathcal I_n[h]
        :=
        \frac12\|\mathcal G_\cdot^y(h(\cdot))-\obsdata_n\|_{L^2(0,T;\mathbb R^r)}^2
        +
        \mathcal R_{\lambda,\mu,\gamma}[h].
    \]
    Define the noiseless regularized functional
    \[
        \mathcal I^0[h]
        :=
        \frac12\|\mathcal G_\cdot^y(h(\cdot))-\mathcal G_\cdot^y(h^\dagger(\cdot))\|_{L^2(0,T;\mathbb R^r)}^2
        +
        \mathcal R_{\lambda,\mu,\gamma}[h].
    \]

    Then the sequence \((h_n)\) is bounded in
    \[
        L^2(0,T;H^s(\Omega))\cap H^1(0,T;L^2_0(\nu_0)).
    \]
    Consequently, after passing to a subsequence,
    \[
        h_n\to h_\ast
        \qquad\text{in } C([0,T];H^{s'}(\Omega))
    \]
    for some \(h_\ast\in\mathcal A_{\mathrm{ad}}\).

    Moreover, \(h_\ast\) is a minimizer of \(\mathcal I^0\) over \(\mathcal A_{\mathrm{ad}}\).

    If \(\mathcal I^0\) has a unique minimizer \(h^0_{\lambda,\mu,\gamma}\), then the whole sequence
    satisfies
    \[
        h_n\to h^0_{\lambda,\mu,\gamma}
        \qquad\text{in } C([0,T];H^{s'}(\Omega)).
    \]

    Finally, assume the hypotheses of Theorem~\ref{thm:pathwise-stability} hold, and let \(r>0\) be
    the radius from Theorem~\ref{thm:pathwise-stability}. If
    \[
        \sup_{t\in[0,T]}
        \|h^0_{\lambda,\mu,\gamma}(t)-h^\dagger(t)\|_{H^{s'}(\Omega)}<r,
    \]
    then for all sufficiently large \(n\), the minimizers \(h_n\) satisfy the proximity condition
    \eqref{eq:proximity-condition}, and therefore
    \[
        \|h_n-h^\dagger\|_{L^2(0,T;L^2(\nu_0))}^2
        \le
        \frac{C}{\kappa}(1+\mu^{-1})
        \bigl(\delta_n^2+R[h^\dagger]\bigr)
    \]
    for all sufficiently large \(n\).
\end{corollary}

\begin{proof}
    Since \(h_n\) minimizes \(\mathcal I_n\) and \(h^\dagger\in\mathcal A_{\mathrm{ad}}\), we have
    \[
        \mathcal I_n[h_n]
        \le
        \mathcal I_n[h^\dagger]
        =
        \frac12\|\eta_n\|_{L^2(0,T;\mathbb R^r)}^2
        +
        \mathcal R_{\lambda,\mu,\gamma}[h^\dagger]
        \le
        \frac12\delta_n^2+\mathcal R_{\lambda,\mu,\gamma}[h^\dagger].
    \]
    Since the data-misfit term and the transport term are nonnegative, this yields the uniform bound
    \[
        \sup_n
        \left(
        \|h_n\|_{L^2(0,T;H^s(\Omega))}
        +
        \|h_n\|_{H^1(0,T;L^2_0(\nu_0))}
        \right)
        <\infty.
    \]
    By Proposition~\ref{prop:ambient-compactness}, after passing to a subsequence,
    \[
        h_n\to h_\ast
        \quad\text{in } C([0,T];H^{s'}(\Omega)),
    \]
    \[
        h_n\rightharpoonup h_\ast
        \quad\text{weakly in } L^2(0,T;H^s(\Omega)),
    \]
    and
    \[
        \dot h_n\rightharpoonup \dot h_\ast
        \quad\text{weakly in } L^2(0,T;L^2_0(\nu_0))
    \]
    for some \(h_\ast\in\mathcal A_{\mathrm{ad}}\).

    Because \(\obsdata_n\to \mathcal G_\cdot^y(h^\dagger(\cdot))\) in
    \(L^2(0,T;\mathbb R^r)\), for each fixed \(h\in\mathcal A_{\mathrm{ad}}\) we have
    \[
        \mathcal I_n[h]\to \mathcal I^0[h].
    \]
    Also, by dominated convergence,
    \[
        \mathcal G_\cdot^y(h_n(\cdot))\to \mathcal G_\cdot^y(h_\ast(\cdot))
        \quad\text{in } L^2(0,T;\mathbb R^r).
    \]
    Together with Proposition~\ref{prop:transport-action-lsc-hsprime} and weak lower semicontinuity of
    the \(\mu\)- and \(\gamma\)-terms, this gives
    \[
        \mathcal I^0[h_\ast]
        \le
        \liminf_{n\to\infty}\mathcal I_n[h_n].
    \]

    Now let \(h\in\mathcal A_{\mathrm{ad}}\) be arbitrary. Since \(h_n\) minimizes \(\mathcal I_n\),
    \[
        \mathcal I_n[h_n]\le \mathcal I_n[h].
    \]
    Passing to the limit superior on the right and combining with the previous liminf inequality yields
    \[
        \mathcal I^0[h_\ast]\le \mathcal I^0[h].
    \]
    Thus \(h_\ast\) is a minimizer of \(\mathcal I^0\).

    If \(\mathcal I^0\) has a unique minimizer \(h^0_{\lambda,\mu,\gamma}\), then every convergent
    subsequence of \((h_n)\) has the same limit, so the whole sequence converges to
    \(h^0_{\lambda,\mu,\gamma}\) in \(C([0,T];H^{s'}(\Omega))\).

    For the final claim, assume
    \[
        \sup_{t\in[0,T]}
        \|h^0_{\lambda,\mu,\gamma}(t)-h^\dagger(t)\|_{H^{s'}(\Omega)}<r.
    \]
    Since \(h_n\to h^0_{\lambda,\mu,\gamma}\) in \(C([0,T];H^{s'}(\Omega))\), it follows that for all
    sufficiently large \(n\),
    \[
        \sup_{t\in[0,T]}
        \|h_n(t)-h^\dagger(t)\|_{H^{s'}(\Omega)}<r.
    \]
    Hence the proximity condition \eqref{eq:proximity-condition} holds for \(h_n\), and
    Theorem~\ref{thm:pathwise-stability} gives
    \[
        \|h_n-h^\dagger\|_{L^2(0,T;L^2(\nu_0))}^2
        \le
        \frac{C}{\kappa}(1+\mu^{-1})
        \bigl(\delta_n^2+R[h^\dagger]\bigr)
    \]
    for all sufficiently large \(n\).
\end{proof}

In particular, under fixed positive regularization parameters \(\lambda,\mu,\gamma\), vanishing-noise
reconstructions converge to the noiseless \emph{regularized} inverse problem. By contrast,
Theorem~\ref{thm:vanishing-regularization-consistency} shows that a simultaneous vanishing-noise,
vanishing-regularization regime yields convergence to an \emph{exact} noiseless solution of minimum
regularization energy.

\begin{theorem}[Vanishing-regularization consistency toward a minimum-regularity exact solution]
    \label{thm:vanishing-regularization-consistency}
    Assume the hypotheses of Theorem~\ref{thm:ambient-existence}. Fix
    \[
        \bar\lambda,\bar\mu,\bar\gamma>0,
    \]
    and define the reference regularization functional
    \begin{align}
        \label{eq:reference-regularization-functional}
        \mathcal R[h]
         & :=
        \frac{\bar\lambda}{2}\int_0^T \mathfrak g_{h(t)}(\dot h(t),\dot h(t))\,\d t
        +
        \frac{\bar\mu}{2}\int_0^T
        \Big(
        \|h(t)\|_{L^2(\nu_0)}^2+\|\dot h(t)\|_{L^2(\nu_0)}^2
        \Big)\,\d t \notag \\
         & \quad
        +
        \frac{\bar\gamma}{2}\int_0^T \|h(t)\|_{H^s(\Omega)}^2\,\d t .
    \end{align}
    Let \(\obsdata^\dagger \in L^2(0,T;\mathbb R^r)\), and assume that the exact solution set
    \[
        \mathcal S(\obsdata^\dagger)
        :=
        \left\{
        h\in \mathcal A_{\mathrm{ad}}
        :
        \mathcal G_\cdot^y(h(\cdot))=\obsdata^\dagger
        \text{ in } L^2(0,T;\mathbb R^r)
        \right\}
    \]
    is nonempty.

    Let \(\obsdata_n\in L^2(0,T;\mathbb R^r)\) satisfy
    \[
        \|\obsdata_n-\obsdata^\dagger\|_{L^2(0,T;\mathbb R^r)}\le \delta_n,
        \qquad
        \delta_n\to 0,
    \]
    and let \(\alpha_n>0\) satisfy
    \[
        \alpha_n\to 0,
        \qquad
        \frac{\delta_n^2}{\alpha_n}\to 0.
    \]
    For each \(n\), define
    \begin{equation}
        \label{eq:vanishing-regularization-functional}
        \mathcal I_n[h]
        :=
        \frac12 \|\mathcal G_\cdot^y(h(\cdot))-\obsdata_n\|_{L^2(0,T;\mathbb R^r)}^2
        +
        \alpha_n \mathcal R[h],
        \qquad
        h\in \mathcal A_{\mathrm{ad}},
    \end{equation}
    and let \(h_n\in \mathcal A_{\mathrm{ad}}\) be a minimizer of \(\mathcal I_n\).

    Then the sequence \((h_n)\) is bounded in
    \[
        L^2(0,T;H^s(\Omega))
        \cap
        H^1(0,T;L^2_0(\nu_0)).
    \]
    Consequently, after passing to a subsequence,
    \[
        h_n \to h_\ast
        \qquad\text{in } C([0,T];H^{s'}(\Omega))
    \]
    for some \(h_\ast\in \mathcal A_{\mathrm{ad}}\).

    Moreover, \(h_\ast\) is an exact solution:
    \[
        h_\ast \in \mathcal S(\obsdata^\dagger),
    \]
    and it minimizes \(\mathcal R\) over the exact solution set:
    \[
        \mathcal R[h_\ast]
        =
        \min\left\{
        \mathcal R[h] : h\in \mathcal S(\obsdata^\dagger)
        \right\}.
    \]

    If, in addition, the \(\mathcal R\)-minimizing exact solution is unique, then the whole sequence
    satisfies
    \[
        h_n \to h_\ast
        \qquad\text{in } C([0,T];H^{s'}(\Omega)).
    \]
\end{theorem}

\begin{proof}
    Let \(\tilde h\in \mathcal S(\obsdata^\dagger)\) be arbitrary. Since \(h_n\) minimizes
    \(\mathcal I_n\),
    \[
        \frac12 \|\mathcal G_\cdot^y(h_n(\cdot))-\obsdata_n\|_{L^2(0,T;\mathbb R^r)}^2
        +
        \alpha_n \mathcal R[h_n]
        \le
        \frac12 \|\mathcal G_\cdot^y(\tilde h(\cdot))-\obsdata_n\|_{L^2(0,T;\mathbb R^r)}^2
        +
        \alpha_n \mathcal R[\tilde h].
    \]
    Because \(\mathcal G_\cdot^y(\tilde h(\cdot))=\obsdata^\dagger\), this gives
    \[
        \frac12 \|\mathcal G_\cdot^y(h_n(\cdot))-\obsdata_n\|_{L^2(0,T;\mathbb R^r)}^2
        +
        \alpha_n \mathcal R[h_n]
        \le
        \frac12 \|\obsdata^\dagger-\obsdata_n\|_{L^2(0,T;\mathbb R^r)}^2
        +
        \alpha_n \mathcal R[\tilde h]
        \le
        \frac12 \delta_n^2 + \alpha_n \mathcal R[\tilde h].
    \]
    Hence
    \begin{equation}
        \label{eq:vanishing-regularization-bound}
        \mathcal R[h_n]
        \le
        \mathcal R[\tilde h] + \frac{\delta_n^2}{2\alpha_n}.
    \end{equation}
    Since \(\delta_n^2/\alpha_n \to 0\), the sequence \((\mathcal R[h_n])\) is bounded.

    Because \(\bar\mu,\bar\gamma>0\), boundedness of \(\mathcal R[h_n]\) implies boundedness of
    \[
        \|h_n\|_{L^2(0,T;H^s(\Omega))}
        \qquad\text{and}\qquad
        \|h_n\|_{H^1(0,T;L^2_0(\nu_0))}.
    \]
    Therefore Proposition~\ref{prop:ambient-compactness} applies, and after passing to a subsequence,
    \[
        h_n \rightharpoonup h_\ast
        \quad\text{weakly in } L^2(0,T;H^s(\Omega)),
    \]
    \[
        h_n \rightharpoonup h_\ast
        \quad\text{weakly in } H^1(0,T;L^2_0(\nu_0)),
    \]
    and
    \[
        h_n \to h_\ast
        \quad\text{in } C([0,T];H^{s'}(\Omega))
    \]
    for some \(h_\ast\in \mathcal A_{\mathrm{ad}}\).

    Next, from the minimizing inequality above we also have
    \[
        \frac12 \|\mathcal G_\cdot^y(h_n(\cdot))-\obsdata_n\|_{L^2(0,T;\mathbb R^r)}^2
        \le
        \frac12 \delta_n^2 + \alpha_n \mathcal R[\tilde h].
    \]
    Since \(\delta_n\to 0\) and \(\alpha_n\to 0\), it follows that
    \[
        \|\mathcal G_\cdot^y(h_n(\cdot))-\obsdata_n\|_{L^2(0,T;\mathbb R^r)} \to 0.
    \]
    Therefore
    \[
        \|\mathcal G_\cdot^y(h_n(\cdot))-\obsdata^\dagger\|_{L^2(0,T;\mathbb R^r)}
        \le
        \|\mathcal G_\cdot^y(h_n(\cdot))-\obsdata_n\|_{L^2(0,T;\mathbb R^r)}
        +
        \|\obsdata_n-\obsdata^\dagger\|_{L^2(0,T;\mathbb R^r)}
        \to 0.
    \]
    Since \(h_n\to h_\ast\) in \(C([0,T];H^{s'}(\Omega))\), dominated convergence yields
    \[
        \mathcal G_\cdot^y(h_\ast(\cdot))=\obsdata^\dagger.
    \]
    Thus \(h_\ast\in \mathcal S(\obsdata^\dagger)\).

    It remains to show that \(h_\ast\) minimizes \(\mathcal R\) over
    \(\mathcal S(\obsdata^\dagger)\). Let \(z\in \mathcal S(\obsdata^\dagger)\) be arbitrary. By
    minimality of \(h_n\),
    \[
        \frac12 \|\mathcal G_\cdot^y(h_n(\cdot))-\obsdata_n\|_{L^2(0,T;\mathbb R^r)}^2
        +
        \alpha_n \mathcal R[h_n]
        \le
        \frac12 \|\mathcal G_\cdot^y(z(\cdot))-\obsdata_n\|_{L^2(0,T;\mathbb R^r)}^2
        +
        \alpha_n \mathcal R[z].
    \]
    Using \(\mathcal G_\cdot^y(z(\cdot))=\obsdata^\dagger\), we obtain
    \[
        \alpha_n \mathcal R[h_n]
        \le
        \frac12 \delta_n^2 + \alpha_n \mathcal R[z].
    \]
    Dividing by \(\alpha_n\) yields
    \[
        \mathcal R[h_n]
        \le
        \mathcal R[z] + \frac{\delta_n^2}{2\alpha_n}.
    \]
    Taking the limit superior and using \(\delta_n^2/\alpha_n\to 0\), we obtain
    \[
        \limsup_{n\to\infty}\mathcal R[h_n]\le \mathcal R[z].
    \]

    On the other hand, Proposition~\ref{prop:transport-action-lsc-hsprime} gives lower semicontinuity of
    the transport part of \(\mathcal R\), while the remaining two terms are weakly lower semicontinuous.
    Hence
    \[
        \mathcal R[h_\ast]
        \le
        \liminf_{n\to\infty}\mathcal R[h_n].
    \]
    Combining the last two inequalities gives
    \[
        \mathcal R[h_\ast]\le \mathcal R[z].
    \]
    Since \(z\in \mathcal S(\obsdata^\dagger)\) was arbitrary, \(h_\ast\) minimizes
    \(\mathcal R\) over \(\mathcal S(\obsdata^\dagger)\).

    Finally, if the \(\mathcal R\)-minimizing exact solution is unique, then every convergent subsequence
    of \((h_n)\) has the same limit \(h_\ast\). Therefore the whole sequence converges to \(h_\ast\) in
    \(C([0,T];H^{s'}(\Omega))\).
\end{proof}

\begin{remark}[Interpretation as Tikhonov consistency]
    \label{rem:tikhonov-consistency}
    Theorem~\ref{thm:vanishing-regularization-consistency} strengthens
    Corollary~\ref{cor:vanishing-noise-noiseless-regularized}. For fixed positive regularization
    parameters, vanishing-noise reconstructions converge only to a minimizer of the noiseless
    \emph{regularized} problem. By contrast, when the regularization strength \(\alpha_n\) tends to zero
    in such a way that \(\delta_n^2/\alpha_n\to 0\), the reconstructions converge to an \emph{exact}
    solution of the noiseless inverse problem, and among exact solutions they select one of minimum
    regularization energy \(\mathcal R\). In this sense, the theorem gives a Tikhonov-type consistency
    statement for the Bayes Hilbert path-space inverse problem with prescribed mobile sensors.
\end{remark}

\subsection{Spectral analysis of the joint form}
\label{subsec:spectral-analysis-mobile-sensors}

The joint transport--observability coercivity condition
(Assumption~\ref{ass:joint-transport-observability}) is sufficient for the stability theory, but it
is a strong hypothesis. In this subsection we remain entirely within the ambient mobile-sensor
setting introduced above and analyze the joint form without assuming coercivity. We first
characterize its null space, then prove that localized mobile sensors still yield a compact
time-averaged observation operator in the infinite-dimensional ambient space, and finally record the
resulting spectral decomposition. This identifies the directions that are effectively resolved and
unresolved by the sensor motion over the full time horizon. The finite-dimensional recovery mechanism
is deferred to Section~\ref{sec:finite-dimensional-reduction}.

\subsubsection{Null space of the joint form}

We begin with the null space of the ambient joint form \(Q_y[h]\).

\begin{proposition}[Null-space characterization]
    \label{prop:null-space-Q}
    Let \(h\in \mathcal A_{\mathrm{ad}}\). Then
    \[
        \xi\in\ker Q_y[h]
        \quad\Longleftrightarrow\quad
        \left\{
        \begin{aligned}
             & \xi(t)\in\ker J_{t,h(t)}^{y} \quad\text{for a.e.\ }t\in[0,T], \\
             & \dot\xi(t)=0 \quad\text{for a.e.\ }t\in[0,T].
        \end{aligned}
        \right.
    \]
    In particular, every element of \(\ker Q_y[h]\) is time-constant:
    \[
        \xi(t)\equiv \xi_0,
    \]
    with
    \begin{equation}
        \label{eq:null-space-Q}
        \ker Q_y[h]
        =
        \left\{
        \xi_0\in L^2_0(\nu_0):
        \xi_0\in\ker J_{t,h(t)}^{y}
        \text{ for a.e.\ }t\in[0,T]
        \right\}.
    \end{equation}
\end{proposition}

\begin{proof}
    If \(Q_y[h](\xi,\xi)=0\), then both terms in the definition of \(Q_y[h]\) vanish:
    \[
        \int_0^T \|J_{t,h(t)}^{y}\xi(t)\|_{\mathbb R^r}^2\,\d t=0,
        \qquad
        \int_0^T \mathfrak g_{h(t)}(\dot\xi(t),\dot\xi(t))\,\d t=0.
    \]
    The first identity gives \(\xi(t)\in\ker J_{t,h(t)}^{y}\) for a.e.\ \(t\). By
    Proposition~\ref{prop:ambient-transport-form-properties}, the second identity gives
    \(\dot\xi(t)=0\) for a.e.\ \(t\). Hence \(\xi\) is time-constant. The converse is immediate.
\end{proof}

Proposition~\ref{prop:null-space-Q} shows that the joint form has a much smaller null space than
the instantaneous observation operator alone. At a fixed time \(t\), the kernel of
\(J_{t,h(t)}^{y}\) may be large, but to lie in \(\ker Q_y[h]\) a perturbation must be
simultaneously invisible at almost every time and constant in time. Thus the relevant obstruction
is the common invisible subspace \eqref{eq:null-space-Q}.

\subsubsection{Compactness obstruction to ambient coercivity}

The next result shows that, even with time-dependent sensor locations, localized mobile sensors do
not produce ambient coercivity on the infinite-dimensional state space.

For \(h\in\mathcal A_{\mathrm{ad}}\) and sensor path \(y\), define the time-averaged observation
form on time-constant perturbations by
\begin{equation}
    \label{eq:time-averaged-observation-form}
    \bar{\mathfrak j}_y[h](\xi_0,\eta_0)
    :=
    \int_0^T
    \bigl\langle J_{t,h(t)}^{y}\xi_0,\;J_{t,h(t)}^{y}\eta_0\bigr\rangle_{\mathbb R^r}\,\d t,
    \qquad
    \xi_0,\eta_0\in L^2_0(\nu_0).
\end{equation}

\begin{proposition}[Compactness obstruction for localized mobile sensors]
    \label{prop:compactness-obstruction}
    Let \(\kappa\in C_c(\mathbb R^d)\), let \(y_1,\dots,y_r:[0,T]\to\Omega\) be measurable, and let
    \(h\in\mathcal A_{\mathrm{ad}}\). Assume in addition that \(L^2_0(\nu_0)\) is
    infinite-dimensional. Then the form \(\bar{\mathfrak j}_y[h]\) is induced by a compact,
    self-adjoint, nonnegative operator
    \[
        \bar J_y[h]:L^2_0(\nu_0)\to L^2_0(\nu_0).
    \]

    Consequently, there is no constant \(\kappa_{\mathrm{obs}}>0\) such that
    \[
        Q_y[h](\xi,\xi)
        \ge
        \kappa_{\mathrm{obs}}
        \|\xi\|_{L^2(0,T;L^2(\nu_0))}^2
    \]
    for every
    \[
        \xi\in
        L^2\bigl(0,T;H^{s'}(\Omega)\cap L^2_0(\nu_0)\bigr)
        \cap
        H^1\bigl(0,T;L^2_0(\nu_0)\bigr).
    \]
\end{proposition}

\begin{proof}
    Define
    \[
        \mathcal K_y[h]:L^2_0(\nu_0)\to L^2(0,T;\mathbb R^r)
    \]
    by
    \[
        (\mathcal K_y[h]\xi_0)(t)
        :=
        J_{t,h(t)}^{y}\xi_0.
    \]
    Then
    \[
        \bar{\mathfrak j}_y[h](\xi_0,\eta_0)
        =
        \langle \mathcal K_y[h]\xi_0,\mathcal K_y[h]\eta_0\rangle_{L^2(0,T;\mathbb R^r)},
    \]
    so \(\bar J_y[h]=\mathcal K_y[h]^\ast \mathcal K_y[h]\).

    It remains to show that \(\mathcal K_y[h]\) is compact. For \(j=1,\dots,r\), write
    \[
        (\mathcal K_y[h]\xi_0)_j(t)
        =
        \int_\Omega b_j(t,x)\,\xi_0(x)\,\d\nu_0(x),
    \]
    where
    \[
        b_j(t,x)
        :=
        \Bigl(
        \kappa(x-y_j(t))
        -
        \mathbb E_{\rho_{h(t)}}[\kappa(\cdot-y_j(t))]
        \Bigr)
        \,w_{h(t)}(x),
        \qquad
        w_{h(t)}:=\frac{\d\rho_{h(t)}}{\d\nu_0}.
    \]
    Since \(\kappa\in L^\infty\) and the admissibility bounds give
    \[
        0<c\le w_{h(t)}(x)\le C<\infty
        \qquad\text{for a.e.\ }(t,x)\in[0,T]\times\Omega,
    \]
    we have
    \[
        |b_j(t,x)|\le 2C\|\kappa\|_{L^\infty(\mathbb R^d)},
    \]
    hence \(b_j\in L^2((0,T)\times\Omega)\). Therefore each component
    \[
        \xi_0\longmapsto (\mathcal K_y[h]\xi_0)_j
    \]
    is a Hilbert--Schmidt operator from \(L^2_0(\nu_0)\) to \(L^2(0,T)\), and so
    \(\mathcal K_y[h]\) itself is Hilbert--Schmidt, hence compact.

    Thus \(\bar J_y[h]=\mathcal K_y[h]^\ast\mathcal K_y[h]\) is compact, self-adjoint, and
    nonnegative.

    For the coercivity claim, suppose by contradiction that there exists
    \(\kappa_{\mathrm{obs}}>0\) such that
    \[
        Q_y[h](\xi,\xi)
        \ge
        \kappa_{\mathrm{obs}}
        \|\xi\|_{L^2(0,T;L^2(\nu_0))}^2
    \]
    for every admissible perturbation \(\xi\). Let \(\xi_0\in H^{s'}(\Omega)\cap L^2_0(\nu_0)\), and
    consider the time-constant path \(\xi(t)\equiv\xi_0\). Then \(\dot\xi=0\), so
    \[
        Q_y[h](\xi,\xi)
        =
        \bar{\mathfrak j}_y[h](\xi_0,\xi_0),
        \qquad
        \|\xi\|_{L^2(0,T;L^2(\nu_0))}^2
        =
        T\|\xi_0\|_{L^2(\nu_0)}^2.
    \]
    Hence
    \[
        \bar{\mathfrak j}_y[h](\xi_0,\xi_0)
        \ge
        \kappa_{\mathrm{obs}}T\|\xi_0\|_{L^2(\nu_0)}^2
        \qquad
        \forall \xi_0\in H^{s'}(\Omega)\cap L^2_0(\nu_0).
    \]
    Because \(H^{s'}(\Omega)\cap L^2_0(\nu_0)\) is dense in \(L^2_0(\nu_0)\) and
    \(\bar{\mathfrak j}_y[h]\) is continuous on \(L^2_0(\nu_0)\), this extends to all
    \(\xi_0\in L^2_0(\nu_0)\). Thus the compact operator \(\bar J_y[h]\) would be bounded below by a
    positive multiple of the identity on an infinite-dimensional Hilbert space, which is impossible.
\end{proof}

\begin{remark}[Physical interpretation]
    \label{rem:compactness-physical}
    The obstruction in Proposition~\ref{prop:compactness-obstruction} is a property of the observation
    modality, not of the sensor trajectory. Localized kernels act as spatial averaging operators. No
    matter how the sensors move, sufficiently oscillatory perturbations remain weakly visible, and the
    corresponding time-averaged observation eigenvalues must decay to zero.
\end{remark}

\subsubsection{Spectral decomposition and effective resolved dimension}

Although ambient coercivity fails, Proposition~\ref{prop:compactness-obstruction} yields a precise
spectral description of the directions that are strongly or weakly seen by the mobile sensors over
the full time horizon.

\begin{proposition}[Spectral decomposition of the averaged observation form]
    \label{prop:spectral-decomposition-mobile}
    Let \(h\in\mathcal A_{\mathrm{ad}}\), and let \(y\) be a measurable sensor trajectory. Then there
    exist an orthonormal basis \(\{e_k\}_{k\ge 1}\) of \(L^2_0(\nu_0)\) and a nonincreasing sequence
    \(\sigma_k^2\downarrow 0\) such that
    \[
        \bar J_y[h]e_k=\sigma_k^2 e_k,
        \qquad k\ge 1,
    \]
    and
    \begin{equation}
        \label{eq:spectral-expansion-mobile}
        \bar{\mathfrak j}_y[h](\xi_0,\xi_0)
        =
        \sum_{k=1}^\infty \sigma_k^2
        \bigl|\langle \xi_0,e_k\rangle_{L^2(\nu_0)}\bigr|^2
        \qquad
        \forall \xi_0\in L^2_0(\nu_0).
    \end{equation}
\end{proposition}

\begin{proof}
    This is the spectral theorem for compact self-adjoint operators applied to \(\bar J_y[h]\).
\end{proof}

\begin{definition}[Resolved and unresolved directions]
    \label{def:resolved-unresolved}
    Let \(\mu>0\). A spectral direction \(e_k\) is called
    \begin{itemize}
        \item \emph{resolved at scale \(\mu\)} if \(\sigma_k^2\ge \mu\),
        \item \emph{unresolved at scale \(\mu\)} if \(\sigma_k^2<\mu\).
    \end{itemize}
    The corresponding effective resolved dimension is
    \begin{equation}
        \label{eq:effective-dimension}
        m_{\mathrm{eff}}(\mu)
        :=
        \#\{k:\sigma_k^2\ge \mu\}.
    \end{equation}
\end{definition}

\begin{remark}[Relation to the regularized inverse functional]
    \label{rem:resolved-unresolved-interpretation}
    The regularized inverse functional from Definition~\ref{def:ambient-inverse-functional} contains
    the term
    \[
        \frac{\mu}{2}\int_0^T \|h(t)\|_{L^2(\nu_0)}^2\,\d t.
    \]
    For time-constant perturbations, the competition between the data term and this \(L^2\)-penalty is
    measured by the eigenvalues \(\sigma_k^2\) in
    \eqref{eq:spectral-expansion-mobile}. Directions with \(\sigma_k^2\gg\mu\) are primarily
    constrained by the observations over the time horizon, whereas directions with \(\sigma_k^2\ll\mu\)
    are controlled predominantly by regularization. In this sense, \(m_{\mathrm{eff}}(\mu)\) counts the
    number of ambient directions that are observed at scale \(\mu\) by the mobile sensor system.
\end{remark}

% !TeX root = ../main.tex

\section{Finite-dimensional models}
\label{sec:finite-dimensional-reduction}

The ambient theory developed in Sections~\ref{sec:intrinsic-forward-geometry}
and~\ref{sec:ambient-inverse-problem} is formulated on admissible Bayes Hilbert paths
\[
    h\in \mathcal A_{\mathrm{ad}}.
\]
We now restrict that theory to finite-dimensional Bayes Hilbert subspaces. This serves two
purposes.

First, it shows that the reduced transport and inverse objects are not ad hoc constructions.
They are simply the coordinate representations of the ambient geometry after restricting the
state variable \(h\) to a finite-dimensional subspace \(V_m\subset \mathcal X\). In particular, the
reduced transport tensor and reduced sensing matrices introduced below are the finite-dimensional
pullbacks of the ambient transport and observability forms.

Second, the finite-dimensional setting is the level at which localized sensing yields genuine
recovery theorems. In the ambient infinite-dimensional problem, mobile sensing can remove
common invisible directions and may even make the joint transport--observability form injective,
but it does not in general produce a coercive stability estimate on the full state space. After
restriction to \(V_m\), this compactness obstruction disappears, and one can formulate explicit
reduced observability and recovery criteria in terms of finite-dimensional Gramian matrices.

There are, however, two distinct levels of result in this section, and it is useful to separate them
from the start. The first level is \emph{truth-specific reduced observability}: for a fixed reduced
truth path, one can design localized sensors---including a single moving sensor---so that the
resulting reduced Gramian is positive definite and the reduced joint form is coercive along that
path. The second level is \emph{global reduced inversion on an admissible class}: to solve a
regularized inverse problem over an entire reduced class of coefficient paths, one needs an
additional classwise reduced stability estimate for the chosen experiment. The last subsection
makes this extra step explicit and then shows how reduced reconstructions lift to approximate
ambient reconstructions.

The section has the following structure. We first pull back the ambient transport and observation
geometry to a finite-dimensional coefficient space. We then show that a sufficiently localized bump
sensor detects any fixed nonzero direction. Next, we prove that on a fixed
\(m\)-dimensional reduced subspace, one can place finitely many stationary sensors so that the
reduced state is instantaneously observable at every time. Finally, we show that a single moving
sensor can recover the same reduced path by cycling through those locations, even though with one
sensor the full reduced state need not be observable at any individual time.

\subsection{Finite-dimensional Bayes Hilbert reductions}

Let
\[
    V_m:=\operatorname{span}\{\phi_1,\dots,\phi_m\}\subset \mathcal X
\]
be an \(m\)-dimensional subspace of the ambient state space \(\mathcal X\), where
\[
    \phi_1,\dots,\phi_m\in \mathcal X
\]
are linearly independent. For \(a=(a_1,\dots,a_m)\in\mathbb R^m\), define
\begin{equation}
    \label{eq:finite-dimensional-state}
    h(a):=\sum_{k=1}^m a_k\phi_k.
\end{equation}
Whenever \(h(a)\in \mathcal X_{\mathrm{ad}}\), we write
\[
    \rho_a:=\rho_{h(a)}.
\]

Thus a coefficient path
\[
    a:[0,T]\to\mathbb R^m
\]
induces a Bayes Hilbert path
\[
    h(t)=h(a(t))=\sum_{k=1}^m a_k(t)\phi_k
\]
and hence a law-valued path
\[
    \rho_t=\rho_{a(t)}.
\]
The corresponding reduced admissible class is
\[
    \mathcal A_{\mathrm{ad}}^{(m)}
    :=
    \left\{
    a\in H^1(0,T;\mathbb R^m)
    :
    h(a(\cdot))\in \mathcal A_{\mathrm{ad}}
    \right\}.
\]

\begin{remark}
    \label{rem:finite-dimensional-as-restriction}
    The finite-dimensional theory is obtained by restricting the ambient state variable \(h\) to the
    subspace \(V_m\). Every reduced transport, observation, and inverse object below is therefore the
    pullback of its ambient counterpart through the coordinate map
    \[
        \mathbb R^m\ni a\longmapsto h(a)\in V_m\subset \mathcal X.
    \]
\end{remark}

\begin{remark}[Finite-dimensional reductions as exponential families]
    \label{rem:finite-dimensional-exponential-family}
    Finite-dimensional affine subspaces of Bayes Hilbert spaces
    correspond to finite-dimensional exponential families relative to \(\nu_0\)
    \citep{van_den_boogaart_bayes_2014}. In particular, the linear reduction
    \[
        h(a)=\sum_{k=1}^m a_k\phi_k
    \]
    represents the normalized exponential family
    \[
        \d\rho_a
        =
        \frac{
            \exp\!\left(\sum_{k=1}^m a_k\phi_k\right)
        }{
            \int_\Omega
            \exp\!\left(\sum_{k=1}^m a_k\phi_k\right)
            \,\d\nu_0
        }
        \,\d\nu_0.
    \]
\end{remark}

\subsection{Reduced transport and observation geometry}

We begin by writing the ambient transport form and the ambient mobile-sensor observation map in
coordinates on \(V_m\).

For \(a\in\mathbb R^m\) such that \(h(a)\in \mathcal X_{\mathrm{ad}}\), define
\begin{equation}
    \label{eq:reduced-transport-matrix}
    H(a)=\bigl(H_{k\ell}(a)\bigr)_{k,\ell=1}^m,
    \qquad
    H_{k\ell}(a):=\mathfrak g_{h(a)}(\phi_k,\phi_\ell).
\end{equation}

\begin{proposition}[Coordinate representation of the transport form]
    \label{prop:coordinate-transport-form}
    Let \(a\in\mathbb R^m\) with \(h(a)\in \mathcal X_{\mathrm{ad}}\), and let
    \[
        \alpha_h=\sum_{k=1}^m \alpha_k\phi_k,
        \qquad
        \beta_h=\sum_{k=1}^m \beta_k\phi_k
    \]
    be directions in \(V_m\), with coefficient vectors
    \[
        \alpha=(\alpha_1,\dots,\alpha_m)^\top,
        \qquad
        \beta=(\beta_1,\dots,\beta_m)^\top.
    \]
    Then
    \begin{equation}
        \label{eq:coordinate-transport-form}
        \mathfrak g_{h(a)}(\alpha_h,\beta_h)=\alpha^\top H(a)\beta.
    \end{equation}
    In particular, if \(a\in \mathcal A_{\mathrm{ad}}^{(m)}\), then
    \begin{equation}
        \label{eq:coordinate-transport-action}
        \mathfrak g_{h(a(t))}\bigl(\dot h(t),\dot h(t)\bigr)
        =
        \dot a(t)^\top H(a(t))\dot a(t)
        \qquad\text{for a.e.\ }t\in[0,T].
    \end{equation}
\end{proposition}

\begin{proof}
    By bilinearity of \(\mathfrak g_{h(a)}\),
    \[
        \mathfrak g_{h(a)}(\alpha_h,\beta_h)
        =
        \sum_{k,\ell=1}^m
        \alpha_k\beta_\ell\,\mathfrak g_{h(a)}(\phi_k,\phi_\ell)
        =
        \sum_{k,\ell=1}^m
        \alpha_k\beta_\ell\,H_{k\ell}(a)
        =
        \alpha^\top H(a)\beta.
    \]
    Taking
    \[
        \dot h(t)=\sum_{k=1}^m \dot a_k(t)\phi_k
    \]
    gives \eqref{eq:coordinate-transport-action}.
\end{proof}

We next specialize the mobile-sensor observation model of Section~4 to the reduced state space.

\begin{definition}[Reduced mobile-sensor observation map]
    \label{def:reduced-mobile-observation-map}
    For \(t\in[0,T]\), a sensor trajectory \(y\), and \(a\in\mathbb R^m\) such that
    \(h(a)\in\mathcal X_{\mathrm{ad}}\), define
    \[
        \mathcal G_{m,t}^{y}(a):=\mathcal G_t^y(h(a))\in\mathbb R^r.
    \]
\end{definition}

\begin{proposition}[Reduced sensing matrix for mobile sensor averages]
    \label{prop:reduced-mobile-sensing-matrix}
    Let
    \[
        \xi=\sum_{k=1}^m \alpha_k\phi_k \in V_m,
        \qquad
        \alpha=(\alpha_1,\dots,\alpha_m)^\top\in \mathbb R^m.
    \]
    For \(t\in[0,T]\) and \(a\in\mathbb R^m\) with \(h(a)\in\mathcal X_{\mathrm{ad}}\), define
    \begin{equation}
        \label{eq:reduced-mobile-sensing-matrix}
        M_y(t,a)\in\mathbb R^{r\times m},
        \qquad
        [M_y(t,a)]_{jk}
        :=
        \operatorname{Cov}_{\rho_{h(a)}}(\kappa(\cdot-y_j(t)),\phi_k).
    \end{equation}
    Then
    \begin{equation}
        \label{eq:reduced-mobile-differential}
        D\mathcal G_{m,t}^{y}(a)\,\alpha
        =
        M_y(t,a)\alpha.
    \end{equation}
    Equivalently,
    \[
        J_{t,h(a)}^y\xi=M_y(t,a)\alpha.
    \]
    Hence
    \begin{equation}
        \label{eq:reduced-mobile-observability-energy}
        \mathfrak j_{t,h(a)}^{y}(\xi,\xi)
        =
        \alpha^\top M_y(t,a)^\top M_y(t,a)\alpha.
    \end{equation}
\end{proposition}

\begin{proof}
    By Proposition~\ref{prop:localized-sensor-differential},
    \[
        J_{t,h(a)}^y\xi
        =
        \Bigl(
        \operatorname{Cov}_{\rho_{h(a)}}(\kappa(\cdot-y_1(t)),\xi),\dots,
        \operatorname{Cov}_{\rho_{h(a)}}(\kappa(\cdot-y_r(t)),\xi)
        \Bigr).
    \]
    Substituting \(\xi=\sum_{k=1}^m \alpha_k\phi_k\) and using bilinearity of covariance gives
    \[
        \operatorname{Cov}_{\rho_{h(a)}}(\kappa(\cdot-y_j(t)),\xi)
        =
        \sum_{k=1}^m \alpha_k
        \operatorname{Cov}_{\rho_{h(a)}}(\kappa(\cdot-y_j(t)),\phi_k)
        =
        \sum_{k=1}^m [M_y(t,a)]_{jk}\alpha_k,
    \]
    which proves \eqref{eq:reduced-mobile-differential}. The identity
    \eqref{eq:reduced-mobile-observability-energy} follows from
    Definition~\ref{def:ambient-observability-form}.
\end{proof}

The reduced joint transport--observability form is therefore the ambient form written in
coordinates.

\begin{proposition}[Pathwise reduced transport--observability]
    \label{prop:pathwise-reduced-observability}
    Let
    \[
        a\in \mathcal A_{\mathrm{ad}}^{(m)},
        \qquad
        \alpha\in H^1(0,T;\mathbb R^m),
    \]
    and define the associated perturbation
    \[
        \xi_\alpha(t):=\sum_{k=1}^m \alpha_k(t)\phi_k \in V_m.
    \]
    Then
    \begin{equation}
        \label{eq:pathwise-reduced-joint-form}
        Q_y[h(a)](\xi_\alpha,\xi_\alpha)
        =
        \int_0^T \alpha(t)^\top M_y(t,a(t))^\top M_y(t,a(t))\alpha(t)\,\d t
        +
        \int_0^T \dot\alpha(t)^\top H(a(t))\dot\alpha(t)\,\d t.
    \end{equation}

    Assume there exists \(\kappa_m>0\) such that for every \(a\in \mathcal A_{\mathrm{ad}}^{(m)}\) and every
    \(\alpha\in H^1(0,T;\mathbb R^m)\),
    \begin{equation}
        \label{eq:pathwise-reduced-coercivity}
        \int_0^T \alpha(t)^\top M_y(t,a(t))^\top M_y(t,a(t))\alpha(t)\,\d t
        +
        \int_0^T \dot\alpha(t)^\top H(a(t))\dot\alpha(t)\,\d t
        \ge
        \kappa_m \int_0^T |\alpha(t)|^2\,\d t.
    \end{equation}
    If, moreover, there exists \(C_m>0\) such that
    \begin{equation}
        \label{eq:reduced-coefficient-norm-equivalence}
        \left\|\sum_{k=1}^m \beta_k \phi_k\right\|_{L^2(\nu_0)}^2
        \le
        C_m |\beta|^2
        \qquad\text{for all } \beta\in\mathbb R^m,
    \end{equation}
    then Assumption~\ref{ass:joint-transport-observability} holds on the reduced class with
    \[
        \kappa=\frac{\kappa_m}{C_m}.
    \]
\end{proposition}

\begin{proof}
    By Proposition~\ref{prop:coordinate-transport-form},
    \[
        \mathfrak g_{h(a(t))}\bigl(\dot\xi_\alpha(t),\dot\xi_\alpha(t)\bigr)
        =
        \dot\alpha(t)^\top H(a(t))\dot\alpha(t).
    \]
    By Proposition~\ref{prop:reduced-mobile-sensing-matrix},
    \[
        \mathfrak j_{t,h(a(t))}^{y}\bigl(\xi_\alpha(t),\xi_\alpha(t)\bigr)
        =
        \alpha(t)^\top M_y(t,a(t))^\top M_y(t,a(t))\alpha(t).
    \]
    Substituting these identities into Definition~\ref{def:joint-transport-observability-form} gives
    \eqref{eq:pathwise-reduced-joint-form}.

    If \eqref{eq:pathwise-reduced-coercivity} holds, then
    \[
        Q_y[h(a)](\xi_\alpha,\xi_\alpha)
        \ge
        \kappa_m \int_0^T |\alpha(t)|^2\,\d t.
    \]
    On the other hand, \eqref{eq:reduced-coefficient-norm-equivalence} gives
    \[
        \|\xi_\alpha(t)\|_{L^2(\nu_0)}^2
        =
        \left\|\sum_{k=1}^m \alpha_k(t)\phi_k\right\|_{L^2(\nu_0)}^2
        \le
        C_m |\alpha(t)|^2
    \]
    for a.e.\ \(t\in[0,T]\). Integrating in time yields
    \[
        \|\xi_\alpha\|_{L^2(0,T;L^2(\nu_0))}^2
        \le
        C_m \int_0^T |\alpha(t)|^2\,\d t.
    \]
    Therefore
    \[
        Q_y[h(a)](\xi_\alpha,\xi_\alpha)
        \ge
        \frac{\kappa_m}{C_m}
        \|\xi_\alpha\|_{L^2(0,T;L^2(\nu_0))}^2,
    \]
    which is exactly Assumption~\ref{ass:joint-transport-observability} on the reduced class.
\end{proof}

Finally, the ambient inverse functional restricts to coefficient space.

\begin{corollary}[Reduced inverse functional]
    \label{cor:finite-dimensional-inverse-problem}
    Let \(a\in \mathcal A_{\mathrm{ad}}^{(m)}\), and set
    \[
        h(t):=h(a(t)).
    \]
    Then the ambient inverse functional \(\mathcal I_{\lambda,\mu,\gamma}\) takes the form
    \begin{align}
        \label{eq:reduced-inverse-functional}
        \mathcal I_{\lambda,\mu,\gamma}[a]
         & =
        \frac12\int_0^T \|\mathcal G_{m,t}^{y}(a(t))-d(t)\|_{\mathbb R^r}^2\,\d t
        +
        \frac{\lambda}{2}\int_0^T \dot a(t)^\top H(a(t))\dot a(t)\,\d t \notag \\
         & \quad
        +
        \frac{\mu}{2}\int_0^T
        \Bigl(
        \|h(a(t))\|_{L^2(\nu_0)}^2
        +
        \|\dot h(t)\|_{L^2(\nu_0)}^2
        \Bigr)\,\d t
        +
        \frac{\gamma}{2}\int_0^T \|h(a(t))\|_{H^s(\Omega)}^2\,\d t.
    \end{align}
    Because \(V_m\) is finite dimensional, all norms on \(V_m\) are equivalent. In particular, the
    last three terms are equivalent to Euclidean quadratic penalties in \(a(\cdot)\) and \(\dot a(\cdot)\).
\end{corollary}

\begin{proof}
    Substitute \(h(t)=h(a(t))\) into Definition~\ref{def:ambient-inverse-functional}. The data term
    becomes
    \[
        \|\mathcal G_t^y(h(a(t)))-d(t)\|_{\mathbb R^r}^2
        =
        \|\mathcal G_{m,t}^{y}(a(t))-d(t)\|_{\mathbb R^r}^2,
    \]
    and the transport term is identified by Proposition~\ref{prop:coordinate-transport-form}. Since
    \[
        h(a(t))=\sum_{k=1}^m a_k(t)\phi_k,
        \qquad
        \dot h(t)=\sum_{k=1}^m \dot a_k(t)\phi_k,
    \]
    the remaining terms are simply the pullbacks of ambient norms to the finite-dimensional space
    \(V_m\), hence are equivalent to Euclidean quadratic forms in \(a(t)\) and \(\dot a(t)\).
\end{proof}

\begin{remark}
    \label{rem:reduced-regularization}
    The finite-dimensional reduction removes the spatial compactness issue that motivated the
    \(\gamma\)-term in the ambient existence theory. Thus \(\gamma\) may be set to zero on \(V_m\).
    The temporal \(L^2\)-regularization remains useful under partial observability, because it provides a
    priori control of \(\dot a\), which is exactly what bounds the transport contribution in the reduced
    joint form.
\end{remark}

\subsection{Detectability of individual directions}

We now turn from abstract reduction to constructive observability. The first step is local:
a sufficiently narrow bump sensor detects any fixed nonzero direction under mild regularity
assumptions.

\begin{proposition}[Lebesgue-point criterion for detection by localized bump sensors]
    \label{prop:lebesgue-point-detectability}
    Assume \(\Omega\subset\mathbb R^d\) is open and bounded, and let \(h\in\mathcal X_{\mathrm{ad}}\)
    be such that \(\rho_h\) admits a Lebesgue density
    \[
        r_h:=\frac{\d\rho_h}{\d x}\in L^\infty(\Omega).
    \]
    Let \(\kappa_0\in C_c^\infty(\mathbb R^d)\) be nonnegative, supported in \(B(0,1)\), and normalized by
    \[
        \int_{\mathbb R^d}\kappa_0(z)\,\d z=1.
    \]
    For \(\delta>0\), define
    \[
        \kappa_\delta(z):=\delta^{-d}\kappa_0(z/\delta).
    \]
    Fix \(\xi\in L^2_0(\nu_0)\), and set
    \[
        m_h(\xi):=\mathbb E_{\rho_h}[\xi],
        \qquad
        g_{h,\xi}(x):=r_h(x)\bigl(\xi(x)-m_h(\xi)\bigr)\in L^1(\Omega).
    \]
    Extend \(g_{h,\xi}\) by zero outside \(\Omega\).

    Let \(y\in\Omega\) satisfy
    \[
        \dist(y,\partial\Omega)>0,
    \]
    and suppose that \(y\) is a Lebesgue point of \(g_{h,\xi}\). Then for every
    \(0<\delta<\dist(y,\partial\Omega)\),
    \begin{equation}
        \label{eq:covariance-as-convolution}
        \operatorname{Cov}_{\rho_h}\!\bigl(\kappa_\delta(\cdot-y),\xi\bigr)
        =
        (\kappa_\delta * g_{h,\xi})(y),
    \end{equation}
    and moreover
    \begin{equation}
        \label{eq:lebesgue-point-detection-limit}
        \lim_{\delta\downarrow 0}
        \operatorname{Cov}_{\rho_h}\!\bigl(\kappa_\delta(\cdot-y),\xi\bigr)
        =
        g_{h,\xi}(y).
    \end{equation}
    In particular, if \(g_{h,\xi}(y)\neq 0\), then there exists \(\delta_0>0\) such that
    \[
        \operatorname{Cov}_{\rho_h}\!\bigl(\kappa_\delta(\cdot-y),\xi\bigr)\neq 0
        \qquad
        \forall\,0<\delta<\delta_0.
    \]
\end{proposition}

\begin{proof}
    Since \(\int \kappa_\delta(\cdot-y)\,\d x=1\), one has
    \[
        \operatorname{Cov}_{\rho_h}\!\bigl(\kappa_\delta(\cdot-y),\xi\bigr)
        =
        \int_\Omega \kappa_\delta(x-y)\,\xi(x)\,\d\rho_h(x)
        -
        \Bigl(\int_\Omega \kappa_\delta(x-y)\,\d\rho_h(x)\Bigr)\,m_h(\xi).
    \]
    Using \(\d\rho_h(x)=r_h(x)\,\d x\), this becomes
    \[
        \operatorname{Cov}_{\rho_h}\!\bigl(\kappa_\delta(\cdot-y),\xi\bigr)
        =
        \int_\Omega \kappa_\delta(x-y)\,r_h(x)\bigl(\xi(x)-m_h(\xi)\bigr)\,\d x
        =
        \int_{\mathbb R^d}\kappa_\delta(y-x)\,g_{h,\xi}(x)\,\d x,
    \]
    which is exactly \eqref{eq:covariance-as-convolution}.

    Because \(y\) is a Lebesgue point of \(g_{h,\xi}\), the approximate-identity theorem implies
    \[
        (\kappa_\delta*g_{h,\xi})(y)\longrightarrow g_{h,\xi}(y)
        \qquad\text{as }\delta\downarrow 0,
    \]
    which proves \eqref{eq:lebesgue-point-detection-limit}. If \(g_{h,\xi}(y)\neq 0\), then for all
    sufficiently small \(\delta\) the quantity \((\kappa_\delta*g_{h,\xi})(y)\) has the same sign as
    \(g_{h,\xi}(y)\), and in particular is nonzero.
\end{proof}

\begin{corollary}[Single-direction detectability]
    \label{cor:arbitrary-direction-detectability}
    Assume the hypotheses of Proposition~\ref{prop:lebesgue-point-detectability}, and suppose in
    addition that \(r_h>0\) almost everywhere in \(\Omega\). Let \(\xi\in L^2_0(\nu_0)\) be nonzero.
    Then there exist \(y\in\Omega\) and \(\delta_0>0\) such that
    \[
        \operatorname{Cov}_{\rho_h}\!\bigl(\kappa_\delta(\cdot-y),\xi\bigr)\neq 0
        \qquad
        \forall\,0<\delta<\delta_0.
    \]
\end{corollary}

\begin{proof}
    Since \(r_h\in L^\infty(\Omega)\) and \(\xi\in L^2(\nu_0)\), one has \(g_{h,\xi}\in L^1(\Omega)\).
    If \(g_{h,\xi}=0\) almost everywhere, then
    \[
        r_h(x)\bigl(\xi(x)-m_h(\xi)\bigr)=0
        \qquad\text{for a.e. }x.
    \]
    Because \(r_h>0\) a.e., it follows that \(\xi=m_h(\xi)\) almost everywhere. Since
    \(\xi\in L^2_0(\nu_0)\), this forces \(\xi=0\) almost everywhere, a contradiction.

    Hence \(g_{h,\xi}\) is not zero a.e. Its set of nonzero points has positive Lebesgue measure.
    By the Lebesgue differentiation theorem, almost every point of that set is a Lebesgue point of
    \(g_{h,\xi}\). Choose such a point \(y\in\Omega\) with \(g_{h,\xi}(y)\neq 0\) and
    \(\dist(y,\partial\Omega)>0\). Proposition~\ref{prop:lebesgue-point-detectability} now yields the
    claim.
\end{proof}

\begin{corollary}[Continuous-direction criterion]
    \label{cor:continuous-direction-detectability}
    Assume the hypotheses of Proposition~\ref{prop:lebesgue-point-detectability}, and suppose in
    addition that
    \[
        r_h\in C(\overline\Omega),
        \qquad
        r_h(x)>0\ \text{ on }\overline\Omega,
        \qquad
        \xi\in C(\overline\Omega)\cap L^2_0(\nu_0),
        \qquad
        \xi\not\equiv 0.
    \]
    Then there exist \(y\in\Omega\) and \(\delta_0>0\) such that
    \[
        \operatorname{Cov}_{\rho_h}\!\bigl(\kappa_\delta(\cdot-y),\xi\bigr)\neq 0
        \qquad
        \forall\,0<\delta<\delta_0.
    \]
    Indeed, one may choose any \(y\in\Omega\) such that
    \[
        \xi(y)\neq \mathbb E_{\rho_h}[\xi].
    \]
\end{corollary}

\begin{proof}
    Because \(\xi\) is continuous and nonzero with \(\xi\in L^2_0(\nu_0)\), it cannot be constant on
    \(\Omega\). Hence \(\xi-\mathbb E_{\rho_h}[\xi]\) is not identically zero, so there exists
    \(y\in\Omega\) such that
    \[
        \xi(y)\neq \mathbb E_{\rho_h}[\xi].
    \]
    For such \(y\),
    \[
        g_{h,\xi}(y)=r_h(y)\bigl(\xi(y)-\mathbb E_{\rho_h}[\xi]\bigr)\neq 0.
    \]
    Since both \(r_h\) and \(\xi\) are continuous, \(g_{h,\xi}\) is continuous, hence every point is a
    Lebesgue point of \(g_{h,\xi}\). The conclusion follows from
    Proposition~\ref{prop:lebesgue-point-detectability}.
\end{proof}

\subsection{Instantaneous reduced observability with multiple static sensors}

The previous subsection treats one direction at a time. We now show that on a fixed
\(m\)-dimensional reduced subspace one can place finitely many stationary sensors so that the
whole reduced state is visible at each time.

\begin{lemma}[Evaluation points for linearly independent continuous functions]
    \label{lem:evaluation-points-independent}
    Let \(K\) be a compact Hausdorff space, and let
    \[
        f_0,\dots,f_m\in C(K)
    \]
    be linearly independent over \(\mathbb R\). Then there exist points
    \[
        y_0,\dots,y_m\in K
    \]
    such that the square matrix
    \[
        A_Y
        :=
        \bigl(f_\ell(y_j)\bigr)_{j,\ell=0}^m
        \in\mathbb R^{(m+1)\times(m+1)}
    \]
    is invertible.
\end{lemma}

\begin{proof}
    For \(y\in K\), define the evaluation vector
    \[
        v(y):=\bigl(f_0(y),\dots,f_m(y)\bigr)\in\mathbb R^{m+1}.
    \]
    If the span of \(\{v(y):y\in K\}\) had dimension strictly smaller than \(m+1\), then there would
    exist \(c=(c_0,\dots,c_m)\neq 0\) such that
    \[
        c\cdot v(y)=\sum_{\ell=0}^m c_\ell f_\ell(y)=0
        \qquad\forall y\in K,
    \]
    contradicting the linear independence of \(f_0,\dots,f_m\). Hence
    \(\{v(y):y\in K\}\) spans \(\mathbb R^{m+1}\), so one can choose
    \(y_0,\dots,y_m\in K\) such that \(v(y_0),\dots,v(y_m)\) form a basis of \(\mathbb R^{m+1}\).
    Equivalently, the matrix \(A_Y\) is invertible.
\end{proof}

\begin{theorem}[Uniform instantaneous reduced observability by fixed static sensors]
    \label{thm:static-reduced-full-observability}
    Let
    \[
        V_m=\operatorname{span}\{\phi_1,\dots,\phi_m\}\subset C(K_{\mathrm{sens}})\cap L^2_0(\nu_0),
    \]
    where \(K_{\mathrm{sens}}\subset\Omega\) is a nonempty compact sensor-feasible region. Assume that the
    restrictions of
    \[
        1,\phi_1,\dots,\phi_m
    \]
    to \(K_{\mathrm{sens}}\) are linearly independent. Let \(U\subset\mathbb R^m\) be compact, and define
    \[
        h(a):=\sum_{k=1}^m a_k\phi_k,
        \qquad a\in U.
    \]
    Assume:
    \begin{enumerate}
        \item \(h(a)\in\mathcal X_{\mathrm{ad}}\) for every \(a\in U\);
        \item for every \(a\in U\), the probability measure \(\rho_{h(a)}\) has a Lebesgue density
              \[
                  r_a:=\frac{\d\rho_{h(a)}}{\d x}\in C(\overline\Omega),
              \]
              and \((a,x)\mapsto r_a(x)\) is continuous on \(U\times\overline\Omega\);
        \item there exists \(c_\rho>0\) such that
              \[
                  r_a(x)\ge c_\rho
                  \qquad
                  \forall (a,x)\in U\times K_{\mathrm{sens}};
              \]
        \item \(\kappa_\delta(z)=\delta^{-d}\kappa_0(z/\delta)\), where
              \(\kappa_0\in C_c^\infty(\mathbb R^d)\) is nonnegative, supported in \(B(0,1)\), and
              \(\int_{\mathbb R^d}\kappa_0=1\).
    \end{enumerate}
    Then there exist points
    \[
        y_0,\dots,y_m\in K_{\mathrm{sens}}
    \]
    and \(\delta_0>0\) such that for every \(0<\delta<\delta_0\) and every \(a\in U\), the static
    reduced sensing matrix
    \[
        M_Y^\delta(a)\in\mathbb R^{(m+1)\times m},
        \qquad
        [M_Y^\delta(a)]_{jk}
        :=
        \operatorname{Cov}_{\rho_{h(a)}}(\kappa_\delta(\cdot-y_j),\phi_k),
    \]
    has full column rank \(m\). Moreover, there exists \(\sigma_*>0\) such that
    \begin{equation}
        \label{eq:uniform-static-gramian-lower-bound}
        M_Y^\delta(a)^\top M_Y^\delta(a)\succeq \sigma_* I_m
        \qquad
        \forall a\in U,\ \forall 0<\delta<\delta_0.
    \end{equation}
\end{theorem}

\begin{proof}
    Apply Lemma~\ref{lem:evaluation-points-independent} to the functions
    \[
        f_0:=1,\qquad f_k:=\phi_k,\quad k=1,\dots,m,
    \]
    on the compact space \(K_{\mathrm{sens}}\). This yields points \(y_0,\dots,y_m\in K_{\mathrm{sens}}\)
    such that the augmented evaluation matrix
    \[
        A_Y
        :=
        \begin{pmatrix}
            1      & \phi_1(y_0) & \cdots & \phi_m(y_0) \\
            \vdots & \vdots      &        & \vdots      \\
            1      & \phi_1(y_m) & \cdots & \phi_m(y_m)
        \end{pmatrix}
        \in\mathbb R^{(m+1)\times(m+1)}
    \]
    is invertible.

    For \(a\in U\), define
    \[
        c(a):=
        \bigl(c_1(a),\dots,c_m(a)\bigr)^\top,
        \qquad
        c_k(a):=\mathbb E_{\rho_{h(a)}}[\phi_k].
    \]
    Also define
    \[
        X_Y:=\bigl(\phi_k(y_j)\bigr)_{j=0,\dots,m}^{k=1,\dots,m}\in\mathbb R^{(m+1)\times m},
        \qquad
        \mathbf 1:=(1,\dots,1)^\top\in\mathbb R^{m+1},
    \]
    and
    \[
        D_Y(a):=\operatorname{diag}(r_a(y_0),\dots,r_a(y_m)).
    \]
    Consider the ideal point-sensor matrix
    \[
        \widetilde M_Y(a):=D_Y(a)\bigl(X_Y-\mathbf 1\,c(a)^\top\bigr).
    \]

    We first show that \(\widetilde M_Y(a)\) has full column rank \(m\) for every \(a\in U\).
    Suppose
    \[
        \widetilde M_Y(a)\alpha=0
        \qquad\text{for some }\alpha\in\mathbb R^m.
    \]
    Since \(D_Y(a)\) is invertible by the uniform positivity of \(r_a\) on \(K_{\mathrm{sens}}\), this implies
    \[
        \bigl(X_Y-\mathbf 1\,c(a)^\top\bigr)\alpha=0,
    \]
    hence
    \[
        X_Y\alpha=(c(a)^\top\alpha)\mathbf 1.
    \]
    Therefore
    \[
        A_Y
        \binom{-c(a)^\top\alpha}{\alpha}
        =
        0.
    \]
    Because \(A_Y\) is invertible, we conclude \(\alpha=0\). Thus
    \(\operatorname{rank}\widetilde M_Y(a)=m\).

    Next, the map \(a\mapsto \widetilde M_Y(a)\) is continuous on \(U\): the values \(r_a(y_j)\) are
    continuous in \(a\), and the averages \(c_k(a)=\int \phi_k\,\d\rho_{h(a)}\) are continuous by
    dominated convergence. Hence the smallest singular value of \(\widetilde M_Y(a)\) is a positive
    continuous function on the compact set \(U\), so there exists \(\sigma_*>0\) such that
    \[
        \widetilde M_Y(a)^\top \widetilde M_Y(a)\succeq 2\sigma_* I_m
        \qquad
        \forall a\in U.
    \]

    It remains to compare the actual bump-sensor matrix \(M_Y^\delta(a)\) with
    \(\widetilde M_Y(a)\). Because \(K_{\mathrm{sens}}\subset\Omega\) is compact, there exists
    \[
        d_*:=\dist(K_{\mathrm{sens}},\partial\Omega)>0.
    \]
    For \(0<\delta<d_*\), the support of \(\kappa_\delta(\cdot-y_j)\) stays inside \(\Omega\) for every
    \(j\). For each \(j,k\),
    \[
        [M_Y^\delta(a)]_{jk}
        =
        \int_\Omega \kappa_\delta(x-y_j)\,r_a(x)\bigl(\phi_k(x)-c_k(a)\bigr)\,\d x.
    \]
    Define
    \[
        g_{a,k}(x):=r_a(x)\bigl(\phi_k(x)-c_k(a)\bigr).
    \]
    Since \((a,x)\mapsto r_a(x)\) is continuous on \(U\times\overline\Omega\), each \(\phi_k\) is
    continuous on \(K_{\mathrm{sens}}\), and \(a\mapsto c_k(a)\) is continuous, the family
    \(\{g_{a,k}:a\in U\}\) is uniformly continuous in a neighborhood of \(K_{\mathrm{sens}}\). Hence the
    approximate-identity property yields
    \[
        \sup_{a\in U}\bigl|[M_Y^\delta(a)]_{jk}-g_{a,k}(y_j)\bigr|\longrightarrow 0
        \qquad\text{as }\delta\downarrow 0.
    \]
    But \(g_{a,k}(y_j)=[\widetilde M_Y(a)]_{jk}\), so
    \[
        \sup_{a\in U}\|M_Y^\delta(a)-\widetilde M_Y(a)\|\longrightarrow 0
        \qquad\text{as }\delta\downarrow 0.
    \]
    Choosing \(\delta_0>0\) sufficiently small, we may arrange
    \[
        M_Y^\delta(a)^\top M_Y^\delta(a)\succeq \sigma_* I_m
        \qquad
        \forall a\in U,\ \forall 0<\delta<\delta_0.
    \]
    This proves both full column rank and \eqref{eq:uniform-static-gramian-lower-bound}.
\end{proof}

\begin{remark}
    \label{rem:localized-sensor-fully-observable-special-case}
    Theorem~\ref{thm:static-reduced-full-observability} places the reduced problem in a fully
    observable regime: with \(m+1\) suitably chosen stationary sensors, the reduced sensing matrix is
    uniformly coercive on the whole reduced family \(U\). The next subsection shows that one may
    trade this instantaneous full observability for time-aggregated observability by moving a single
    sensor through those same locations.
\end{remark}

\subsection{Pathwise reduced observability and one moving sensor}

We now return to the pathwise viewpoint. In finite dimensions, the correct design target for
mobile sensing is the averaged reduced Gramian along the reduced path.

\begin{proposition}[Time-averaged coercivity on constant reduced modes]
    \label{prop:finite-dimensional-constant-coercivity}
    Fix \(a\in \mathcal A_{\mathrm{ad}}^{(m)}\) and a measurable sensor trajectory \(y\). Define the
    averaged reduced Gramian
    \begin{equation}
        \label{eq:averaged-reduced-gramian}
        \overline G_y[a]
        :=
        \int_0^T M_y(t,a(t))^\top M_y(t,a(t))\,\d t.
    \end{equation}
    Then the following are equivalent:
    \begin{enumerate}
        \item \(\overline G_y[a]\) is positive definite;
        \item there exists \(\kappa_{\mathrm{obs}}^{(m)}>0\) such that
              \[
                  \int_0^T
                  \|J_{t,h(a(t))}^{y}\xi_0\|_{\mathbb R^r}^2\,\d t
                  \ge
                  \kappa_{\mathrm{obs}}^{(m)}\|\xi_0\|_{L^2(\nu_0)}^2
                  \qquad
                  \forall \xi_0\in V_m;
              \]
        \item \(\ker Q_y[h(a)]\cap V_m=\{0\}\), where \(V_m\) is identified with time-constant perturbations.
    \end{enumerate}
\end{proposition}

\begin{proof}
    Let
    \[
        B_m:=\bigl(\langle \phi_k,\phi_\ell\rangle_{L^2(\nu_0)}\bigr)_{k,\ell=1}^m
        \in \mathbb R^{m\times m}.
    \]
    Since \(\{\phi_1,\dots,\phi_m\}\) is a basis of \(V_m\), \(B_m\) is symmetric positive definite.

    Write
    \[
        \xi_0=\sum_{k=1}^m c_k\phi_k,
        \qquad
        c=(c_1,\dots,c_m)^\top\in\mathbb R^m.
    \]
    Then
    \begin{equation}
        \label{eq:xi0-norm-gram}
        \|\xi_0\|_{L^2(\nu_0)}^2
        =
        c^\top B_m c.
    \end{equation}
    By Proposition~\ref{prop:reduced-mobile-sensing-matrix},
    \[
        J_{t,h(a(t))}^{y}\xi_0
        =
        M_y(t,a(t))\,c,
    \]
    and therefore
    \begin{align}
        \int_0^T \|J_{t,h(a(t))}^{y}\xi_0\|_{\mathbb R^r}^2\,\d t
         & =
        \int_0^T c^\top M_y(t,a(t))^\top M_y(t,a(t))\,c\,\d t \notag \\
         & =
        c^\top \overline G_y[a]\,c.
        \label{eq:obs-energy-gram}
    \end{align}

    Assume \((1)\). Since both \(B_m\) and \(\overline G_y[a]\) are symmetric positive definite, the
    matrix
    \[
        A_m:=B_m^{-1/2}\,\overline G_y[a]\,B_m^{-1/2}
    \]
    is symmetric positive definite. Hence
    \[
        \lambda_*:=\lambda_{\min}(A_m)>0.
    \]
    Using \eqref{eq:xi0-norm-gram} and \eqref{eq:obs-energy-gram},
    \[
        c^\top \overline G_y[a]\,c
        =
        (B_m^{1/2}c)^\top A_m (B_m^{1/2}c)
        \ge
        \lambda_*\,|B_m^{1/2}c|^2
        =
        \lambda_*\,\|\xi_0\|_{L^2(\nu_0)}^2.
    \]
    Thus \((2)\) holds.

    Conversely, if \((2)\) holds and \(\overline G_y[a]\) were not positive definite, there would exist
    \(c\neq 0\) such that
    \[
        c^\top \overline G_y[a]\,c=0.
    \]
    Then the associated \(\xi_0=\sum_{k=1}^m c_k\phi_k\) is nonzero, while
    \eqref{eq:obs-energy-gram} contradicts \((2)\). Hence \((1)\) follows.

    Finally, a time-constant perturbation \(\xi(t)\equiv \xi_0\in V_m\) has \(\dot\xi=0\), so its
    transport contribution vanishes. Therefore
    \[
        Q_y[h(a)](\xi_0,\xi_0)
        =
        \int_0^T \|J_{t,h(a(t))}^{y}\xi_0\|_{\mathbb R^r}^2\,\d t.
    \]
    Thus
    \[
        \ker Q_y[h(a)]\cap V_m
        =
        \left\{
        \xi_0\in V_m:
        \int_0^T \|J_{t,h(a(t))}^{y}\xi_0\|_{\mathbb R^r}^2\,\d t=0
        \right\},
    \]
    which shows \((2)\Longleftrightarrow(3)\).
\end{proof}

The previous proposition identifies the reduced design target: one should choose the sensor
trajectory so that \(\overline G_y[a]\) is positive definite. In the single-sensor setting, this means
that the time-average of the rank-one reduced observation matrices must span the whole reduced
coefficient space.

For the remainder of this subsection we consider a single moving sensor, so \(r=1\). We write
\[
    \mathcal Y_{\mathrm{ad}}^{(1)}
    :=
    \left\{
    y\in H^1(0,T;\mathbb R^d):
    y(t)\in K_{\mathrm{sens}}
    \text{ for a.e.\ }t\in[0,T]
    \right\}.
\]
We also work on the global reduced admissible class
\[
    \mathcal A_m(U)
    :=
    \left\{
    a\in H^1(0,T;\mathbb R^m):
    a(t)\in U
    \text{ for a.e.\ }t\in[0,T]
    \right\},
\]
where \(U\subset \mathbb R^m\) is compact.

\begin{remark}
    \label{rem:global-reduced-class-natural}
    The class \(\mathcal A_m(U)\) is the natural global reduced admissible class for the variational
    theory, since Remark~\ref{rem:reduced-norm-equivalence-last-subsection} shows that the reduced
    regularization is equivalent to a Euclidean \(H^1\)-in-time penalty on the coefficient path.
\end{remark}

To pass from the static observability theorem to a single moving sensor, we need a mild geometric
assumption on the feasible sensor region: any two feasible positions can be joined by an
\(H^1\)-curve that remains feasible.

\begin{theorem}[Path-dependent single-sensor trajectory design]
    \label{thm:path-dependent-mobile-design}
    Let \(U\subset\mathbb R^m\) be compact, and assume the hypotheses of
    Theorem~\ref{thm:static-reduced-full-observability}. Assume in addition that the feasible sensor
    region \(K_{\mathrm{sens}}\subset \Omega\) is \(H^1\)-path connected in the sense that for every
    \(p,q\in K_{\mathrm{sens}}\) there exists
    \[
        \gamma_{p,q}\in H^1([0,1];\mathbb R^d)
    \]
    such that
    \[
        \gamma_{p,q}(0)=p,\qquad
        \gamma_{p,q}(1)=q,\qquad
        \gamma_{p,q}(s)\in K_{\mathrm{sens}}
        \quad\forall s\in[0,1].
    \]

    Fix any
    \[
        a\in \mathcal A_m(U).
    \]
    Then there exists a single-sensor trajectory
    \[
        y[a]\in \mathcal Y_{\mathrm{ad}}^{(1)}
    \]
    such that the averaged reduced Gramian
    \[
        \overline G_{y[a]}[a]
        =
        \int_0^T M_{y[a]}(t,a(t))^\top M_{y[a]}(t,a(t))\,\d t
    \]
    is positive definite.
\end{theorem}

\begin{proof}
    Choose once and for all a kernel width \(0<\delta<\delta_0\), where \(\delta_0\) is the radius
    given by Theorem~\ref{thm:static-reduced-full-observability}. We suppress the dependence on
    \(\delta\) in the notation.

    By Theorem~\ref{thm:static-reduced-full-observability}, there exist points
    \[
        z_0,\dots,z_m\in K_{\mathrm{sens}}
    \]
    and a constant \(\sigma_*>0\) such that for every \(b\in U\),
    \begin{equation}
        \label{eq:static-uniform-coercivity-proof}
        \sum_{j=0}^m M_{z_j}(b)^\top M_{z_j}(b)
        \succeq
        \sigma_* I_m,
    \end{equation}
    where \(M_{z_j}(b)\in \mathbb R^{1\times m}\) denotes the reduced sensing row for a single static
    sensor at the location \(z_j\).

    For each \(j=0,\dots,m-1\), choose an \(H^1\)-curve in \(K_{\mathrm{sens}}\) connecting
    \(z_j\) to \(z_{j+1}\), and choose an \(H^1\)-curve connecting \(z_m\) back to \(z_0\). Because the
    set of curves is finite, we may regard these as fixed transition templates.

    For the fixed path \(a\in \mathcal A_m(U)\), define
    \[
        G_j(t):=M_{z_j}(a(t))^\top M_{z_j}(a(t))
        \in \mathbb R^{m\times m},
        \qquad
        G_\Sigma(t):=\sum_{j=0}^m G_j(t).
    \]
    Since \(a\in H^1(0,T;\mathbb R^m)\subset C([0,T];\mathbb R^m)\) and
    \(b\mapsto M_{z_j}(b)\) is continuous on \(U\), each \(G_j\) is continuous on \([0,T]\), hence
    uniformly continuous. Also, by \eqref{eq:static-uniform-coercivity-proof},
    \begin{equation}
        \label{eq:Gsigma-lower-bound}
        G_\Sigma(t)\succeq \sigma_* I_m
        \qquad\forall t\in[0,T].
    \end{equation}

    Let \(q:=m+1\), and choose \(\eta>0\) so small that
    \[
        q\eta \le \frac{\sigma_*}{2}.
    \]
    By uniform continuity, there exists a partition
    \[
        0=t_0<t_1<\cdots<t_N=T
    \]
    such that on each interval \(I_n:=[t_{n-1},t_n]\),
    \begin{equation}
        \label{eq:Gj-small-variation}
        \|G_j(t)-G_j(s)\|\le \eta
        \qquad
        \forall s,t\in I_n,\ \forall j=0,\dots,m.
    \end{equation}

    We now construct \(y[a]\) interval by interval. On each \(I_n\), first let the sensor dwell at
    \(z_0,\dots,z_m\) in that order, each for a time interval of length
    \[
        \frac{|I_n|}{2q}.
    \]
    Use the remaining half of \(I_n\) to traverse the chosen transition curves between successive
    locations, scaled in time so as to fit. This gives a piecewise \(H^1\) path on each \(I_n\), and by
    construction the endpoint of the last transition on \(I_n\) matches the initial position of the first
    dwell on \(I_{n+1}\). Hence the resulting global path \(y[a]\) belongs to
    \(\mathcal Y_{\mathrm{ad}}^{(1)}\).

    Let \(D_{n,j}\subset I_n\) denote the dwell interval on which \(y[a](t)=z_j\). Since the transition
    contributions are positive semidefinite, we have
    \[
        \overline G_{y[a]}[a]
        \succeq
        \sum_{n=1}^N \sum_{j=0}^m \int_{D_{n,j}} G_j(t)\,\d t.
    \]
    Fix \(n\in\{1,\dots,N\}\) and choose any \(s_n\in I_n\). By \eqref{eq:Gj-small-variation},
    \[
        G_j(t)\succeq G_j(s_n)-\eta I_m
        \qquad
        \forall t\in I_n,\ \forall j=0,\dots,m.
    \]
    Therefore
    \begin{align*}
        \sum_{j=0}^m \int_{D_{n,j}} G_j(t)\,\d t
         & \succeq
        \sum_{j=0}^m \frac{|I_n|}{2q}\bigl(G_j(s_n)-\eta I_m\bigr) \\
         & =
        \frac{|I_n|}{2q}
        \left(
        G_\Sigma(s_n)-q\eta I_m
        \right).
    \end{align*}
    Using \eqref{eq:Gsigma-lower-bound} and the choice of \(\eta\),
    \[
        G_\Sigma(s_n)-q\eta I_m
        \succeq
        \frac{\sigma_*}{2} I_m.
    \]
    Hence
    \[
        \sum_{j=0}^m \int_{D_{n,j}} G_j(t)\,\d t
        \succeq
        \frac{|I_n|\,\sigma_*}{4q} I_m.
    \]
    Summing over \(n\) gives
    \[
        \overline G_{y[a]}[a]
        \succeq
        \sum_{n=1}^N \frac{|I_n|\,\sigma_*}{4q} I_m
        =
        \frac{T\,\sigma_*}{4(m+1)} I_m.
    \]
    Thus \(\overline G_{y[a]}[a]\) is positive definite.
\end{proof}

\begin{remark}[Path-dependent design]
    \label{rem:path-dependent-sensor-design}
    The trajectory \(y[a]\) in Theorem~\ref{thm:path-dependent-mobile-design} is allowed to depend
    on the reduced path \(a\). Thus, when the theorem is applied to a ground-truth path
    \(a^\dagger\), the resulting sensor trajectory may depend on that ground truth. At this point we
    are not claiming that a single trajectory can be designed a priori, independently of the unknown
    reduced path, so as to give the same averaged-Gramian positivity uniformly over a whole
    admissible class.
\end{remark}

Once a path \(a^\dagger\) and an associated trajectory \(y^\dagger\) have been fixed, positivity of the
averaged reduced Gramian upgrades immediately to full pathwise coercivity of the reduced joint
form. This is the finite-dimensional mechanism by which the transport term controls oscillatory
modes and the averaged observation term controls the surviving constant mode.

\begin{theorem}[Fixed-path reduced coercivity from a positive averaged Gramian]
    \label{thm:local-reduced-coercivity-mobile}
    Let \(U\subset\mathbb R^m\) be compact, let
    \[
        a^\dagger\in \mathcal A_m(U),
        \qquad
        y^\dagger\in \mathcal Y_{\mathrm{ad}}^{(1)},
    \]
    and assume:
    \begin{enumerate}
        \item the reduced transport tensor \(H(a)\) is continuous on \(U\) and uniformly positive definite:
              \begin{equation}
                  \label{eq:uniform-transport-lower-bound}
                  H(a)\succeq c_{\mathrm{tr}} I_m
                  \qquad\text{for all }a\in U
              \end{equation}
              for some \(c_{\mathrm{tr}}>0\);
        \item the reduced sensing row along the fixed pair \((a^\dagger,y^\dagger)\) satisfies
              \[
                  t\longmapsto M_{y^\dagger}(t,a^\dagger(t))
                  \in L^\infty(0,T;\mathbb R^{1\times m});
              \]
        \item the averaged reduced Gramian along the fixed path is positive definite:
              \begin{equation}
                  \label{eq:positive-averaged-gramian-reference}
                  \overline G_{y^\dagger}[a^\dagger]
                  =
                  \int_0^T M_{y^\dagger}(t,a^\dagger(t))^\top
                  M_{y^\dagger}(t,a^\dagger(t))\,\d t
                  \succ 0.
              \end{equation}
    \end{enumerate}
    Then there exists \(\kappa_m^\dagger>0\) such that
    \begin{equation}
        \label{eq:fixed-path-reduced-coercivity}
        \int_0^T \alpha(t)^\top
        M_{y^\dagger}(t,a^\dagger(t))^\top
        M_{y^\dagger}(t,a^\dagger(t))
        \alpha(t)\,\d t
        +
        \int_0^T \dot\alpha(t)^\top H(a^\dagger(t))\,\dot\alpha(t)\,\d t
        \ge
        \kappa_m^\dagger \int_0^T |\alpha(t)|^2\,\d t
    \end{equation}
    for every \(\alpha\in H^1(0,T;\mathbb R^m)\).
\end{theorem}

\begin{proof}
    Suppose the conclusion fails. Then there exists a sequence
    \[
        \alpha_n\in H^1(0,T;\mathbb R^m),
        \qquad
        \|\alpha_n\|_{L^2(0,T;\mathbb R^m)}=1,
    \]
    such that
    \begin{equation}
        \label{eq:fixed-path-coercivity-contradiction}
        \int_0^T \alpha_n^\top
        M_{y^\dagger}(t,a^\dagger(t))^\top
        M_{y^\dagger}(t,a^\dagger(t))
        \alpha_n\,\d t
        +
        \int_0^T \dot\alpha_n^\top H(a^\dagger(t))\,\dot\alpha_n\,\d t
        \longrightarrow 0.
    \end{equation}

    By \eqref{eq:uniform-transport-lower-bound},
    \[
        c_{\mathrm{tr}}\int_0^T |\dot\alpha_n(t)|^2\,\d t
        \le
        \int_0^T \dot\alpha_n(t)^\top H(a^\dagger(t))\,\dot\alpha_n(t)\,\d t
        \longrightarrow 0,
    \]
    so
    \[
        \|\dot\alpha_n\|_{L^2(0,T;\mathbb R^m)}\to 0.
    \]
    Let
    \[
        \bar\alpha_n:=\frac1T\int_0^T \alpha_n(t)\,\d t.
    \]
    By the Poincar\'e--Wirtinger inequality,
    \[
        \|\alpha_n-\bar\alpha_n\|_{L^2(0,T;\mathbb R^m)}
        \le
        C_P \|\dot\alpha_n\|_{L^2(0,T;\mathbb R^m)}
        \to 0.
    \]
    Since \((\bar\alpha_n)\) is bounded in \(\mathbb R^m\), after passing to a subsequence we may
    assume
    \[
        \bar\alpha_n\to c\in\mathbb R^m.
    \]
    Therefore
    \[
        \alpha_n\to c
        \qquad\text{strongly in }L^2(0,T;\mathbb R^m),
    \]
    where \(c\) is identified with the corresponding time-constant function on \([0,T]\). Since
    \(\|\alpha_n\|_{L^2}=1\), we have
    \[
        \|c\|_{L^2(0,T;\mathbb R^m)}=1,
    \]
    so \(c\neq 0\).

    Set
    \[
        M^\dagger(t):=M_{y^\dagger}(t,a^\dagger(t)).
    \]
    Because \(M^\dagger\in L^\infty(0,T;\mathbb R^{1\times m})\),
    \[
        M^\dagger \alpha_n \to M^\dagger c
        \qquad\text{strongly in }L^2(0,T;\mathbb R).
    \]
    Passing to the limit in the first term of \eqref{eq:fixed-path-coercivity-contradiction} gives
    \[
        0
        =
        \int_0^T |M^\dagger(t)c|^2\,\d t
        =
        c^\top \overline G_{y^\dagger}[a^\dagger]\,c.
    \]
    This contradicts \eqref{eq:positive-averaged-gramian-reference}, since
    \(\overline G_{y^\dagger}[a^\dagger]\succ 0\) and \(c\neq 0\). Hence
    \eqref{eq:fixed-path-reduced-coercivity} holds.
\end{proof}

\begin{remark}
    \label{rem:fixed-path-versus-uniform-path-class}
    Theorem~\ref{thm:path-dependent-mobile-design} is global in the reduced path class
    \(\mathcal A_m(U)\), but the resulting trajectory \(y[a]\) depends on the chosen path \(a\).
    Theorem~\ref{thm:local-reduced-coercivity-mobile} then gives a global coercivity statement along
    that fixed path, rather than a neighborhood theorem uniform over nearby reduced paths.
    This is exactly the form needed when the sensor design is allowed to depend on the true reduced
    path.
\end{remark}

\subsection{Example: a simple two-dimensional moving-sensor design}
\label{subsec:2d-moving-sensor-example}

We illustrate the preceding theory on a simple two-dimensional reduction with a single moving
localized sensor.

Let
\[
    \Omega=(0,1)^2,
    \qquad
    \nu_0 = \d x,
\]
and consider the two-dimensional reduced space
\[
    V_2:=\operatorname{span}\{\phi_1,\phi_2\},
    \qquad
    \phi_1(x)=x_1-\tfrac12,
    \qquad
    \phi_2(x)=x_2-\tfrac12.
\]
For \(a=(a_1,a_2)\in\mathbb R^2\), define
\[
    h(a):=a_1\phi_1+a_2\phi_2.
\]
The associated probability density is
\[
    \rho_{h(a)}(x)
    =
    \frac{\exp(h(a)(x))}{\int_\Omega \exp(h(a)(z))\,\d z}.
\]

Choose a nonzero reference reduced path
\[
    a^\dagger(t)
    =
    \bigl(\bar a_1+\varepsilon \cos(2\pi t/T),\;
    \bar a_2+\varepsilon \sin(2\pi t/T)\bigr),
    \qquad t\in[0,T],
\]
where \(\bar a=(\bar a_1,\bar a_2)\neq 0\) and \(\varepsilon>0\) is small. Thus
\[
    h^\dagger(t,x)
    =
    \bigl(\bar a_1+\varepsilon \cos(2\pi t/T)\bigr)\Bigl(x_1-\tfrac12\Bigr)
    +
    \bigl(\bar a_2+\varepsilon \sin(2\pi t/T)\bigr)\Bigl(x_2-\tfrac12\Bigr).
\]

Let \(\kappa_0\in C_c^\infty(\mathbb R^2)\) be nonnegative, radial, supported in \(B(0,1)\), and
normalized by \(\int_{\mathbb R^2}\kappa_0(z)\,\d z=1\). For \(\delta>0\), define
\[
    \kappa_\delta(x):=\delta^{-2}\kappa_0(x/\delta).
\]
Given a sensor position \(y\in \Omega\), define
\[
    M(y,a)
    :=
    \bigl(
    \operatorname{Cov}_{\rho_{h(a)}}(\kappa_\delta(\,\cdot-y),\phi_1),\;
    \operatorname{Cov}_{\rho_{h(a)}}(\kappa_\delta(\,\cdot-y),\phi_2)
    \bigr)
    \in \mathbb R^{1\times 2}.
\]
For a sensor trajectory \(y:[0,T]\to\Omega\), write
\[
    M_y(t,a):=M(y(t),a).
\]

We choose two dwell locations,
\[
    p^{(1)}=\Bigl(\tfrac78,\tfrac12\Bigr),
    \qquad
    p^{(2)}=\Bigl(\tfrac12,\tfrac78\Bigr),
\]
and let \(y\in H^1(0,T;\mathbb R^2)\) be any trajectory such that
\[
    y(t)=p^{(1)}
    \quad\text{for }t\in\Bigl[0,\frac{T}{2}-\tau\Bigr],
    \qquad
    y(t)=p^{(2)}
    \quad\text{for }t\in\Bigl[\frac{T}{2}+\tau,T\Bigr],
\]
with a smooth transition on \([T/2-\tau,T/2+\tau]\), where \(0<\tau<T/4\).

The key point is that the two dwell locations favor the two reduced directions differently.
At the unperturbed state \(a=0\) and in the point-sensor limit \(\delta\to 0\),
\[
    M(p^{(1)},0)\approx
    \bigl(\phi_1(p^{(1)}),\phi_2(p^{(1)})\bigr)
    =
    \Bigl(\tfrac38,0\Bigr),
\]
and
\[
    M(p^{(2)},0)\approx
    \bigl(\phi_1(p^{(2)}),\phi_2(p^{(2)})\bigr)
    =
    \Bigl(0,\tfrac38\Bigr).
\]
Hence the ideal averaged reduced Gramian is
\[
    \overline G_{\mathrm{ideal}}
    =
    \int_0^{T/2}
    \begin{pmatrix}\tfrac38\\[1mm]0\end{pmatrix}
    \begin{pmatrix}\tfrac38 & 0\end{pmatrix}\,\d t
    +
    \int_{T/2}^{T}
    \begin{pmatrix}0\\[1mm]\tfrac38\end{pmatrix}
    \begin{pmatrix}0 & \tfrac38\end{pmatrix}\,\d t
    =
    \frac{9T}{128}\,I_2,
\]
which is positive definite.

Since \(a\mapsto \rho_{h(a)}\) and \((y,a)\mapsto M(y,a)\) are continuous on compact sets, the
actual averaged reduced Gramian stays positive definite for small perturbations of this ideal
configuration.

\begin{proposition}
    \label{prop:2d-moving-sensor-example}
    There exist parameters
    \[
        \bar a\in\mathbb R^2\setminus\{0\},
        \qquad
        \varepsilon>0,
        \qquad
        \delta>0,
        \qquad
        \tau\in(0,T/4),
    \]
    such that the above nonzero reference path \(a^\dagger\) and moving sensor trajectory \(y\)
    satisfy
    \[
        \overline G_y[a^\dagger]\succ 0.
    \]
    Consequently, if \(U\subset\mathbb R^2\) is a compact neighborhood of the image of \(a^\dagger\)
    on which the reduced transport tensor \(H(a)\) is continuous and uniformly positive definite, then
    the conclusion of Theorem~\ref{thm:local-reduced-coercivity-mobile} applies along \(a^\dagger\).
\end{proposition}

\begin{proof}
    The idealized Gramian computed above is \(\frac{9T}{128}I_2\), hence positive definite.
    By continuity of the reduced sensing map with respect to the state, sensor location, and kernel
    width, the actual averaged Gramian is a small perturbation of this ideal matrix whenever
    \(|\bar a|\), \(\varepsilon\), \(\delta\), and \(\tau\) are sufficiently small. Positive definiteness is stable
    under small perturbations, so \(\overline G_y[a^\dagger]\succ 0\). The final claim then follows
    from Theorem~\ref{thm:local-reduced-coercivity-mobile}.
\end{proof}

\subsection{Interpretation}

The finite-dimensional theory is a direct restriction of the ambient theory, but its recovery
consequences are stronger.

Indeed, the reduced transport matrix \(H(a)\) is the coordinate form of the ambient transport
geometry \(\mathfrak g_h\), and \(M_y(t,a)^\top M_y(t,a)\) is the coordinate form of the
instantaneous ambient observability form \(\mathfrak j_{t,h(a)}^y\). Thus the reduced joint form is
exactly the pullback of the ambient form \(Q_y[h]\) to the coefficient space. In this sense, the
finite-dimensional model is not a separate inverse problem: it is the ambient inverse problem
restricted to a finite-dimensional Bayes Hilbert subspace.

At the same time, restriction changes the observability mechanism in an essential way. In the
ambient infinite-dimensional setting, localized mobile sensors can shrink the common invisible
subspace but cannot produce coercivity on the full state space, because the time-averaged
observation operator remains compact. In the reduced setting, that compactness obstruction
disappears. One can first detect individual directions by localized bump sensors, then choose
finitely many static locations that make the reduced state instantaneously observable, and finally
trade those multiple static sensors for a single moving sensor whose time-averaged reduced
Gramian is positive definite. When the reduced transport tensor is uniformly positive definite, the
transport term upgrades that averaged observability to coercivity of the reduced joint form along
the chosen reduced path.

It is important, however, to distinguish what has and has not been proved at this stage.
The sensor-design theorems in this section are pathwise: they produce observability and coercivity
statements along a fixed reduced truth path, and in the moving-sensor case the trajectory may
depend on that truth path. They do not by themselves yield a single global inverse theorem over
an entire reduced admissible class. That stronger conclusion requires an additional classwise
reduced stability property for the experiment, which we formulate explicitly in the next
subsection.

This is the sense in which the finite-dimensional theory is both a specialization of the ambient
theory and a concrete recovery theory in its own right.

\subsection{Approximate ambient reconstruction from finite-dimensional regularized reductions}
\label{sec:ambient-approximate-recovery}

The previous subsections establish a \emph{truth-specific reduced observability theory}. For each
reduced truth path \(a^\dagger\in \mathcal A_m(U)\), one may design localized sensors---in
particular, a truth-dependent single-sensor trajectory \(y[a^\dagger]\)---for which the averaged
reduced Gramian is positive definite. Along that fixed truth path, this positivity combines with the
reduced transport tensor to yield coercivity of the reduced joint transport--observability form on
all reduced perturbations.

To convert that pathwise observability statement into a genuine inverse problem posed over a full
reduced admissible class, one needs a stronger \emph{classwise reduced stability property} for the
chosen experiment. We therefore do not claim that the earlier moving-sensor design theorems
alone yield global reduced reconstruction on \(\mathcal A_m(U)\). Instead, we formulate the required
reduced stability estimate directly at the reduced level. Once that stronger property is available,
the reduced Tikhonov problem admits global minimizers on the whole admissible coefficient class,
and those reduced reconstructions lift to approximate ambient reconstructions with explicit
projection-error, induced data-mismatch, and regularization-bias terms.

Throughout this subsection we fix the path-space norm
\[
    \mathfrak H
    :=
    L^2(0,T;L^2(\nu_0))
    \cap
    H^1(0,T;L^2_0(\nu_0)),
    \qquad
    \|h\|_{\mathfrak H}^2
    :=
    \|h\|_{L^2(0,T;L^2(\nu_0))}^2
    +
    \|\dot h\|_{L^2(0,T;L^2_0(\nu_0))}^2.
\]

Let
\[
    V_m=\operatorname{span}\{\phi_1,\dots,\phi_m\}\subset \mathcal X
\]
be the nested reduced spaces from Section~\ref{sec:finite-dimensional-reduction}, and let
\[
    P_m:L^2_0(\nu_0)\to V_m
\]
be the \(L^2(\nu_0)\)-orthogonal projection, extended pointwise in time to path space by
\[
    (P_m h)(t):=P_m(h(t)).
\]
For \(a=(a_1,\dots,a_m)\in\mathbb R^m\), we continue to write
\[
    h(a):=\sum_{k=1}^m a_k\phi_k,
\]
and for a coefficient path \(a\in H^1(0,T;\mathbb R^m)\) we write
\[
    \partial_t h(a)(t):=\sum_{k=1}^m \dot a_k(t)\phi_k.
\]

Given a sensor trajectory \(y\), we retain the reduced pathwise forward map
\[
    \mathcal G_m^y[a](t):=\mathcal G_t^y(h(a(t)))\in \mathbb R^r,
    \qquad
    \mathcal Y_T:=L^2(0,T;\mathbb R^r).
\]

\begin{definition}[Reduced regularization energy and reduced Tikhonov functional]
    \label{def:reduced-tikhonov-functional}
    For \(a\in \mathcal A_m(U)\) such that \(h(a(\cdot))\in \mathcal A_{\mathrm{ad}}\), define
    \begin{align}
        \label{eq:reduced-regularization-energy}
        \mathcal R_{\lambda,\mu}^{(m)}[a]
         & :=
        \frac{\lambda}{2}\int_0^T \dot a(t)^\top H(a(t))\dot a(t)\,\d t \notag \\
         & \quad
        +
        \frac{\mu}{2}\int_0^T
        \Bigl(
        \|h(a(t))\|_{L^2(\nu_0)}^2
        +
        \|\partial_t h(a)(t)\|_{L^2(\nu_0)}^2
        \Bigr)\,\d t .
    \end{align}
    Given data \(d\in \mathcal Y_T\), the corresponding reduced Tikhonov functional is
    \begin{equation}
        \label{eq:reduced-tikhonov}
        \mathcal J_{\lambda,\mu}^{(m),y}[a;d]
        :=
        \frac12\|\mathcal G_m^y[a]-d\|_{\mathcal Y_T}^2
        +
        \mathcal R_{\lambda,\mu}^{(m)}[a].
    \end{equation}
\end{definition}

\begin{remark}
    \label{rem:reduced-norm-equivalence-last-subsection}
    Because \(V_m\) is finite dimensional, all norms on \(V_m\) are equivalent. In particular, there
    exist constants \(c_m,C_m>0\) such that for every \(b\in\mathbb R^m\),
    \[
        c_m |b|^2
        \le
        \left\|\sum_{k=1}^m b_k\phi_k\right\|_{L^2(\nu_0)}^2
        \le
        C_m |b|^2.
    \]
    Hence \(\mathcal R_{\lambda,\mu}^{(m)}\) is equivalent to a Euclidean \(H^1\)-in-time penalty on the
    coefficient path \(a(\cdot)\). We keep the form \eqref{eq:reduced-regularization-energy} because
    it is exactly the pullback of the ambient regularization to \(V_m\).
\end{remark}

\begin{proposition}[Existence of global reduced minimizers]
    \label{prop:existence-local-reduced-minimizers}
    Let \(U\subset\mathbb R^m\) be compact, assume \(h(a)\in\mathcal X_{\mathrm{ad}}\) for every
    \(a\in U\), and assume the map
    \[
        \mathcal A_m(U)\ni a \longmapsto \mathcal G_m^y[a]\in \mathcal Y_T
    \]
    is continuous with respect to strong convergence in \(C([0,T];\mathbb R^m)\). Then for every
    \(d\in \mathcal Y_T\) and every \(\lambda,\mu>0\), the reduced functional
    \[
        a\longmapsto \mathcal J_{\lambda,\mu}^{(m),y}[a;d]
    \]
    attains a minimum on \(\mathcal A_m(U)\).
\end{proposition}

\begin{proof}
    Let \((a_n)\subset \mathcal A_m(U)\) be a minimizing sequence. Since \(\mu>0\), the definition of
    \(\mathcal J_{\lambda,\mu}^{(m),y}\) and
    Remark~\ref{rem:reduced-norm-equivalence-last-subsection} imply that \((a_n)\) is bounded in
    \(H^1(0,T;\mathbb R^m)\). Passing to a subsequence if necessary,
    \[
        a_n \rightharpoonup a
        \quad\text{weakly in }H^1(0,T;\mathbb R^m),
        \qquad
        a_n\to a
        \quad\text{strongly in }C([0,T];\mathbb R^m).
    \]
    Since \(U\) is compact, \(a(t)\in U\) for all \(t\in[0,T]\), so \(a\in \mathcal A_m(U)\).

    By continuity of the reduced forward map under strong \(C([0,T];\mathbb R^m)\)-convergence,
    \[
        \mathcal G_m^y[a_n]\to \mathcal G_m^y[a]
        \qquad\text{in }\mathcal Y_T.
    \]
    The transport term in \(\mathcal R_{\lambda,\mu}^{(m)}\) is weakly lower semicontinuous because
    \(H(\cdot)\) is continuous on the compact set \(U\), hence bounded and uniformly positive
    semidefinite there. The remaining two regularization terms are convex quadratic forms in
    \(a\) and \(\dot a\), and therefore also weakly lower semicontinuous. Consequently,
    \[
        \mathcal J_{\lambda,\mu}^{(m),y}[a;d]
        \le
        \liminf_{n\to\infty}\mathcal J_{\lambda,\mu}^{(m),y}[a_n;d],
    \]
    so \(a\) is a minimizer.
\end{proof}

The next theorem is the reduced global reconstruction statement. It is formulated under a
classwise reduced stability hypothesis for the chosen experiment. This is the natural reduced
analogue of the ambient stability estimate from Section~\ref{sec:ambient-inverse-problem}, but
now imposed directly on the full reduced admissible class.

\begin{theorem}[Global reduced regularized reconstruction]
    \label{thm:local-reduced-reconstruction}
    Let \(U\subset\mathbb R^m\) be compact, let
    \[
        a^\dagger\in \mathcal A_m(U),
        \qquad
        h_m^\dagger:=h(a^\dagger)\in \mathcal A_{\mathrm{ad}}^{s'},
    \]
    and fix a sensor trajectory
    \[
        y^\dagger\in \mathcal Y_{\mathrm{ad}}^{(1)}.
    \]
    Assume:
    \begin{enumerate}
        \item \(h(a)\in \mathcal A_{\mathrm{ad}}^{s'}\) for every \(a\in \mathcal A_m(U)\);
        \item the reduced Tikhonov functional
              \[
                  \mathcal J_{\lambda,\mu}^{(m),y^\dagger}[\cdot;d]
              \]
              admits minimizers on \(\mathcal A_m(U)\) for every \(d\in \mathcal Y_T\);
        \item there exists \(C_m^{\mathrm{stab}}>0\) such that for every \(a\in \mathcal A_m(U)\),
              \begin{equation}
                  \label{eq:global-reduced-stability}
                  \|h(a)-h(a^\dagger)\|_{L^2(0,T;L^2(\nu_0))}^2
                  \le
                  C_m^{\mathrm{stab}}
                  \left(
                  \|\mathcal G_m^{y^\dagger}[a]-\mathcal G_m^{y^\dagger}[a^\dagger]\|_{\mathcal Y_T}^2
                  +
                  \mathcal R_{\lambda,\mu}^{(m)}[a]
                  \right).
              \end{equation}
    \end{enumerate}

    Then for every noisy data \(d^\delta\in \mathcal Y_T\) and every minimizer
    \[
        \hat a_m^\delta\in \mathcal A_m(U)
    \]
    of \(\mathcal J_{\lambda,\mu}^{(m),y^\dagger}[\cdot;d^\delta]\), one has
    \begin{equation}
        \label{eq:global-reduced-regularized-estimate}
        \|h(\hat a_m^\delta)-h(a^\dagger)\|_{L^2(0,T;L^2(\nu_0))}^2
        \le
        5\,C_m^{\mathrm{stab}}
        \left(
        \|d^\delta-\mathcal G_m^{y^\dagger}[a^\dagger]\|_{\mathcal Y_T}^2
        +
        \mathcal R_{\lambda,\mu}^{(m)}[a^\dagger]
        \right).
    \end{equation}
\end{theorem}

\begin{proof}
    By the minimizing property of \(\hat a_m^\delta\),
    \[
        \frac12\|\mathcal G_m^{y^\dagger}[\hat a_m^\delta]-d^\delta\|_{\mathcal Y_T}^2
        +
        \mathcal R_{\lambda,\mu}^{(m)}[\hat a_m^\delta]
        \le
        \frac12\|\mathcal G_m^{y^\dagger}[a^\dagger]-d^\delta\|_{\mathcal Y_T}^2
        +
        \mathcal R_{\lambda,\mu}^{(m)}[a^\dagger].
    \]
    Hence
    \begin{equation}
        \label{eq:reduced-minimizer-basic-bound}
        \mathcal R_{\lambda,\mu}^{(m)}[\hat a_m^\delta]
        \le
        \frac12\|d^\delta-\mathcal G_m^{y^\dagger}[a^\dagger]\|_{\mathcal Y_T}^2
        +
        \mathcal R_{\lambda,\mu}^{(m)}[a^\dagger]
    \end{equation}
    and
    \begin{equation}
        \label{eq:reduced-data-residual-basic-bound}
        \|\mathcal G_m^{y^\dagger}[\hat a_m^\delta]-d^\delta\|_{\mathcal Y_T}^2
        \le
        \|d^\delta-\mathcal G_m^{y^\dagger}[a^\dagger]\|_{\mathcal Y_T}^2
        +
        2\,\mathcal R_{\lambda,\mu}^{(m)}[a^\dagger].
    \end{equation}

    Now
    \begin{align*}
        \|\mathcal G_m^{y^\dagger}[\hat a_m^\delta]-\mathcal G_m^{y^\dagger}[a^\dagger]\|_{\mathcal Y_T}^2
         & \le
        2\|\mathcal G_m^{y^\dagger}[\hat a_m^\delta]-d^\delta\|_{\mathcal Y_T}^2
        +
        2\|d^\delta-\mathcal G_m^{y^\dagger}[a^\dagger]\|_{\mathcal Y_T}^2 \\
         & \le
        4\|d^\delta-\mathcal G_m^{y^\dagger}[a^\dagger]\|_{\mathcal Y_T}^2
        +
        4\,\mathcal R_{\lambda,\mu}^{(m)}[a^\dagger],
    \end{align*}
    where in the last step we used \eqref{eq:reduced-data-residual-basic-bound}.

    Applying the global reduced stability estimate \eqref{eq:global-reduced-stability} with
    \(a=\hat a_m^\delta\), and then using the last inequality together with
    \eqref{eq:reduced-minimizer-basic-bound}, we obtain
    \begin{align*}
        \|h(\hat a_m^\delta)-h(a^\dagger)\|_{L^2(0,T;L^2(\nu_0))}^2
         & \le
        C_m^{\mathrm{stab}}
        \left(
        \|\mathcal G_m^{y^\dagger}[\hat a_m^\delta]-\mathcal G_m^{y^\dagger}[a^\dagger]\|_{\mathcal Y_T}^2
        +
        \mathcal R_{\lambda,\mu}^{(m)}[\hat a_m^\delta]
        \right) \\
         & \le
        C_m^{\mathrm{stab}}
        \left(
        4\|d^\delta-\mathcal G_m^{y^\dagger}[a^\dagger]\|_{\mathcal Y_T}^2
        +
        4\,\mathcal R_{\lambda,\mu}^{(m)}[a^\dagger]
        +
        \frac12\|d^\delta-\mathcal G_m^{y^\dagger}[a^\dagger]\|_{\mathcal Y_T}^2
        +
        \mathcal R_{\lambda,\mu}^{(m)}[a^\dagger]
        \right) \\
         & \le
        5\,C_m^{\mathrm{stab}}
        \left(
        \|d^\delta-\mathcal G_m^{y^\dagger}[a^\dagger]\|_{\mathcal Y_T}^2
        +
        \mathcal R_{\lambda,\mu}^{(m)}[a^\dagger]
        \right),
    \end{align*}
    which is \eqref{eq:global-reduced-regularized-estimate}.
\end{proof}

We now lift the preceding reduced estimate to the ambient path space. Since the sensor trajectory
may depend on the reduced truth, the lifted statement is formulated for a truth-dependent sequence
of reduced experiments.

\begin{theorem}[Approximate ambient reconstruction from reduced regularized minimizers]
    \label{thm:ambient-approximate-recovery}
    Let
    \[
        \mathcal K\subset \mathcal A_{\mathrm{ad}}^{s'}\cap \mathfrak H
    \]
    be a \(\mathfrak H\)-compact admissible path class. Assume:
    \begin{enumerate}
        \item for every \(m\in\mathbb N\),
              \[
                  P_m\mathcal K\subset \mathcal A_{\mathrm{ad}}^{s'}\cap \mathfrak H,
              \]
              and
              \[
                  \sup_{h\in\mathcal K}\|(I-P_m)h\|_{\mathfrak H}\longrightarrow 0
                  \qquad\text{as }m\to\infty;
              \]
        \item for each \(m\) and each admissible sensor trajectory \(y\) used below, the pathwise
              observation map is Lipschitz on \(\mathcal K\cup P_m\mathcal K\): there exists
              \(L_{\mathcal K,m,y}>0\) such that
              \begin{equation}
                  \label{eq:ambient-observation-lipschitz-last-subsection}
                  \|\mathcal G^{y}[h_1]-\mathcal G^{y}[h_2]\|_{\mathcal Y_T}
                  \le
                  L_{\mathcal K,m,y}\|h_1-h_2\|_{\mathfrak H}
              \end{equation}
              for all \(h_1,h_2\in \mathcal K\cup P_m\mathcal K\).
    \end{enumerate}

    Fix
    \[
        h^\dagger\in\mathcal K.
    \]
    For each \(m\), write
    \[
        z_m^\dagger:=P_m h^\dagger\in V_m,
        \qquad
        z_m^\dagger=h(a_m^\dagger)
    \]
    for the unique coefficient path \(a_m^\dagger\in \mathcal A_m(U_m)\), where
    \(U_m\subset\mathbb R^m\) is a compact set containing the image of \(a_m^\dagger\).
    Let
    \[
        y_m^\dagger\in \mathcal Y_{\mathrm{ad}}^{(1)}
    \]
    be the reduced experiment chosen for the reduced truth \(a_m^\dagger\), and assume the
    hypotheses of Theorem~\ref{thm:local-reduced-reconstruction} hold on \(\mathcal A_m(U_m)\) with
    this trajectory \(y_m^\dagger\). Let
    \[
        d_m^\delta\in \mathcal Y_T,
        \qquad
        \|d_m^\delta-\mathcal G^{y_m^\dagger}[h^\dagger]\|_{\mathcal Y_T}\le \delta,
    \]
    and let
    \[
        \hat a_m^\delta\in \mathcal A_m(U_m)
    \]
    be a minimizer of the reduced Tikhonov functional
    \(\mathcal J_{\lambda,\mu}^{(m),y_m^\dagger}[\cdot;d_m^\delta]\). Define the lifted reduced
    reconstruction
    \[
        \hat h_m^\delta:=h(\hat a_m^\delta)\in V_m.
    \]

    Then
    \begin{align}
        \label{eq:ambient-lifted-regularized-estimate}
        \|\hat h_m^\delta-h^\dagger\|_{L^2(0,T;L^2(\nu_0))}
         & \le
        \|(I-P_m)h^\dagger\|_{L^2(0,T;L^2(\nu_0))} \notag \\
         & \quad
        +
        \Biggl[
        5\,C_m^{\mathrm{stab}}
        \Bigl(
        \bigl(\delta + L_{\mathcal K,m,y_m^\dagger}\|(I-P_m)h^\dagger\|_{\mathfrak H}\bigr)^2
        +
        \mathcal R_{\lambda,\mu}^{(m)}[a_m^\dagger]
        \Bigr)
        \Biggr]^{1/2}.
    \end{align}
\end{theorem}

\begin{proof}
    Set
    \[
        d_m^\dagger:=\mathcal G_m^{y_m^\dagger}[a_m^\dagger]
        =\mathcal G^{y_m^\dagger}[z_m^\dagger].
    \]
    By the observation Lipschitz estimate
    \eqref{eq:ambient-observation-lipschitz-last-subsection},
    \[
        \|\mathcal G^{y_m^\dagger}[h^\dagger]-d_m^\dagger\|_{\mathcal Y_T}
        =
        \|\mathcal G^{y_m^\dagger}[h^\dagger]-\mathcal G^{y_m^\dagger}[z_m^\dagger]\|_{\mathcal Y_T}
        \le
        L_{\mathcal K,m,y_m^\dagger}\|h^\dagger-z_m^\dagger\|_{\mathfrak H}
        =
        L_{\mathcal K,m,y_m^\dagger}\|(I-P_m)h^\dagger\|_{\mathfrak H}.
    \]
    Therefore
    \begin{equation}
        \label{eq:effective-reduced-data-error-global}
        \|d_m^\delta-d_m^\dagger\|_{\mathcal Y_T}
        \le
        \|d_m^\delta-\mathcal G^{y_m^\dagger}[h^\dagger]\|_{\mathcal Y_T}
        +
        \|\mathcal G^{y_m^\dagger}[h^\dagger]-d_m^\dagger\|_{\mathcal Y_T}
        \le
        \delta + L_{\mathcal K,m,y_m^\dagger}\|(I-P_m)h^\dagger\|_{\mathfrak H}.
    \end{equation}

    Now apply Theorem~\ref{thm:local-reduced-reconstruction} with reference path
    \(a_m^\dagger\), chosen experiment \(y_m^\dagger\), and reduced data \(d_m^\delta\). Using
    \eqref{eq:effective-reduced-data-error-global}, we obtain
    \[
        \|\hat h_m^\delta-z_m^\dagger\|_{L^2(0,T;L^2(\nu_0))}^2
        \le
        5\,C_m^{\mathrm{stab}}
        \Bigl(
        \bigl(\delta + L_{\mathcal K,m,y_m^\dagger}\|(I-P_m)h^\dagger\|_{\mathfrak H}\bigr)^2
        +
        \mathcal R_{\lambda,\mu}^{(m)}[a_m^\dagger]
        \Bigr).
    \]
    Finally, use the triangle inequality:
    \[
        \|\hat h_m^\delta-h^\dagger\|_{L^2(0,T;L^2(\nu_0))}
        \le
        \|\hat h_m^\delta-z_m^\dagger\|_{L^2(0,T;L^2(\nu_0))}
        +
        \|z_m^\dagger-h^\dagger\|_{L^2(0,T;L^2(\nu_0))},
    \]
    and note that \(z_m^\dagger=P_m h^\dagger\). This gives
    \eqref{eq:ambient-lifted-regularized-estimate}.
\end{proof}

\begin{remark}[Interpretation of the lifted estimate]
    \label{rem:ambient-lifted-estimate-interpretation}
    The estimate \eqref{eq:ambient-lifted-regularized-estimate} separates three effects.

    First, the term
    \[
        \|(I-P_m)h^\dagger\|_{L^2(0,T;L^2(\nu_0))}
    \]
    is the unavoidable Galerkin projection error.

    Second, the quantity
    \[
        L_{\mathcal K,m,y_m^\dagger}\|(I-P_m)h^\dagger\|_{\mathfrak H}
    \]
    is the induced data/model mismatch: even in the absence of measurement noise, the \(m\)-th
    reduced inverse problem uses data generated by the full ambient truth rather than by its
    projection.

    Third, the term
    \[
        \mathcal R_{\lambda,\mu}^{(m)}[a_m^\dagger]
    \]
    is the regularization bias of the reduced reference path. At fixed \(\lambda,\mu>0\), the theorem
    therefore gives an \emph{approximate ambient reconstruction estimate for reduced regularized
        minimizers}, not an exact consistency statement. Removing that bias would require an additional
    reduced vanishing-regularization analysis, which we do not pursue here.
\end{remark}

\begin{remark}
    \label{rem:ambient-approximate-recovery-section5}
    The role of Section~\ref{sec:finite-dimensional-reduction} can now be stated precisely.

    First, the finite-dimensional reduction identifies the reduced tensors and sensing matrices as
    coordinate pullbacks of the ambient transport and observability geometry.

    Second, Theorem~\ref{thm:path-dependent-mobile-design} and
    Theorem~\ref{thm:local-reduced-coercivity-mobile} provide a \emph{truth-specific reduced
    observability theory}: for each fixed reduced truth path, one may design a localized sensing
    experiment---including a single moving sensor---whose reduced averaged Gramian is positive
    definite and whose reduced joint form is coercive along that path.

    Third, the present subsection isolates the additional ingredient needed to pass from those
    pathwise observability statements to a \emph{global reduced inverse problem} on
    \(\mathcal A_m(U)\): one needs a classwise reduced stability estimate for the chosen experiment.
    Under that stronger assumption, reduced regularized minimizers exist globally and lift to
    approximate ambient reconstructions with explicit projection-error, induced data-mismatch, and
    regularization-bias terms.

    In particular, the finite-dimensional theory should be read in two layers: constructive reduced
    observability along fixed reduced truths, and global reduced inversion once an additional
    classwise stability property is imposed.
\end{remark}

\bibliographystyle{plainnat}
\bibliography{references}

\end{document}